\documentclass[journal]{IEEEtran}

\usepackage[T1]{fontenc}
\usepackage[utf8]{inputenc}
\usepackage{mathtools}

\usepackage{amssymb,mathrsfs}
\usepackage{amsthm}
\usepackage{bm}
\usepackage{scalerel}
\usepackage{nicefrac}
\usepackage{microtype} 
\usepackage[shortlabels]{enumitem}
\usepackage{graphicx}
\usepackage{epstopdf}
\DeclareGraphicsExtensions{.eps,.png,.jpg,.pdf}

\usepackage[hyphens]{url}
\usepackage{colortbl}
\usepackage{booktabs}
\usepackage{multirow}
\usepackage[table,dvipsnames]{xcolor}
\usepackage[normalem]{ulem}
\usepackage{xparse}
\usepackage{calc}
\usepackage{etoolbox}

\makeatletter
\@ifpackageloaded{natbib}{
	\relax
}{
	\usepackage{cite}
}
\makeatother

\usepackage{array}
\newcolumntype{L}[1]{>{\raggedright\let\newline\\\arraybackslash\hspace{0pt}}m{#1}}
\newcolumntype{C}[1]{>{\centering\let\newline\\\arraybackslash\hspace{0pt}}m{#1}}
\newcolumntype{R}[1]{>{\raggedleft\let\newline\\\arraybackslash\hspace{0pt}}m{#1}}

\makeatletter
\let\MYcaption\@makecaption
\makeatother
\usepackage[font=footnotesize]{subcaption}
\makeatletter
\let\@makecaption\MYcaption
\makeatother

\usepackage{glossaries}
\makeatletter
\sfcode`\.1006

\let\oldgls\gls
\let\oldglspl\glspl

\newcommand\fussy@ifnextchar[3]{%
	\let\reserved@d=#1%
	\def\reserved@a{#2}%
	\def\reserved@b{#3}%
	\futurelet\@let@token\fussy@ifnch}
\def\fussy@ifnch{%
	\ifx\@let@token\reserved@d
		\let\reserved@c\reserved@a
	\else
		\let\reserved@c\reserved@b
	\fi
	\reserved@c}

\renewcommand{\gls}[1]{%
\oldgls{#1}\fussy@ifnextchar.{\@checkperiod}{\@}}
\renewcommand{\glspl}[1]{%
\oldglspl{#1}\fussy@ifnextchar.{\@checkperiod}{\@}}

\newcommand{\@checkperiod}[1]{%
	\ifnum\sfcode`\.=\spacefactor\else#1\fi
}

\robustify{\gls}
\robustify{\glspl}
\makeatother

\newacronym{wrt}{w.r.t.}{with respect to}
\newacronym{RHS}{R.H.S.}{right-hand side}
\newacronym{LHS}{L.H.S.}{left-hand side}
\newacronym{iid}{i.i.d.}{independent and identically distributed}

\usepackage{float}

\ifx\notloadhyperref\undefined
	\ifx\loadbibentry\undefined
		\usepackage[hidelinks,hypertexnames=false]{hyperref}
	\else
		\usepackage{bibentry}
		\makeatletter\let\saved@bibitem\@bibitem\makeatother
		\usepackage[hidelinks,hypertexnames=false]{hyperref}
		\makeatletter\let\@bibitem\saved@bibitem\makeatother
	\fi
\else
	\ifx\loadbibentry\undefined
		\relax
	\else
		\usepackage{bibentry}
	\fi
\fi

\usepackage[capitalize]{cleveref}
\crefname{equation}{}{}
\Crefname{equation}{}{}
\crefname{claim}{claim}{claims}
\crefname{step}{step}{steps}
\crefname{line}{line}{lines}
\crefname{condition}{condition}{conditions}
\crefname{dmath}{}{}
\crefname{dseries}{}{}
\crefname{dgroup}{}{}

\crefname{Problem}{Problem}{Problems}
\crefformat{Problem}{Problem~(#2#1#3)}
\crefrangeformat{Problem}{Problems~(#3#1#4) to~(#5#2#6)}

\crefname{Theorem}{Theorem}{Theorems}
\crefname{Corollary}{Corollary}{Corollaries}
\crefname{Proposition}{Proposition}{Propositions}
\crefname{Lemma}{Lemma}{Lemmas}
\crefname{Definition}{Definition}{Definitions}
\crefname{Example}{Example}{Examples}
\crefname{Assumption}{Assumption}{Assumptions}
\crefname{Remark}{Remark}{Remarks}
\crefname{Rem}{Remark}{Remarks}
\crefname{remarks}{Remarks}{Remarks}
\crefname{Supplement}{Supplement}{Supplements}
\crefname{Exercise}{Exercise}{Exercises}
\crefname{Theorem_A}{Theorem}{Theorems}
\crefname{Corollary_A}{Corollary}{Corollaries}
\crefname{Proposition_A}{Proposition}{Propositions}
\crefname{Lemma_A}{Lemma}{Lemmas}
\crefname{Definition_A}{Definition}{Definitions}

\usepackage{crossreftools}
\ifx\notloadhyperref\undefined
\else
	\relax
\fi

\usepackage{algorithm,algorithmic}

\ifx\loadbreqn\undefined
	\relax
\else
	\usepackage{breqn}
\fi

\makeatletter
\def\cleartheorem#1{%
    \expandafter\let\csname#1\endcsname\relax
    \expandafter\let\csname c@#1\endcsname\relax
}
\def\clearthms#1{ \@for\tname:=#1\do{\cleartheorem\tname} }
\makeatother

\ifx\renewtheorem\undefined
	\ifx\useTheoremCounter\undefined
		\newtheorem{Theorem}{Theorem}
		\newtheorem{Corollary}{Corollary}
		\newtheorem{Proposition}{Proposition}
		
	\else
		\newtheorem{Theorem}{Theorem}
		
		\newtheorem{Proposition}[Theorem]{Proposition}
	\fi

	\newtheorem{Remark}{Remark}
	\newtheorem{Assumption}{Assumption}

\fi

\theoremstyle{remark}

\theoremstyle{plain}

\newcommand{\qednew}{\nobreak \ifvmode \relax \else
		\ifdim\lastskip<1.5em \hskip-\lastskip
			\hskip1.5em plus0em minus0.5em \fi \nobreak
		\vrule height0.75em width0.5em depth0.25em\fi}



\newcommand{\nn}{\nonumber\\ }

\NewDocumentCommand{\movedownsub}{e{^_}}{%
	\IfNoValueTF{#1}{%
		\IfNoValueF{#2}{^{}}
	}{%
		^{#1}
	}%
	\IfNoValueF{#2}{_{#2}}
}

\let\latexchi\chi
\RenewDocumentCommand{\chi}{}{\latexchi\movedownsub}

\newcommand{\Real}{\mathbb{R}}

\newcommand{\calF}{\mathcal{F}}
\newcommand{\calG}{\mathcal{G}}

\newcommand{\calM}{\mathcal{M}}

\newcommand{\bbA}{\mathbb{A}}
\newcommand{\bbB}{\mathbb{B}}

\newcommand{\bbR}{\mathbb{R}}

\DeclareSymbolFont{bsfletters}{OT1}{cmss}{bx}{n}
\DeclareSymbolFont{ssfletters}{OT1}{cmss}{m}{n}
\DeclareMathSymbol{\bsfGamma}{0}{bsfletters}{'000}
\DeclareMathSymbol{\ssfGamma}{0}{ssfletters}{'000}
\DeclareMathSymbol{\bsfDelta}{0}{bsfletters}{'001}
\DeclareMathSymbol{\ssfDelta}{0}{ssfletters}{'001}
\DeclareMathSymbol{\bsfTheta}{0}{bsfletters}{'002}
\DeclareMathSymbol{\ssfTheta}{0}{ssfletters}{'002}
\DeclareMathSymbol{\bsfLambda}{0}{bsfletters}{'003}
\DeclareMathSymbol{\ssfLambda}{0}{ssfletters}{'003}
\DeclareMathSymbol{\bsfXi}{0}{bsfletters}{'004}
\DeclareMathSymbol{\ssfXi}{0}{ssfletters}{'004}
\DeclareMathSymbol{\bsfPi}{0}{bsfletters}{'005}
\DeclareMathSymbol{\ssfPi}{0}{ssfletters}{'005}
\DeclareMathSymbol{\bsfSigma}{0}{bsfletters}{'006}
\DeclareMathSymbol{\ssfSigma}{0}{ssfletters}{'006}
\DeclareMathSymbol{\bsfUpsilon}{0}{bsfletters}{'007}
\DeclareMathSymbol{\ssfUpsilon}{0}{ssfletters}{'007}
\DeclareMathSymbol{\bsfPhi}{0}{bsfletters}{'010}
\DeclareMathSymbol{\ssfPhi}{0}{ssfletters}{'010}
\DeclareMathSymbol{\bsfPsi}{0}{bsfletters}{'011}
\DeclareMathSymbol{\ssfPsi}{0}{ssfletters}{'011}
\DeclareMathSymbol{\bsfOmega}{0}{bsfletters}{'012}
\DeclareMathSymbol{\ssfOmega}{0}{ssfletters}{'012}

\makeatletter
\newcommand*\rel@kern[1]{\kern#1\dimexpr\macc@kerna}
\newcommand*\widebar[1]{%
  \begingroup
  \def\mathaccent##1##2{%
    \rel@kern{0.8}%
    \overline{\rel@kern{-0.8}\macc@nucleus\rel@kern{0.2}}%
    \rel@kern{-0.2}%
  }%
  \macc@depth\@ne
  \let\math@bgroup\@empty \let\math@egroup\macc@set@skewchar
  \mathsurround\z@ \frozen@everymath{\mathgroup\macc@group\relax}%
  \macc@set@skewchar\relax
  \let\mathaccentV\macc@nested@a
  \macc@nested@a\relax111{#1}%
  \endgroup
}
\makeatother

\newcommand{\ifbcdot}[1]{\ifblank{#1}{\cdot}{#1}}

\DeclarePairedDelimiterX\abs[1]{\lvert}{\rvert}{\ifbcdot{#1}}
\DeclarePairedDelimiterX\parens[1]{(}{)}{\ifbcdot{#1}}
\DeclarePairedDelimiterX\brk[1]{[}{]}{\ifbcdot{#1}}
\DeclarePairedDelimiterX\braces[1]{\{}{\}}{\ifbcdot{#1}}
\DeclarePairedDelimiterX\angles[1]{\langle}{\rangle}{\ifblank{#1}{\cdot,\cdot}{#1}}
\DeclarePairedDelimiterX\ip[2]{\langle}{\rangle}{\ifbcdot{#1},\ifbcdot{#2}}
\DeclarePairedDelimiterX\norm[1]{\lVert}{\rVert}{\ifbcdot{#1}}
\DeclarePairedDelimiterX\ceil[1]{\lceil}{\rceil}{\ifbcdot{#1}}
\DeclarePairedDelimiterX\floor[1]{\lfloor}{\rfloor}{\ifbcdot{#1}}

\DeclareFontFamily{U}{matha}{\hyphenchar\font45}
\DeclareFontShape{U}{matha}{m}{n}{
      <5> <6> <7> <8> <9> <10> gen * matha
      <10.95> matha10 <12> <14.4> <17.28> <20.74> <24.88> matha12
      }{}
\DeclareSymbolFont{matha}{U}{matha}{m}{n}
\DeclareFontSubstitution{U}{matha}{m}{n}

\DeclareFontFamily{U}{mathx}{\hyphenchar\font45}
\DeclareFontShape{U}{mathx}{m}{n}{
      <5> <6> <7> <8> <9> <10>
      <10.95> <12> <14.4> <17.28> <20.74> <24.88>
      mathx10
      }{}
\DeclareSymbolFont{mathx}{U}{mathx}{m}{n}
\DeclareFontSubstitution{U}{mathx}{m}{n}

\DeclareMathDelimiter{\vvvert}{0}{matha}{"7E}{mathx}{"17}
\DeclarePairedDelimiterX\vertiii[1]{\vvvert}{\vvvert}{\ifbcdot{#1}}

\DeclarePairedDelimiterXPP\trace[1]{\operatorname{Tr}}{(}{)}{}{\ifbcdot{#1}} 
\DeclarePairedDelimiterXPP\col[1]{\operatorname{col}}{\{}{\}}{}{\ifbcdot{#1}} 
\DeclarePairedDelimiterXPP\row[1]{\operatorname{row}}{\{}{\}}{}{\ifbcdot{#1}} 
\DeclarePairedDelimiterXPP\erf[1]{\operatorname{erf}}{(}{)}{}{\ifbcdot{#1}}
\DeclarePairedDelimiterXPP\erfc[1]{\operatorname{erfc}}{(}{)}{}{\ifbcdot{#1}}
\DeclarePairedDelimiterXPP\KLD[2]{D}{(}{)}{}{\ifbcdot{#1}\, \delimsize\|\, \ifbcdot{#2}} 
\DeclarePairedDelimiterXPP\op[2]{\operatorname{#1}}{(}{)}{}{#2} 

\DeclarePairedDelimiterXPP\indicate[1]{{\bf 1}}{\{}{\}}{}{\ifbcdot{#1}}
\newcommand{\indicator}[1]{{\bf 1}_{\braces*{\ifbcdot{#1}}}}

\NewDocumentCommand\ofrac{s m}{%
	\IfBooleanTF#1%
	{\dfrac{1}{#2}}%
	{\frac{1}{#2}}%
}
\NewDocumentCommand\ddfrac{s m m}{%
	\IfBooleanTF#1%
	{\dfrac{\mathrm{d} {#2}}{\mathrm{d} {#3}}}%
	{\frac{\mathrm{d} {#2}}{\mathrm{d} {#3}}}%
}
\NewDocumentCommand\ppfrac{s m m}{%
	\IfBooleanTF#1%
	{\dfrac{\partial {#2}}{\partial {#3}}}%
	{\frac{\partial {#2}}{\partial {#3}}}%
}

\providecommand\given{}

\DeclarePairedDelimiterX\Set[2]\{\}{%
\renewcommand\given{\SetSymbol[\delimsize]{#1}}
#2
}
\DeclarePairedDelimiterX\Setc[1]\{\}{%
\renewcommand\given{\SetSymbol{:}}
#1
}

\NewDocumentCommand\set{s o m}{%
	\IfBooleanTF#1%
	{\IfValueTF{#2}{\Set*{#2}{#3}}{\Setc*{#3}}}%
	{\IfValueTF{#2}{\Set{#2}{#3}}{\Setc{#3}}}%
}

\NewDocumentCommand{\evalat}{ s O{\big} m e{_^} }{%
\IfBooleanTF{#1}%
{\left. #3 \right|}{#3#2|}%
\IfValueT{#4}{_{#4}}%
\IfValueT{#5}{^{#5}}%
}

\providecommand\given{}
\DeclarePairedDelimiterXPP\cprob[1]{}(){}{
\renewcommand\given{\nonscript\,\delimsize\vert\allowbreak\nonscript\,\mathopen{}}%
#1%
}
\DeclarePairedDelimiterXPP\cexp[1]{}[]{}{
\renewcommand\given{\nonscript\,\delimsize\vert\allowbreak\nonscript\,\mathopen{}}%
#1%
}

\DeclareDocumentCommand \P { s e{_^} d() g } {%
	\mathbb{P}%
	\IfBooleanTF{#1}%
		{
			\IfValueT{#2}{_{#2}}%
			\IfValueT{#3}{^{#3}}%
			\IfValueTF{#5}{\cprob{#4 \given #5}}{\IfValueT{#4}{\cprob{#4}}}%
		}%
		{
			\IfValueT{#2}{_{#2}}%
			\IfValueT{#3}{^{#3}}%
			\IfValueTF{#5}{\cprob*{#4 \given #5}}{\IfValueT{#4}{\cprob*{#4}}}%
		}%
}

\DeclareDocumentCommand \E { s e{_^} o g } {%
	\mathbb{E}%
	\IfBooleanTF{#1}%
		{
			\IfValueT{#2}{_{#2}}%
			\IfValueT{#3}{^{#3}}%
			\IfValueTF{#5}{\cexp{#4 \given #5}}{\IfValueT{#4}{\cexp{#4}}}%
		}%
		{
			\IfValueT{#2}{_{#2}}%
			\IfValueT{#3}{^{#3}}%
			\IfValueTF{#5}{\cexp*{#4 \given #5}}{\IfValueT{#4}{\cexp*{#4}}}%
		}%
}

\ExplSyntaxOn
\NewDocumentCommand \dist {m o o} {%
\mathrm{#1}\left(%
	\IfValueT{#3}{%
		\tl_if_blank:nTF{ #3 }{\cdot\, \middle|\, }{#3\, \middle|\, }%
	}
	\IfValueT{#2}{#2}%
\right)%
}
\ExplSyntaxOff

\NewDocumentCommand {\cbrace} {t+ D[]{black} D(){\widthof{#5}} m m } {%
	\begingroup%
		\color{#2}
		\IfBooleanTF{#1}{%
			\overbrace{#4}^%
		}{
			\underbrace{#4}_%
		}%
		{\parbox[c]{#3}{\centering\footnotesize{#5}}}%
	\endgroup%
}

\let\oldforall\forall
\renewcommand{\forall}{\oldforall \, }

\let\oldexist\exists
\renewcommand{\exists}{\oldexist \, }

\graphicspath{{./Figures/}{./figures/}}
\DeclareDocumentCommand{\includeCroppedPdf}{ o O{./Figures/} m }{
	\IfFileExists{#2#3-crop.pdf}{}{%
		\immediate\write18{pdfcrop #2#3.pdf #2#3-crop.pdf}}%
	\includegraphics[#1]{#2#3-crop.pdf}
}

\makeatletter
\newcommand*{\addFileDependency}[1]{
  \typeout{(#1)}
  \@addtofilelist{#1}
  \IfFileExists{#1}{}{\typeout{No file #1.}}
}
\makeatother

\definecolor{gray90}{gray}{0.9}

\ifx\nohighlights\undefined
	\newcommand{\red}[1]{{\color{red} #1}}
	\newcommand{\blue}[1]{{{\color{blue} #1}}}

	\newcommand{\msout}[1]{\text{\color{green} \sout{\ensuremath{#1}}}}
	\newcommand{\del}[1]{{\color{green}\ifmmode \msout{#1}\else\sout{#1}\fi}}
\else
	\newcommand{\red}[1]{#1}
	\newcommand{\blue}[1]{#1}

	\newcommand{\msout}[1]{#1}
	\newcommand{\del}[1]{#1}
\fi

\newcommand{\hhide}[1]{}

\ifx\diagnoselabel\undefined
	\relax
\else
	\makeatletter
	\def\@testdef #1#2#3{%
		\def\reserved@a{#3}\expandafter \ifx \csname #1@#2\endcsname
			\reserved@a  \else
			\typeout{^^Jlabel #2 changed:^^J%
				\meaning\reserved@a^^J%
				\expandafter\meaning\csname #1@#2\endcsname^^J}%
			\@tempswatrue \fi}
	\makeatother
\fi

\usepackage{stfloats}
\usepackage{textcomp}
\usepackage{hhline}
\usepackage{multirow}

\newcommand{\pfgm}{p(\varphi(x), \psi(\calG^{y}) \mid x \leftrightarrow y)}
\newcommand{\pfgu}{p(\varphi(x), \psi(\calG^{y}) \mid x \nleftrightarrow y)}
\newcommand{\pfgGm}[1]{p(\varphi(x), \psi(\calG^{y}) \mid #1, x \leftrightarrow y)}
\newcommand{\pfgGu}[1]{p(\varphi(x), \psi(\calG^{y}) \mid #1, x \nleftrightarrow y)}
\newcommand{\Gm}{\calG^x \leftrightarrow \calG^y}
\newcommand{\Gu}{\calG^x \nleftrightarrow \calG^y}
\newacronym{GNN}{GNN}{graph neural network}

\begin{document}




\title{Image Patch-Matching with Graph-Based Learning in Street Scenes}

\author{Rui~She,~Qiyu~Kang,~Sijie~Wang,~Wee~Peng~Tay,~\IEEEmembership{Senior~Member,~IEEE,}
~Yong~Liang~Guan,~\IEEEmembership{Senior~Member,~IEEE,}~Diego~Navarro~Navarro,~and~Andreas~Hartmannsgruber
\thanks{This work is supported by the Singapore Ministry of Education Academic Research Fund Tier 2 grant MOE-T2EP20220-0002 and the RIE2020 Industry Alignment Fund-Industry Collaboration Projects (IAF-ICP) Funding Initiative, as well as cash and in-kind contribution from the industry partner(s).}
\thanks{The first two authors, R. She and Q. Kang, contributed equally to this work.}
\thanks{R. She, Q. Kang, S. Wang, W. P. Tay and Y. L. Guan are with the Continental-NTU Corporate Lab, 
Nanyang Technological University, Singapore 639798 (emails: {\{rui.she@;wang1679@e.;qiyu.kang@; wptay@;eylguan@\}ntu.edu.sg}).}
\thanks{D. N. Navarro and A. Hartmannsgruber are with Continental Automotive Singapore Pte. Ltd., 80 Boon Keng Road, 339780, Singapore. (emails: {\{diego.navarro.navarro; andreas.hartmannsgruber\}@continental.com}).}
}



\maketitle

\begin{abstract}
Matching landmark patches from a real-time image captured by an on-vehicle camera with landmark patches in an image database plays an important role in various computer perception tasks for autonomous driving. Current methods focus on local matching for regions of interest and do not take into account spatial neighborhood relationships among the image patches, which typically correspond to objects in the environment. In this paper, we construct a spatial graph with the graph vertices corresponding to patches and edges capturing the spatial neighborhood information. We propose a joint feature and metric learning model with graph-based learning. We provide a theoretical basis for the graph-based loss by showing that the information distance between the distributions conditioned on matched and unmatched pairs is maximized under our framework. We evaluate our model using several street-scene datasets and demonstrate that our approach achieves state-of-the-art matching results.
\end{abstract}


\section{Introduction}\label{sect:introduction}

\IEEEPARstart{A}{s} a critical and fundamental technique in visual perception, image matching is widely used in many applications, such as image retrieval \cite{zhang2017landmark} and vehicle re-identification \cite{zhu2019vehicle}.
Conceptually, the target of a matching task is to solve the similarity correspondence problem for contents from an image pair \cite{wang2020deep,quan2021multi,liao2011nonrigid}. 
In landmark-based street-scene applications, semantic objects such as traffic signs, traffic lights and road-side poles \cite{engel2019deeplocalization,wang2020robust,zhu2007ten} often serve as landmarks. The correspondence between the landmark patches captured at different locations may be further utilized as cornerstones to solve other problems, including loop-closure detection in simultaneous localization and mapping (SLAM) \cite{wang2020robust,vincent2020dynamic}, place recognition \cite{sunderhauf2015place,hausler2021patch}, multi-view camera relocalization \cite{xue2020learning}, landmark-LiDAR vehicle relocalization \cite{zhang2015visual,engel2019deeplocalization}, and landmark-based odometry estimation \cite{zhu2007ten}. 

\begin{figure}[!htb]
\centering
\includegraphics[width=\linewidth]{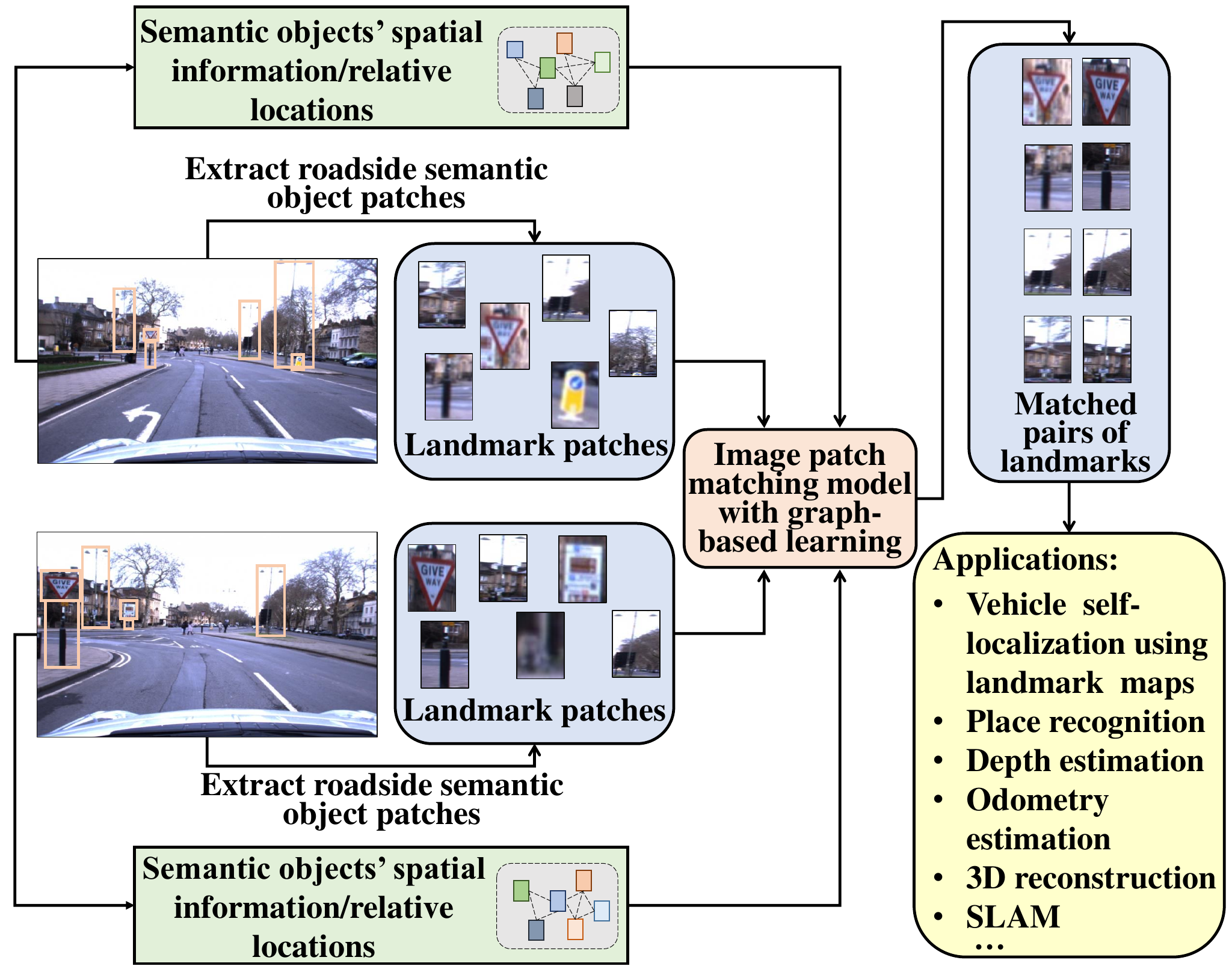}
\vspace{-0.6cm}
\caption{Landmark patch-matching using spatial graphs in street scenes and its potential applications.}
\label{fig:figure_overview}
\vspace{-0.5cm}
\end{figure}

In traditional image patch-matching methods, handcrafted local features using pixel statistics or gradient information, such as SIFT \cite{lowe2004distinctive}, SURF  \cite{bay2008speeded}, HOG \cite{felzenszwalb2008discriminatively} and ORB \cite{rublee2011orb}, are used. 
The similarity of a feature pair is commonly computed using different predefined metrics, like the $\mathcal{L}_2$ distance and cosine distance.
Moreover, a circular pattern with an adjustable radius is exploited in BRISK \cite{leutenegger2011brisk} and FREAK \cite{alahi2012freak}, which provides more efficient neighborhood information for computing relevant pixel statistics.
However, these handcrafted features are not robust to viewpoint changes, varying illuminations and transformations.
Consequently, the matching performance for methods based on such handcrafted local features is often unstable \cite{tian2017l2}.

With the rapid development of artificial intelligence techniques, deep learning methods, such as convolutional neural networks (CNNs), are widely used in image matching \cite{zagoruyko2015learning,kumar2016learning,balntas2016learning}.
In this case, high-dimensional features are exploited to replace handcrafted features in image representations.
In the joint feature and metric learning method \cite{han2015matchnet,subramaniam2018ncc,quan2019afd}, the representations and similarity metrics are combined in an end-to-end learning framework, in which high-level features of the images are extracted, and their similarities are learned simultaneously. 
The feature descriptor learning method \cite{tian2017l2,mishchuk2017working,wang2019better,ng2020solar,miao2021learning} focuses on high-level feature learning and tries to keep matched samples close and unmatched samples far from each other in the corresponding feature space. The similarity is computed using a predefined similarity metric. In these approaches, matching is based on learning feature representations of each image separately and does not exploit the relationships between objects in the images. 
Recent keypoint-based learning methods such as D2-Net \cite{dusmanu2019d2}, ASLFeat \cite{luo2020aslfeat} and SuperGlue \cite{sarlin2020superglue} perform the point-level correspondence based on the detected keypoints and their descriptors for the input images, which can also be used for the image matching task \cite{amiri2011rasim}.

\begin{figure}[!htb]
\centering
\includegraphics[width=0.5\linewidth, height=0.13\textheight]{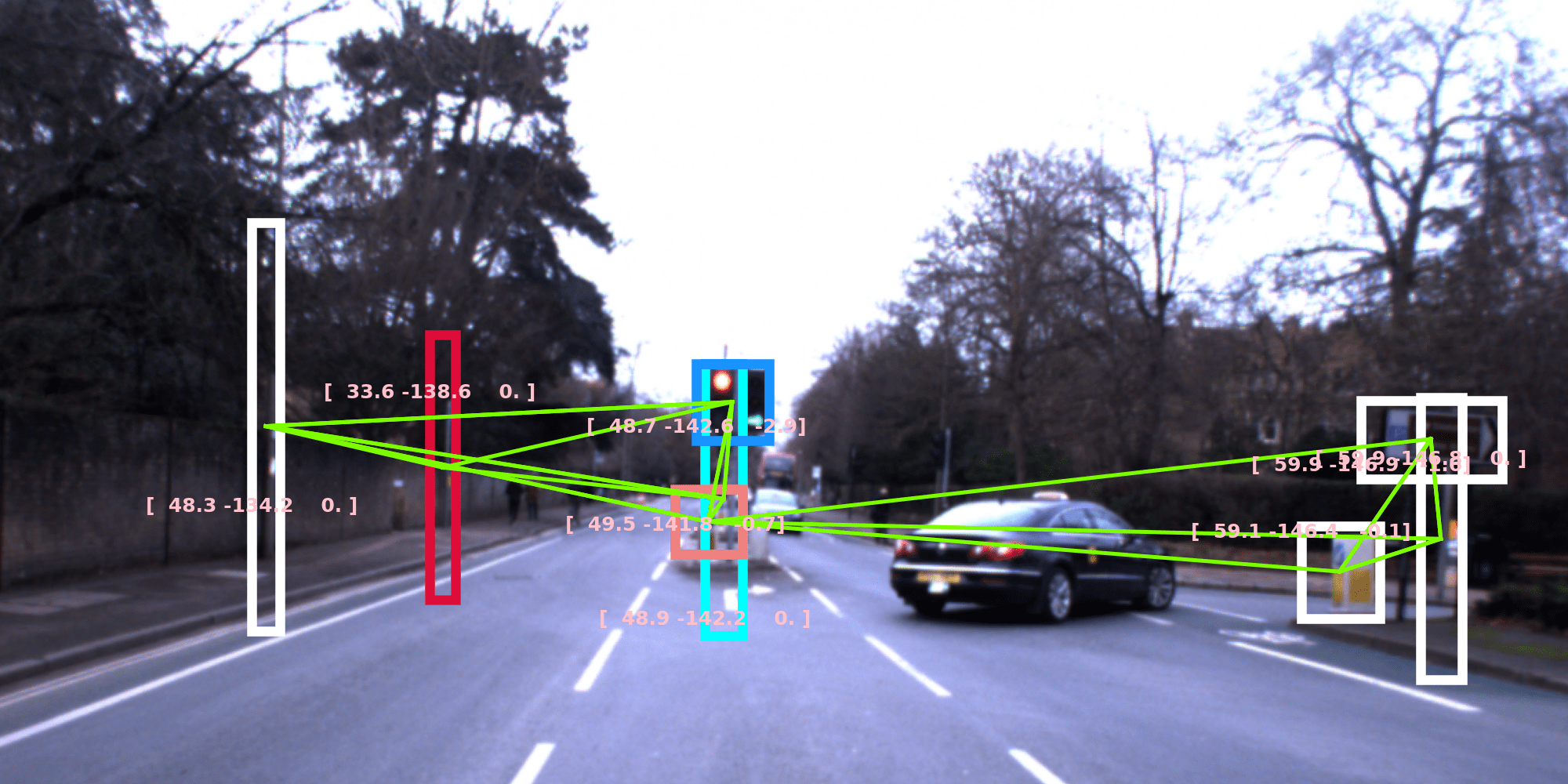}\hfill
\includegraphics[width=0.5\linewidth, height=0.13\textheight]{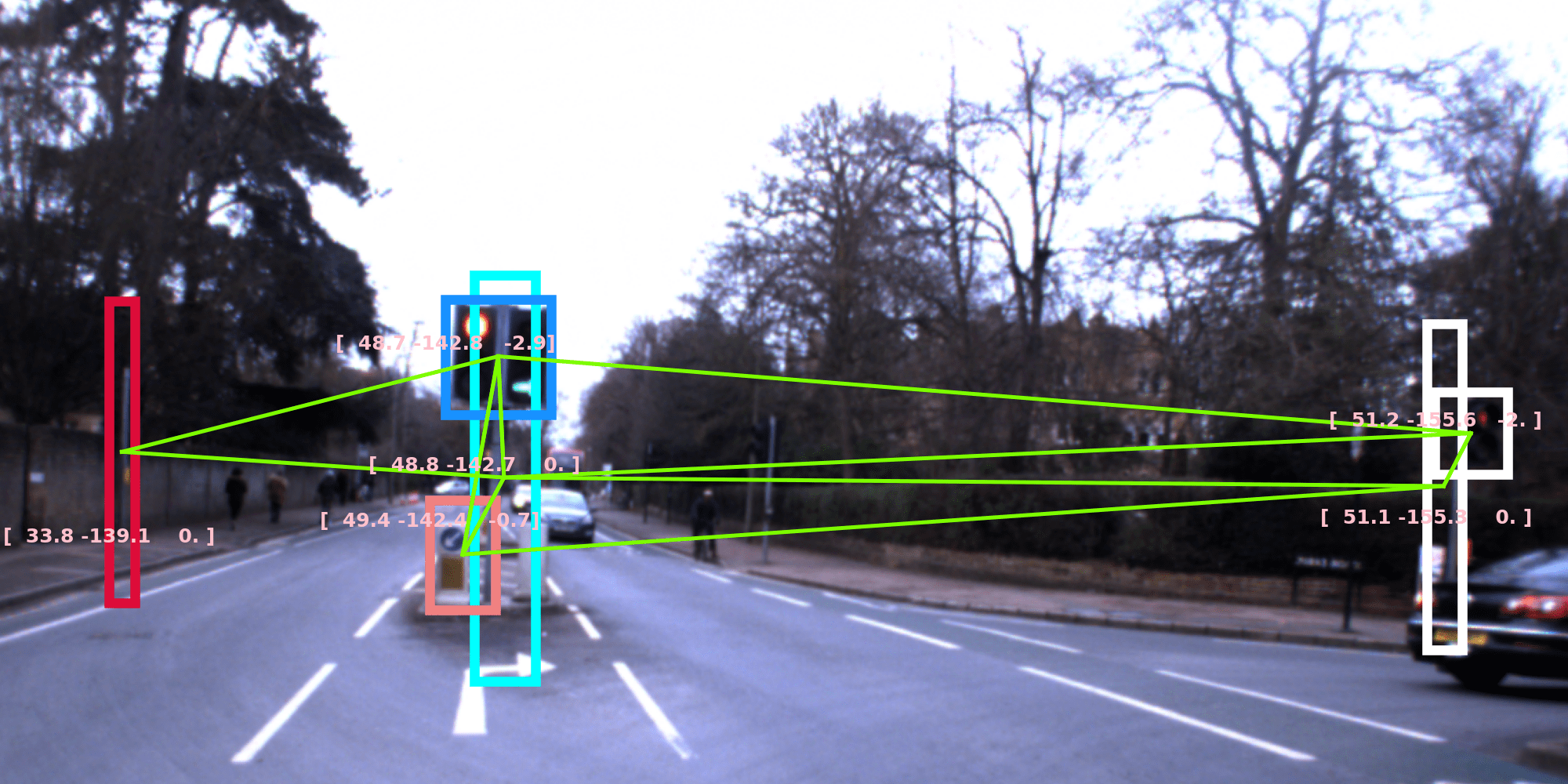}\hfill
\vspace{-0.1cm}
\caption{Landmark patches matching in two full-sized images sampled from the Oxford Radar RobotCar dataset. The matched landmark patches are labeled with the same colored bounding boxes, while the white bounding box indicates that the landmark patch in one image has no matched pair in the other image. Green lines indicate the constructed graph edges in our model.}
\label{fig:figure_landmark_patch1}
\vspace{-0.2cm}
\end{figure}

Unlike other image patch-matching tasks, rich spatial information for landmark patches is often available. For example, lamp posts along a road are usually spaced at equal intervals, and their relative locations \gls{wrt} each other in the environment provide additional information for the matching task.
In special street scenes like the downtown or central business district (CBD), landmark patch-matching has advantages over conventional pixel-/point-level matching due to the presence of dynamic objects, such as vehicles and pedestrians. 
These dynamic objects captured by the vehicular cameras may have more matched pixels across different frames than static landmarks. However, matching these objects is useless or even harmful for tasks such as place recognition. To mitigate this issue, in this work, we perform the patch-matching task based on static landmarks such as traffic lights, traffic signs, poles, and windows.

Inspired by graph-level representation learning \cite{velivckovic2018deep,sun2020infograph}, we propose to construct a graph for the neighborhood of an image patch and use graph-level representations to enrich the landmark patch embedding. 
We identify each landmark patch as a vertex of a graph and find the $K$-nearest neighbors based on estimated spatial information. In the literature, there exist various spatial information estimation techniques like structure-from-motion (SfM)\cite{schonberger2016structure}, monocular or stereo depth estimation \cite{yin2021learning} and optical attenuation masks \cite{farid1996differential}. In this paper, for the sake of illustration, we choose an off-the-shelf monocular depth estimation method from \cite{yin2021learning} to estimate the landmark spatial relations. However, any other spatial estimation or augmented ranging sensors like LiDAR or depth camera can also be utilized in our framework. We form a clique whose vertex embeddings are learnable via a \gls{GNN}\cite{velivckovic2017graph,wu2020comprehensive}. This graph is utilized in our proposed patch-matching framework for object information characterization. The final matching score is an average of the graph and vertex embedding similarity.

We also introduce two landmark patch-matching datasets derived from the street-scene KITTI dataset \cite{geiger2012we} and the Oxford Radar RobotCar dataset \cite{barnes2020oxford}. Our paper focuses on matching image patches of specific \emph{static} roadside objects from two full-sized images taken by cameras onboard vehicles. See \cref{fig:figure_overview} for an illustration. 
More specifically, we focus on static roadside objects including traffic lights, signs, lamp posts, and even windows on a building facade. This is because in most landmark-based applications,  other transient static objects like parked cars, are inappropriate landmarks or do not have sufficient distinctive features. Due to complex environmental conditions like dynamic element occlusion, e.g., due to pedestrians, vehicles, or the scene viewpoint changing (especially when turning at sharp corners or traversing a stretch in opposite directions),  the landmark patches may have dramatic differences in appearance. We refer the readers to the supplementary material for more details on the landmark patch-matching datasets' preparation. For a concrete illustration, some examples of matched or unmatched landmark patches are presented in \cref{fig:figure_landmark_patch1}. 

Our contributions are summarized as follows:
\begin{itemize}
\item We propose a landmark patch-matching method with graph-based learning for vehicles in street scenes, which extends the feature representation approach used in traditional image patch-matching tasks and incorporates spatial relationship information.
\item We analyze the fundamental principle and properties of the proposed graph-based loss function from an information-theoretic perspective.
\item We introduce two landmark patch-matching datasets, which contain challenging street-view landmark patches captured in an autonomous driving environment.
\item We empirically demonstrate that our method achieves state-of-the-art performance on the landmark patch-matching task when compared to various other benchmarks.
\end{itemize}

The rest of this paper is organized as follows. 
In \cref{sect:relatedworks}, related works are discussed.
Our model and framework are introduced in \cref{sect:model}, where we also provide a theoretical analysis of our graph-based loss. 
We present experimental results in \cref{sect:exper} and conclude the paper in \cref{sect:conc}. 
The proofs for all lemmas or propositions proposed in this paper are provided in the appendices.

\section{Related Works}\label{sect:relatedworks}

Since deep learning-based methods play dominant roles in the image-matching problem, we only discuss deep learning-based works here.
Deep learning-based methods include feature descriptor learning, joint feature and metric learning, as well as keypoint-based correspondence learning. 

\textbf{Feature descriptor learning.}
High-level features of an image are first extracted using a neural network like a CNN so that matched samples are close while unmatched samples are distant under a similarity metric, which is chosen to be a feature distance function.
In many models \cite{kumar2016learning,balntas2016learning}, pairwise or triplet loss is used to train the neural networks.
To improve performance, in \cite{zhang2017learning}, a regularization is designed by maximizing the spread of local feature descriptors over the descriptor space, from which a better embedding for image-level features is obtained.
To ensure many samples are accessible to the descriptor network within a few epochs, L2-Net \cite{tian2017l2} uses a progressive sampling strategy.
Furthermore, HardNet \cite{mishchuk2017working} is designed to fully utilize the hard negative samples by making the closest positive sample far away from the closest negative sample in a batch. The reference \cite{wang2019better} overcomes the hard sample learning issue by use of exponential Siamese and triplet losses, which naturally pay more attention to hard samples and less attention to easy ones. 
SOSNet is studied in \cite{tian2019sosnet} to learn better local descriptors, where the second-order similarity (SOS) is introduced into the loss function as a regularization.
Moreover, \cite{ng2020solar} designs two second-order components, i.e., the second-order spatial information and the second-order descriptor space similarity, to achieve feature map re-weighting and global descriptors learning, respectively.
The paper \cite{pan2021tcdesc} proposes topology consistent descriptors (TCDesc) based on neighborhood information of descriptors, which can be combined with other methods via the triplet loss.

\textbf{Joint feature and metric learning.}
In joint feature and metric learning, the similarity metric is not predefined and is instead set as a trainable network together with the feature extraction network.
In this case, the matching task is regarded as a binary classification task by resorting to the similarity metric network with a classification loss function.
As a classical method, MatchNet proposed by \cite{han2015matchnet} extracts high-level features by using deep CNNs and measures the feature similarity using fully connected (FC) layers.
To compare the different network architectures for the matching task, several networks, including SiameseNet, Pseudo-SiameseNet and 2-channel network, are investigated in \cite{zagoruyko2015learning,melekhov2016siamese}.
The 2-channel network merges the two images into a 2-channel image to achieve faster convergence.
The SiameseNet and Pseudo-SiameseNet both use two branches based on the same structure to extract high-dimensional features, with and without the shared weights respectively.
Using the normalized cross-correlation (NCC) as a metric, \cite{subramaniam2018ncc} proposes NCC-Net, which utilizes robust matching layers to measure the similarity of feature pairs.
To tackle cross-spectral image matching, AFD-Net is proposed by \cite{quan2019afd} to aggregate multi-level feature differences, which can strengthen the discrimination of the network.

\textbf{Keypoint-based correspondence learning.} 
In keypoint-based correspondence learning, the main procedure is to construct neural networks to perform keypoint detection and description and to measure or learn the keypoints' similarity for matching inference.
For instance, LIFT \cite{yi2016lift} is designed based on a united deep network architecture where keypoints are detected in the first network, the orientation for cropped regions is estimated in the second network, and the feature description is performed in the third network. Here, the Euclidean distance is used to measure the similarity of features.
The SuperPoint approach \cite{detone2018superpoint} introduces a self-supervised domain adaptation framework named Homographic Adaptation into interest point detection and description. 
The D2-Net \cite{dusmanu2019d2} makes use of a single CNN to perform dense feature description and detection simultaneously, where the detection, instead of being based on low-level image structures, is postponed to the high-level structures, which are also used for image descriptions.
Based on the D2-Net backbone architecture, ASLFeat \cite{luo2020aslfeat} is equipped with three lightweight effective modifications, which have better local shape estimation and more accurate keypoint localization.
The above methods all measure the point-level correspondence based on Euclidean distances.
On the other hand, SuperGlue \cite{sarlin2020superglue} is designed using attention GNNs and the Sinkhorn algorithm for keypoint-based feature matching.
LoFTR \cite{sun2021loftr} achieves accurate semi-dense matches with Transformers including self and cross-attention layers. 
Generally speaking, all the above keypoint-based correspondence learning methods can be used to perform the image matching task with further operations on the keypoint matching scores \cite{sarlin2020superglue}.

To improve image matching performance, spatial information is used in \cite{brogan2021fast,sarlin2020superglue,sun2021loftr} through spatial verification, graph learning, and cross attention. 
In spatial verification, spatial information is usually used for the transformation calibration \gls{wrt} the key points or objects, as well as a correspondence auxiliary for direction or location \gls{wrt} the objects of interest \cite{brogan2021fast}. 
This can introduce global information to improve local correspondence. 
In particular, transformation optimization methods like RANSAC \cite{fischler1981random}, fast spatial measure (FSM) \cite{philbin2007object}, hough pyramid matching (HPM) \cite{avrithis2014hough} and pairwise geometric matching (PGM) \cite{li2015pairwise}, can filter out weak correspondences for keypoints or local features obtained by key feature detection and descriptors such as SIFT \cite{lowe2004distinctive}, SURF  \cite{bay2008speeded}, and ORB \cite{rublee2011orb}. 
The region-based or object-based verification methods such as Objects in Scene to Objects in Scene (OS2OS) \cite{brogan2021fast} and block-based image matching \cite{wang2021block}, make use of the relative positions of local patches to refine the whole image matching.
Different from the above approaches, our method uses distance-based spatial information for the neighborhood graph construction, rather than for transformation correction or weak correspondence filtering. 

Graph learning methods such as SuperGlue \cite{sarlin2020superglue}, GLMNet \cite{jiang2022glmnet}, and joint graph learning and matching network (GLAM) \cite{liu2023joint}, are exploited to represent local features based on the neighborhood graphs for keypoints. 
The graphs are constructed based on the detected keypoints or the corresponding features, and GNNs are used to learn graph representations.  
These methods achieve more robust and stable representations for the corresponding features based on spatial information.

Different from the above methods, our approach focuses on the \emph{neighborhood information based on landmark distances}, which is used for patch-level, rather than point-level, representation and not used to filter weak or invalid correspondences. 
Moreover, we also adopt GNNs to represent the patch-level neighborhood graphs, which is demonstrated to be beneficial for the landmark patch-matching task.  

\section{Landmark patch-matching with Graph-Based Learning}\label{sect:model}

In this section, we first introduce our graph-based learning framework to find matched landmark patch pairs that are extracted from two images taken from on-vehicle cameras. 
The images may be taken from different perspectives and our framework can also identify those patches that are unmatched. \cref{fig:figure_landmark_patch1} shows examples of matched and unmatched landmark patch pairs. 
We then discuss the theoretical basis for our graph-based learning approach.

\subsection{Framework Overview}\label{subsection.Framework}
Similar to other patch-matching datasets like the multi-view stereo (MVS) dataset \cite{winder2009picking} and the DTU dataset\cite{aanaes2012interesting}, in our work, the landmark patches are extracted from the full-sized images and the matching ground truths are established using $3$D points. More specifically, the landmark patches are extracted using well-known object detection techniques like Faster R-CNN \cite{ren2016faster}. To distinguish the full-sized images from the landmark patches, we use the term \textit{frame} to denote the full-sized image from which the patches are extracted. We refer the readers to \cref{sect:Datasets} for more details on the preparation of the landmark patch-matching datasets.

We assume that the spatial information (i.e., approximate relative distances between landmark objects) of landmark patches is available.
The spatial information can be obtained from range estimation methods like the monocular depth estimation networks \cite{zhou2021diffnet,lyu2020hrdepth,yan2021cadepth} in both the training and the testing phases.
To construct a graph, we let the landmark patches of a frame be vertices of the graph. For each patch or vertex $x$, we find the $K$ nearest neighbors in terms of spatial locations as indicated by the observed spatial information. An example of the constructed graph is shown in \cref{fig:figure_landmark_patch1}.
For the vertex $x$, we form a complete graph or clique with its $K$ nearest neighbors found. Let $\mathcal{G}^{x}$ denote this neighborhood graph. 

\begin{figure*}[!htb]
\begin{center}
\fbox{\includegraphics[width=0.98\linewidth, height=0.45\textheight]{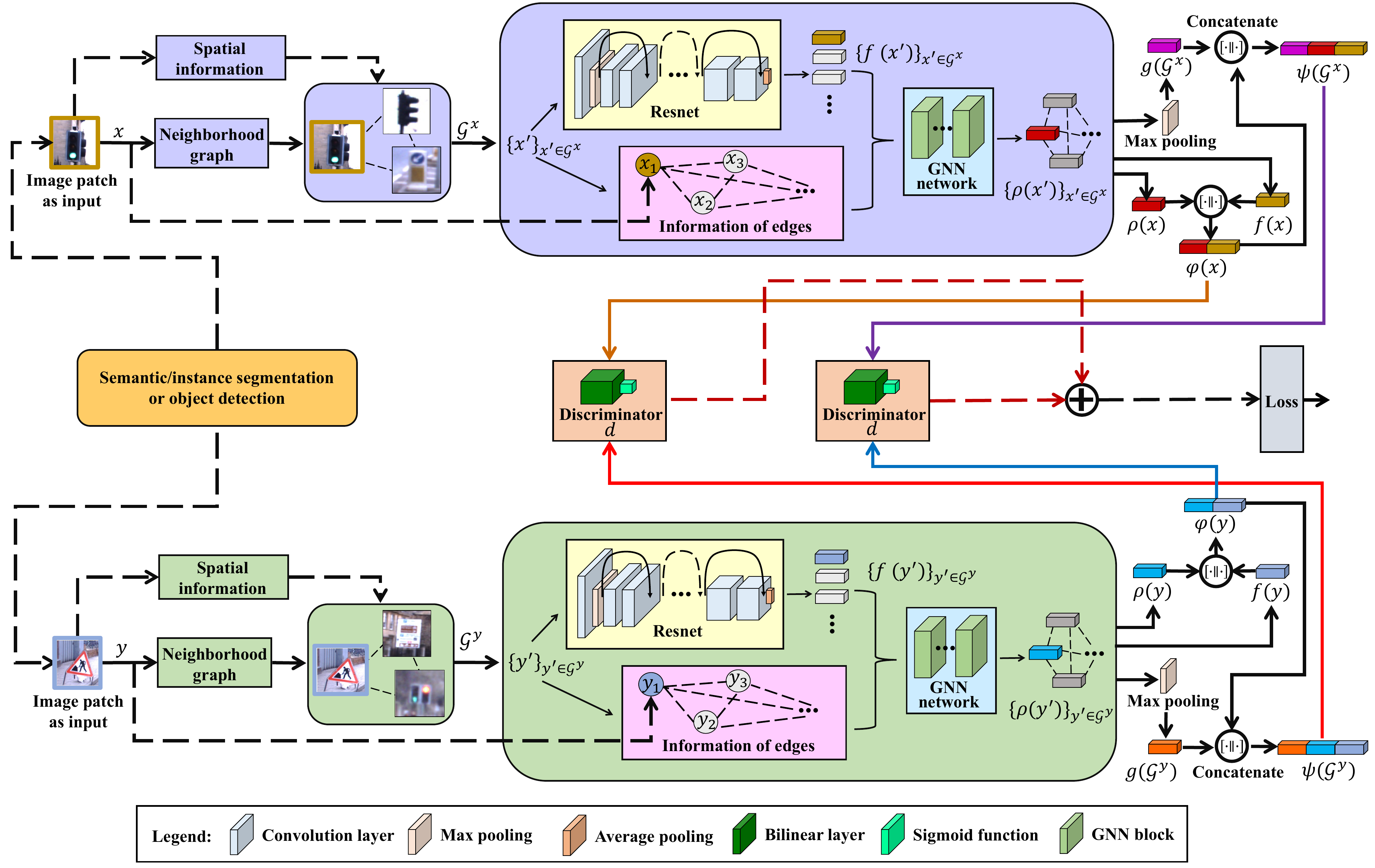}}
\end{center}
\vspace{-0.2cm}
\caption{VGIDM: landmark patch-matching with the graph-based learning. The Resnet $f$ shown in the framework is a shared network serving as the feature descriptor function $f$ to extract high-dimensional features from patches. 
Likewise, the discriminator $d$ is also shared to make a decision for the vertex-to-graph correspondence.
The model takes as input a pair of image patches that correspond to street scene landmarks.
}\label{fig:model}
\vspace{-0.3cm}
\end{figure*}

Our image patch-matching framework is illustrated in \cref{fig:model}.
In this framework, the inputs for the model are image patches obtained by semantic or instance segmentation methods, e.g., Mask R-CNN \cite{he2017mask}, or object detection methods, e.g., Faster R-CNN \cite{ren2016faster}. These two kinds of methods can extract objects of interest such as traffic lights and traffic signs from image frames. 
Two main modules, Resnet \cite{Resnet} and \gls{GNN}, are respectively used for image feature extraction and neighborhood graph embedding, where the GNN can be the graph attention network (GAT) \cite{velivckovic2017graph}, graph convolutional network (GCN) \cite{kipf2017semi}, GraphSAGE \cite{hamilton2017inductive} or any other GNN architecture. 
Given the vertex and graph embedding features from our model, we maximize the empirical information distance between the cases where patches are matched and unmatched.
We call our image patch-matching approach \textit{Vertex-Graph-learning-and-Information-Distance-Maximization (VGIDM)}.
The details are given as follows.

\subsection{Model Details}\label{sect:model_details}

Our objective is to determine if two landmark patches from different frames are matched with each other. In VGIDM, the feature extraction module $f$ first learns embeddings for the input landmark patch as well as the patches in its neighborhood graph. 
The model then makes use of a learnable graph embedding module $g$ to represent the neighborhood graph-level and vertex-level feature readout features. Finally, it uses a decision-making module to compute the matching classification.

\textbf{Feature extraction for patches.}
We use the Resnet $f$ to extract high-dimensional features for each landmark patch $x$. The Resnet output is denoted by $f(x)\in \mathbb{R}^n$. Recall that for a patch $x$, we form a neighborhood graph $\calG^x$. For the graph $\calG^x$, applying $f$ to each node in $\calG^x$, we have $\{f(x')\}_{x' \in \calG^x}$. 

\textbf{Embedding representation for the neighborhood graph and its vertices.}\label{section.regularization}
The graph $\mathcal{G}^{x}$ is input to a GNN network $g$ to obtain a graph-level embedding representation $g(\mathcal{G}^{x})$. 
Specifically, the vertex features $(f(x'))_{x' \in \calG^x} \in \mathbb{R}^{|\mathcal{G}^{x}| \times n}$ are updated via the GNN, which consists of several layers of neighborhood aggregation and node update \cite{velivckovic2017graph,wu2020comprehensive}, followed by some activation functions and a final pooling layer. The vertex embeddings $(\rho(x'))_{x' \in \calG^x} \in\mathbb{R}^{|\mathcal{G}^{x}| \times n}$ are obtained from the last graph convolutional/attentional layer of the GNN, while the graph-level embedding representation $g(\mathcal{G}^{x}) \in \mathbb{R}^n$ is obtained as the output of the last pooling layer.  

Compared to $f(x) \in \mathbb{R}^n$ which extracts features directly from the patch $x$, $\rho(x) \in \mathbb{R}^n$ learns a feature embedding with additional information from its neighborhood, while $g(\mathcal{G}^{x}) \in \mathbb{R}^n$ learns an embedding for the surrounding environment itself.

\textbf{Correspondence comparison.}
Suppose that $x$ and $y$ are landmark patches from two different frames, respectively. If $x$ and $y$ are patches for the same real-world object, we say that they are \emph{matched} and denote this event as $x\leftrightarrow y$. Otherwise, they are \emph{unmatched} and denoted as $x\nleftrightarrow y$.
For any patch pair $(x,y)$, we denote the matching ground truth label as $\indicator{x\leftrightarrow y}$, where $\indicator{\cdot}$ is the indicator function. 
In order to compare the correspondence between the patch pair $(x,y)$, 
we design a decision-making mechanism based on the patch features. 
For the two patches $x$ and $y$, we respectively obtain $f(x)$ and $f(y)$ as the features from the Resnet, $\rho(x)$ and $\rho(y)$ as the vertex-level embedding features, and  $g(\mathcal{G}^{x})$ and $g(\mathcal{G}^{y})$ as the graph-level embedding features from the GNN network. 

Let the ensemble vertex embedding for a patch $x$ be 
\begin{align}
\varphi(x)=\rho(x) \| f(x)
\end{align}
and the neighborhood graph embedding for $\calG^{x}$ be 
\begin{align}
\psi(\calG^{x})=g(\mathcal{G}^{x}) \| \varphi(x) = g(\mathcal{G}^{x}) \| \rho(x) \| f(x),
\end{align}
 where $\|$ is the concatenation operation. 

The ensemble vertex feature for $x$ and the graph embedding for $\calG^y$ are input to a discriminator $d$ consisting of a \emph{bilinear} layer of the form:
\begin{align}
d(a,b)= \sigma(a^{\top} \times M \times b), \label{eq:ddd}
\end{align}
where $M\in \Real^{n\times m}$ is a trainable matrix and $\sigma(\cdot)$ denotes the sigmoid function. 
In particular, the matrix $M$ is designed as 
\begin{align}
M = 
\begin{bmatrix}
\bf{0} & M_{12} & \bf{0} \\
M_{21} & M_{22} & M_{23} 
\end{bmatrix}, \label{eq:M_matrix}
\end{align}
where $M_{12}$, $M_{21}$, $M_{22}$ and $M_{23}$ serve as the  matrix blocks with learnable parameters and $\bf{0}$ denotes the zero matrix. The specifically designed block matrix \cref{eq:M_matrix} is to restrict the comparison between the features.
Inputting $(\varphi(x), \psi(\calG^{y}))$ to the discriminator $d$, we have 
\begin{align}
& d(\varphi(x), \psi(\calG^{y})) \nn
& = \sigma \left(
\begin{bmatrix}
\rho(x)^{\top} & f(x)^{\top}
\end{bmatrix}
\times M \times
\begin{bmatrix}
g(\mathcal{G}^{y}) \\
\rho(y) \\
f(y) \\
\end{bmatrix}
\right) \\
& = \sigma \bigg(\rho(x)^{\top}  M_{12} \rho(y) + f(x)^{\top} M_{21} g(\mathcal{G}^{y}) \nn
& \qquad + f(x)^{\top} M_{22} \rho(y) + f(x)^{\top} M_{23} f(y) \bigg). \label{eq.d_M}
\end{align}
The first term $\rho(x)^{\top}  M_{12} \rho(y)$ in \cref{eq.d_M} is used to compare the vertex embeddings of $x$ and $y$ obtained from the GNN. This emphasizes the domain part of the embedding. The second term $f(x)^{\top} M_{21} g(\mathcal{G}^{y})$ and third term $f(x)^{\top} M_{22} \rho(y)$ are used for the comparison of the vertex $x$ and the neighborhood graph of $y$. This helps to constrain GNN learning. The last term $f(x)^{\top} M_{23} f(y)$ is used to compare the Resnet features for the two vertices, which updates the Resnet training.
The same procedure is performed analogously for $\varphi(y)=\rho(y) \| f(y)$ and $\psi(\calG^{x})= g(\mathcal{G}^{x}) \| \varphi(x)$. 

The learnable discriminator $d(\varphi(x), \psi(\calG^{y}))$ from \cref{eq.d_M} utilizes the ensemble vertex embedding $\varphi(x)$ and the neighborhood graph embedding $\psi(\calG^{y})$. The vertex-level embedding $\rho(x)$ and graph-level embedding $g(\calG^x)$ contain information from the vertex feature $f(x)$ (output of Resnet) due to the incorporation of neighborhood information from the GNN. 
In the case of a large number of frames, the neighborhood graphs can be quite different as they typically consist of vertices from different frames. 
As a result, the embeddings $\rho(x), g(\calG^x)$ and $\rho(y), g(\calG^y)$ can have different features to some degree even if $x \leftrightarrow y$. 
Therefore, it may be appropriate to use the original vertex feature $f(x)$ to constrain the graph learning for the vertex-level embedding $\rho(y)$ and the graph-level embedding $g(\calG^y)$. 
The comparisons between $f(x)$ and $\rho(y)$ or $g(\calG^y)$ can emphasize the principal component for the learned graph features. 
When vertices $x$ and $y$ are matched, $\rho(y)$ and $g(\calG^y)$ essentially contain the information of $f(x)$.
Therefore, comparing $f(x)$ with $\rho(y)$ and $g(\calG^y)$ can introduce more information with neighborhood characteristics for the matching process.

\begin{figure*}[!tb]
\hrulefill
\begin{align}
    L_{\mathrm{empID}} 
    & = {\frac{1}{|\calM|}} \sum_{(x,y)\in \calM} 
    \bigg\{ \indicator{x \leftrightarrow y}
    \frac{1}{2}\Big(\log [d(\varphi(x),\psi(\mathcal{G}^{y}))] +  \log [d(\varphi(y),\psi(\mathcal{G}^{x}))]\Big) \nn
    & \qquad \quad
    + \indicator{x \nleftrightarrow y} 
    \frac{1}{2}\Big(\log [1-d(\varphi(x),\psi(\mathcal{G}^{y}))] + \log [1-d(\varphi(y),\psi(\mathcal{G}^{x}))] \Big)  \bigg\} \\
    & = \frac{1}{2} \underbrace{ \frac{1}{|\calM|} \sum_{(x,y)\in \calM} 
     \Big\{ \indicator{x \leftrightarrow y} 
     \log [d(\varphi(x),\psi(\mathcal{G}^{y}))] 
     + \indicator{x \nleftrightarrow {y}}
     \log [1-d(\varphi(x),\psi(\mathcal{G}^{y}))] \Big\} }_{L_{\mathrm{empID-1}}} \nn
    & \qquad + \frac{1}{2} \underbrace{\frac{1}{|\calM|} \sum_{(x,y)\in \calM} 
     \Big\{ \indicator{y \leftrightarrow {x}} 
     \log [d(\varphi(y),\psi(\mathcal{G}^{x}))] 
     + \indicator{y \nleftrightarrow {x}}
     \log [1-d(\varphi(y),\psi(\mathcal{G}^{x}))] \Big\} }_{L_{\mathrm{empID-2}}} \label{eq.Lce_ID} 
\end{align}
\hrulefill
\end{figure*}
\textbf{Loss function and matching score.}
Let $\calM$ be a training set consisting of patch pairs $(x,y)$.
Define the graph-based learning objective function as $L_{\mathrm{empID}}$ given in \cref{eq.Lce_ID}, which depends on the discriminator $d$ in \cref{eq.d_M}. We show that $L_{\mathrm{empID}}$ is the empirical version of an information distance between the distributions conditioned by matched and unmatched pairs in \cref{prop.KL_divergence_Max}. 
We set our overall loss as 
\begin{align}
\min_{\varphi, \psi, d} \{- L_{\mathrm{empID}}\}, \label{eq.L_loss}
\end{align}
to maximize the information distance.

In the testing phase, the final matching score is given by 
\begin{align}
    S_{\mathrm{match}}(x,y) & = \frac{d(\varphi(x),\psi(\calG^y))+d(\varphi(y),\psi(\calG^x))}{2},\label{eq:ssss}
\end{align}
and the prediction function for whether there is a match is given by 
\begin{align}
A^{\mathrm{test}}_{S}(x,y) =\left\{
\begin{array}{ll}
1, &\ \mathrm{if}\ S_{\mathrm{match}}(x,y) > \Gamma,\\
0, &\ \mathrm{otherwise}, 
\end{array}
\right. \label{eq.decision}
\end{align}
where $\Gamma$ is a predefined threshold. A decision ``$1$'' indicates that $x$ and $y$ are matched and ``$0$'' otherwise.

\subsection{Theoretical Basis}\label{section.theoretical_analysis}
In this subsection, we discuss the theoretical basis for the graph-based learning objective function $L_{\mathrm{empID}}$ defined in \cref{eq.Lce_ID}. To make the analysis tractable, we assume that patch pairs $(x,y)$ are randomly generated from a distribution $\P$. Let $\E$ be the expectation operator. We start with a simplifying assumption as follows. 

\begin{Assumption}\label{ass:ass_1}
$\varphi(x)$ and $\psi(\mathcal{G}^{y})$ are continuous random variables induced from $\P$.
\end{Assumption}
In practice, due to the chosen activation functions used in Resnet $f$ and the GNN network $g$, their outputs typically satisfy the continuity requirement of \cref{ass:ass_1}.

In our analysis, the discriminator $d$ is assumed to be a general function without necessarily having the form \cref{eq:ddd}. 

Let $\mathbb{A}$ be the set of all possible $(\varphi(x),\psi(\mathcal{G}^{y}))$ where $\int_{\mathbb{A}} p(\varphi(x),\psi(\mathcal{G}^{y})) {\rm d}(\varphi(x),\allowbreak \psi(\mathcal{G}^{y}))=1$, $p:\bbA \mapsto \bbR_{+}$ is a probability density whose set of discontinuities has Lebesgue measure zero. 

For any given landmark patches $x$ and $y$, we assume that $\P(x\leftrightarrow y)>0$ and $\P(x\nleftrightarrow y)>0$.
The probability densities of $(\varphi(x), \psi(\mathcal{G}^{y}))$ conditioned on $x \leftrightarrow y$ and $x \nleftrightarrow y$ are denoted by $p(\varphi(x),\psi(\mathcal{G}^{y})\mid{x \leftrightarrow y})$ and $p(\varphi(x),\psi(\mathcal{G}^{y})\mid{x \nleftrightarrow y})$, respectively.\footnote{
Here we abuse notations $p(\varphi(x),\psi(\mathcal{G}^{y})|{x \leftrightarrow y})$ to denote the conditional probability density of $(\varphi(x),\psi(\mathcal{G}^{y}))$ given that $x$ and $y$ are matched. This is to avoid the cluttered notation $p_{(\varphi(x),\psi(\mathcal{G}^{y}))\mid {x \leftrightarrow y}}(\cdot,\cdot)$.
} 

We discuss only $L_{\mathrm{empID-1}}$ in \cref{eq.Lce_ID} since $L_{\mathrm{empID-2}}$ is symmetrical to it. The expectation form of $L_{\mathrm{empID-1}}$ is given by
\begin{align}
    L_{\mathrm{ID}} 
    &= L_{\mathrm{ID}}(\varphi,\psi,d) \nn
    &= \E[\indicator{x \leftrightarrow y}\log d(\varphi(x),\psi(\mathcal{G}^y))] \nn
    & \quad + \E[\indicator{x\nleftrightarrow y}\log(1-d(\varphi(x),\psi(\mathcal{G}^y)))]. \label{eq.L_ID_expect}
\end{align}
In minimizing the loss in \cref{eq.L_loss}, in the asymptotic regime $|\calM|\to\infty$, we aim at  $\max\limits_{\varphi,\psi,d} L_{\mathrm{ID}}$. 
Let $\KLD{}{}$ denote the Kullback-Leibler (KL) divergence.

\begin{Proposition}[Relationship with KL divergence]
\label{prop.KL_divergence_Max}
Suppose \cref{ass:ass_1} holds. 
For a vertex embedding $\varphi$ and a neighborhood graph embedding $\psi$, 
let $L_{\mathrm{\rm ID}}^{d^*}(\varphi,\psi)= \max_d L_{\mathrm{ID}}(\varphi,\psi,d)$, where $d^*$ is the corresponding optimal discriminator. Then 
\begin{align}
& \KLD*{p(\varphi(x),\psi(\mathcal{G}^{y})\mid{x \leftrightarrow y})}{p(\varphi(x),\psi(\mathcal{G}^{y})\mid{x \nleftrightarrow y})} \\
& 
\geq \ofrac{\P(x\leftrightarrow y)} \parens*{L_{\mathrm{ID}}^{d^*}(\varphi,\psi) + H_b(\P(x\leftrightarrow y))}.
\end{align}
where $H_b(p)=-p\log p - (1-p)\log(1-p)$ is the binary entropy function.
\end{Proposition}
\begin{proof}
See Appendix~\ref{app.proposition_1}.
\end{proof}

\begin{Remark}
{\cref{prop.KL_divergence_Max} suggests that maximizing $L_{\text{ID}}$ over $(\varphi,\psi,d)$ helps to distinguish between the matched and unmatched patch pairs since their conditional distributions are forced to be very different in terms of the KL divergence.}
\end{Remark}

We next consider how the graph-based learning objective function $L_\mathrm{ID}$ in \cref{eq.L_ID_expect} is influenced by perturbations in the discriminator $d$.

\begin{Proposition}[Effect of discriminator perturbation]
\label{prop.disturbance_discriminator}
{Suppose \cref{ass:ass_1} holds. Let $\varepsilon$ be a sufficiently small perturbation to the discriminator $d$.} Then, $|L_{\mathrm{ID}}(\varphi,\psi, d+\varepsilon)- L_{\mathrm{ID}}(\varphi,\psi,d)| =  O(\varepsilon)$.
Furthermore, we have $|\max_{d} L_{\mathrm{ID}}(\varphi,\psi, d+\varepsilon) -  \max_d \allowbreak L_{\mathrm{ID}}(\varphi,\psi,d)| =  O(\varepsilon^2).$
\end{Proposition}

\begin{proof}
See Appendix~\ref{app.proposition_disturbance}.
\end{proof}

In the following, we consider how the GNN embedding of the neighborhood graph $\calG^x$ of a vertex $x$ affects the matching effectiveness under further assumptions.

For two landmark patches $x$ and $y$, if their neighborhood graphs $\calG^x$ and $\calG^y$ have vertices corresponding to the same set of objects, i.e., the patch and spatial information procedure identifies the same objects as the neighbors of $x$ and $y$, we write $\calG^x \leftrightarrow \calG^y$.

\begin{Assumption}\label{ass.case1}
The ranges of $\varphi(\cdot)$ and $\psi(\cdot)$ are finite sets.
The embedding $\varphi(x)=\varphi(y)$ for landmark patches $x$ and $y$ are the same if $x\leftrightarrow y$. 
If furthermore $\calG^x \leftrightarrow \calG^y$, then $\psi(\calG^x)=\psi(\calG^y)$.
\end{Assumption}

While the Resnet $f$ and GNN block $g$ are in general continuous functions of their inputs, \cref{ass.case1} can be satisfied by restricting to a finite number of objects of interest in the environment, assuming that frames are captured from approximately the same perspectives (e.g., from an on-vehicle camera of a vehicle traveling along a fixed road) so that landmark patches of the same object are within a certain similarity distance of each other. Finally, the outputs of $f$ and $g$ can be quantized into discrete ranges, which implies $\varphi$ and $\psi$ have finite sets of ranges. 
For the same object $o$ in the environment but under two different frames $\calF_1$ and $\calF_2$, 
\cref{ass.case1} says that the outputs from the embedding $\varphi$ are the same for the two frames. This implicitly assumes that $\varphi$ is robust to perturbation in its input. 
Furthermore, the outputs of the embedding $\psi$ are also the same if the patch and spatial information are noiseless.  

\begin{Proposition}\label{prop.proposition_3}
Suppose \cref{ass.case1} holds, and $x$ and $y$ are landmark patches of frames $\calF_1$ and $\calF_2$ (based on the same environment), respectively. 
Let $m(x\leftrightarrow y) = \P(\calG^x \leftrightarrow \calG^y \given x \leftrightarrow y)$ and $m(x \nleftrightarrow y) = \P(\calG^x \leftrightarrow \calG^y \given x \nleftrightarrow y)$. 
Then we have
\begin{align}
&\norm{p(\varphi(x), \psi(\mathcal{G}^{y}) \mid x \leftrightarrow y) - p(\varphi(x), \psi(\mathcal{G}^{y}) \mid x \nleftrightarrow y)}_{\mathrm{TV}} \nn
& \geq \min_{x\leftrightarrow y}m(x\leftrightarrow y) \sum_{\mathclap{(\varphi(x), \psi(\calG^{y})) \in \bbB}}\Big\{\pfgGm{\Gm} \nn
& \qquad\qquad\qquad\qquad  - \pfgGu{\Gu} \Big\} \nn
&\qquad + \min_{x\leftrightarrow y}m(x\leftrightarrow y) - \max_{x \nleftrightarrow y} m(x\nleftrightarrow y) - 1, \label{ineq:prop3}
\end{align}
where $\bbB = \big\{(\varphi(x), \psi(\calG^{y})) : \pfgm \ge \pfgu \big\}$.
Here, $\norm{\cdot}_{\mathrm{TV}}$ denotes the total variation distance, and $\min_{x\leftrightarrow y}$ and $\max_{x \nleftrightarrow y}$ denotes minimization over all matched patch pairs $(x,y)$ and maximization over all unmatched patch pairs, respectively.
\end{Proposition}

\begin{proof}
See Appendix~\ref{app.proposition_3}.
\end{proof}

In the ideal case where the patch and spatial information are noiseless, we have $\min_{x \leftrightarrow y} m(x \leftrightarrow y)=1$ and $\max_{x \nleftrightarrow y} m(x \nleftrightarrow y)=0$. Then the right-hand side of \cref{ineq:prop3} simplifies to 
\begin{align}
&\sum_{(\varphi(x), \psi(\calG^{y})) \in \bbB}\Big\{\pfgGm{\Gm} \nn
&\qquad - \pfgGu{\Gu} \Big\}. \label{TV2}
\end{align}
In this case, we also have $\pfgm = \pfgGm{\Gm}$ and $\pfgu = \pfgGu{\Gu}$. Furthermore, from \cref{ass.case1}, any $(\varphi(x),\psi(\calG^y))$ such that $\pfgGm{\Gm} >0$ implies that $\pfgGu{\Gu}=0$. These probability measures are thus mutually singular and have a total variation distance of $1$.
Therefore, in the ideal case, the model perfectly distinguishes between $p(\varphi(x),\psi(\mathcal{G}^{y}) \mid {x \leftrightarrow y})$ and $p(\varphi(x),\psi(\mathcal{G}^{y}) \mid {x \nleftrightarrow y})$.

\section{Experiments}\label{sect:exper}

\subsection{Datasets}\label{sect:Datasets}

As there are no existing standard datasets for street-scene landmark patch-matching, we introduce in this paper two new datasets: the Landmark KITTI dataset and the Landmark Oxford dataset,\footnote{https://github.com/AI-IT-AVs/Landmark\_patch\_datasets} which are derived from the street-scene KITTI dataset \cite{geiger2012we} and the Oxford Radar RobotCar dataset \cite{barnes2020oxford}, respectively. 

Both datasets contain image frames and LiDAR scans captured from onboard cameras and Velodyne LiDAR sensors. The landmark patches are extracted from the full-sized image frames using the object detection neural network Faster R-CNN \cite{ren2016faster}. To facilitate detection efficacy, we manually label several street-scene compact landmark objects including traffic lights, traffic signs, poles, and facade windows for the sampled frames. The labels are used to train Faster R-CNN, which is used to produce landmark object detection for the image frames. The detected landmarks in bounding boxes are then used to obtain the landmark patches for our matching experiments with some intentionally included background, shown in \cref{fig_image_patches_pair_dataset} for example.

\begin{figure}[!htb]
\centering
\begin{center}
\includegraphics[width=0.85\linewidth]{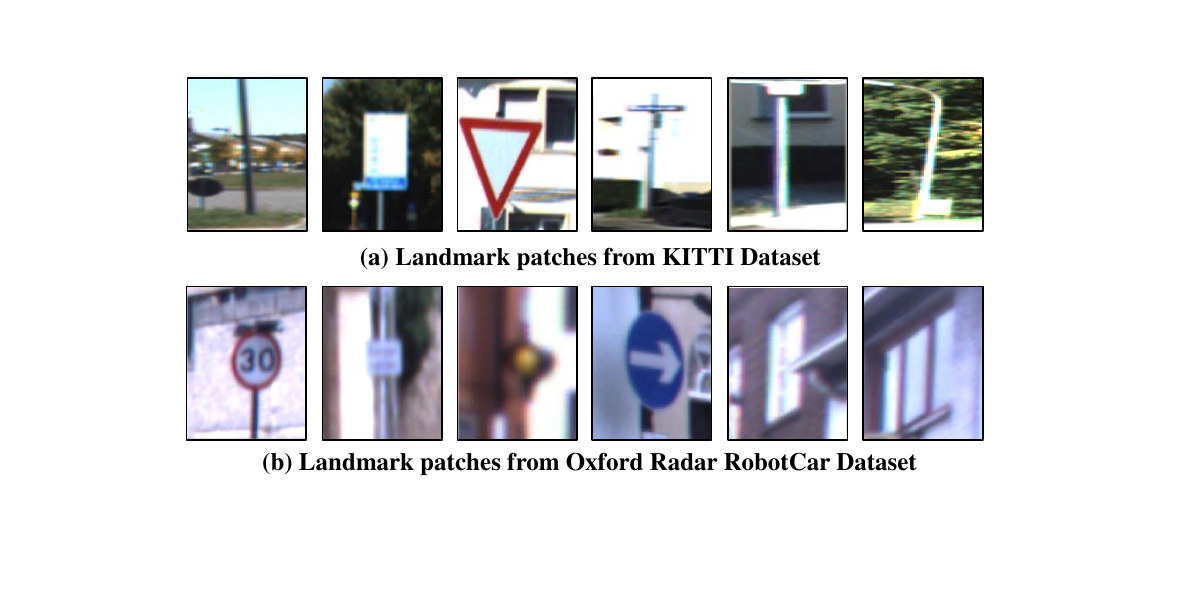}
\end{center}
\vspace{-0.5cm}
\caption{(a) and (b) are landmark patch samples (displayed with intentionally included background) from the KITTI dataset and Oxford Radar RobotCar dataset respectively.}
\label{fig_image_patches_pair_dataset}
\vspace{-0.1cm}
\end{figure}

To establish the patch-matching ground truth, we use the vehicle locations and collected LiDAR scans to build the $3$D LiDAR reference map similar to the operations in \cite{uy2018pointnetvlad}. The $3$D reference map is used to determine the landmark locations by projecting the $3$D LiDAR points to the image frames. The LiDAR points reflected from the landmark patch are read out to get the global locations of the corresponding landmark objects. We then compute the $\mathcal{L}_2$ distance of each landmark patch pair from two frames to determine the patch-matching ground truth. Some details of the two landmark patch-matching datasets are introduced as follows. More dataset preparation details are given in the supplementary material. 

\textbf{Landmark KITTI Dataset.}
The KITTI dataset\footnote{http://www.cvlibs.net/datasets/kitti} contains street-scene image frames and their corresponding LiDAR point clouds collected in Karlsruhe, Germany. We use the object labels provided by \cite{Alhaija2018IJCV} to detect landmark patches for all frames including traffic lights, traffic signs and poles. An example is shown in \cref{fig:label_kitti_dataset}. We do not include windows as landmarks in this dataset due to the lack of labels. Furthermore, to avoid ``trivial matchings'' between consecutive images, a minimum difference of $2 \mathrm{m}$ between the image frames is also set. 
The aforementioned operations are performed to obtain the landmark patch-matching ground truth by projecting the $3$D LiDAR scans to the image frames. Finally, $1500$ frames are selected for landmark patch-matching experiments. The dataset is randomly split into training and testing sets, with a ratio around $2:1$. In both training and testing, we select frame pairs that are captured at locations with relative distances not more than $25 \mathrm{m}$  to ensure the presence of common landmarks.

\begin{figure}[!htb]
\centering
\includegraphics[width=\linewidth, height=0.08\textheight]{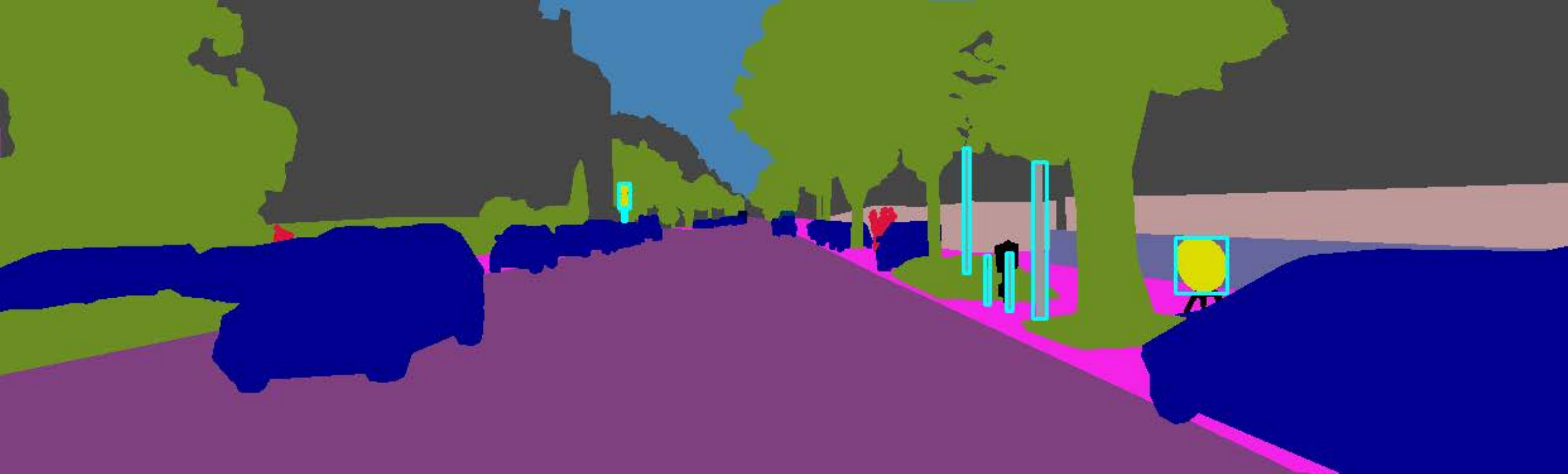}\hfill
\includegraphics[width=\linewidth, height=0.08\textheight]{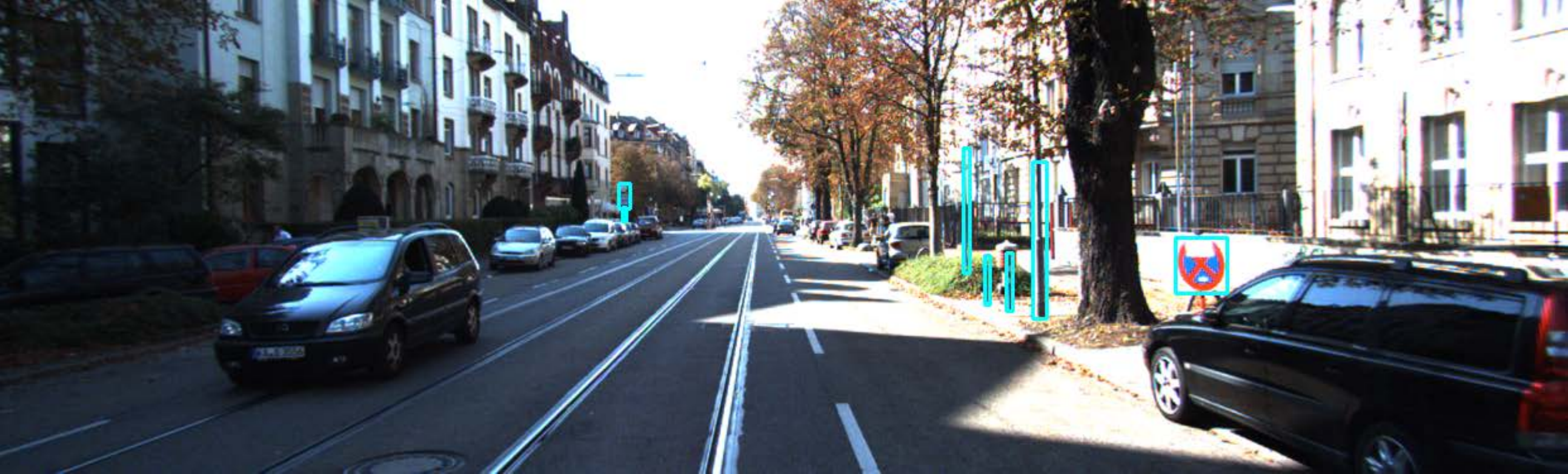}\hfill
\vspace{-0.2cm}
\caption{A semantic segmentation image and its corresponding real image, both with bounding box labels, from the KITTI dataset.}
\label{fig:label_kitti_dataset}
\vspace{-0.1cm}
\end{figure}

\textbf{Landmark Oxford Dataset.}
The Oxford Radar RobotCar dataset\footnote{http://ori.ox.ac.uk/datasets/radar-robotcar-dataset} contains image frames and LiDAR scans captured on the streets in Oxford, UK. We manually label landmarks including traffic lights, traffic signs, poles, and facade windows for $500$ sampled frames. An example is shown in \cref{fig:label_oxford_dataset}. We then train Faster R-CNN to obtain the landmarks for all $29,687$ frames. To avoid ``trivial matchings'' between consecutive images, a minimum difference of $2 \mathrm{m}$ between the image frames is also set. Finally, $3000$ frames are selected for landmark patch-matching experiments. The remaining steps are similar to that for the Landmark KITTI dataset.

\begin{figure}[!htb]
\centering
\includegraphics[width=0.5\linewidth, height=0.12\textheight]{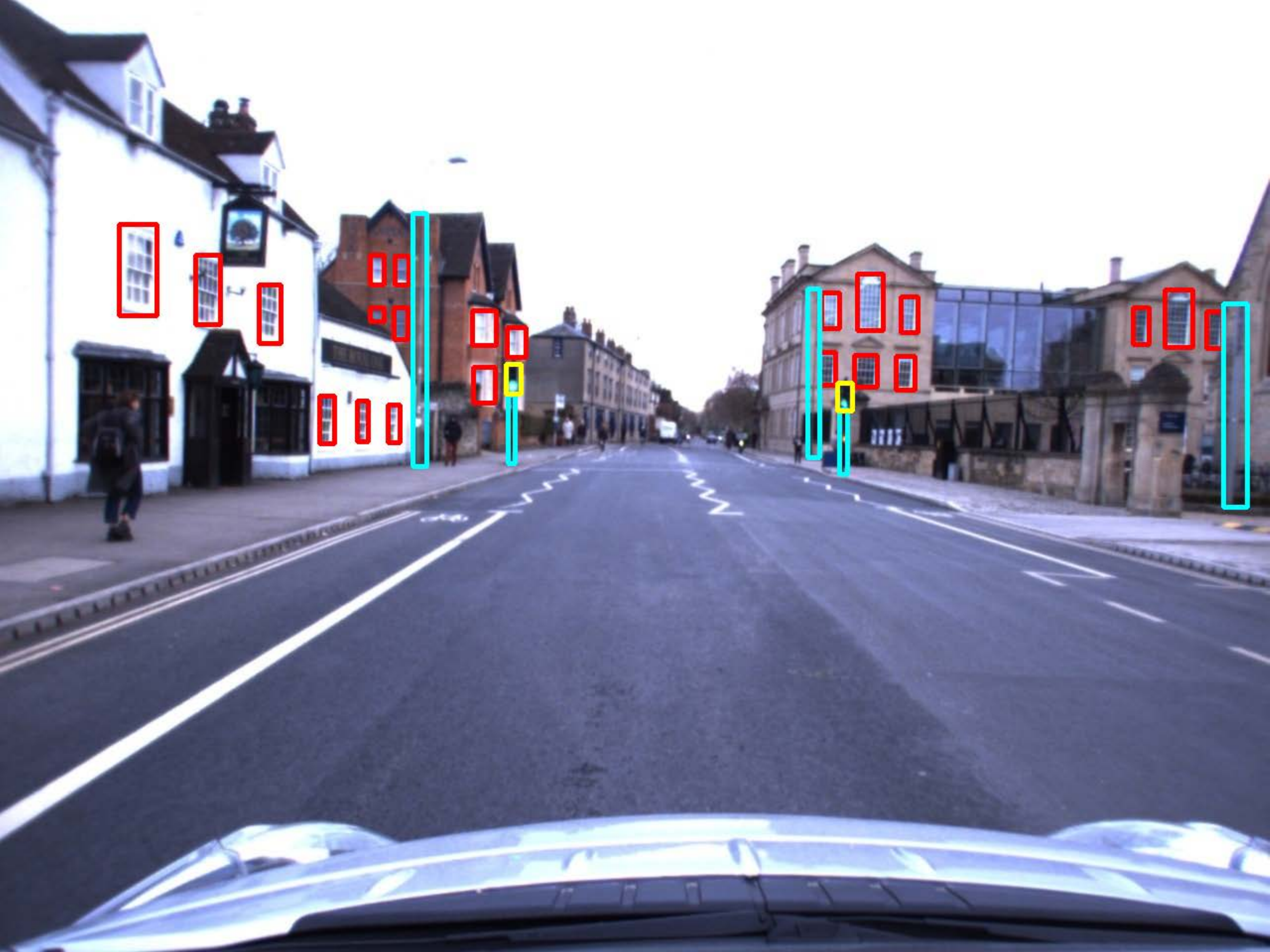}\hfill
\includegraphics[width=0.5\linewidth, height=0.12\textheight]{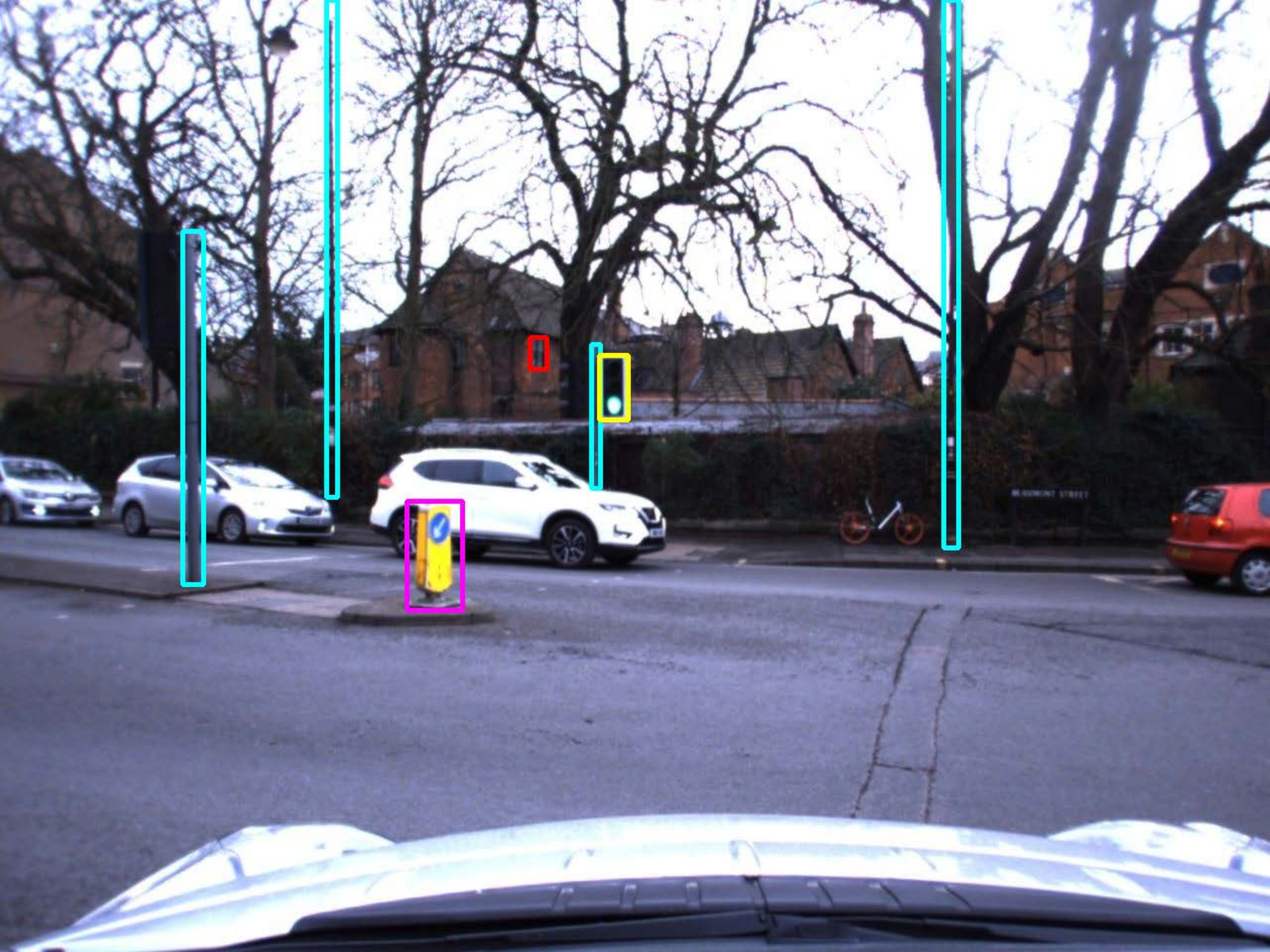}\hfill
\vspace{-0.2cm}
\caption{Examples of the ground truth landmark bounding box labels for the Oxford Radar RobotCar dataset.}
\label{fig:label_oxford_dataset}
\vspace{-0.3cm}
\end{figure}

\subsection{Experimental Details}\label{sect:experiment} 
\textbf{Baseline Methods.}
We compare VGIDM with several baseline methods, including MatchNet \cite{han2015matchnet}, SiameseNet \cite{melekhov2016siamese}, HardNet \cite{mishchuk2017working}, SOSNet \cite{tian2019sosnet}, D2-Net \cite{dusmanu2019d2}, ASLFeat \cite{luo2020aslfeat}, SuperGlue \cite{sarlin2020superglue} and LoFTR \cite{sun2021loftr}.
The MatchNet and SiameseNet are regarded as joint feature and metric learning methods, combining deep CNNs and an FC layer to learn features and their metrics. The decision-making process for the matching task is based on the output of the FC layer. 
HardNet and SOSNet focus on similarity measures to distinguish the learned high-dimensional features, where the feature descriptors are almost all based on deep CNNs consisting of several convolution layers with batch normalization (BN) or rectified linear units (ReLUs).
In testing, the Euclidean distance between the output patch features is used for the decision-making. 
D2-Net, ASLFeat, SuperGlue and LoFTR are based on keypoint correspondence and perform the matching task according to the ratio of the matched keypoints among the whole set of keypoints. In this regard, a patch pair with a large enough ratio of matched keypoints is regarded as a match.

\begin{table*}[!htb]
\caption{Matching performance on the Landmark KITTI dataset. The best and the second-best result for each criterion are highlighted in \red{red} and \blue{blue} respectively.} 
\label{Kitti-table}
\vspace{-0.3cm}
\begin{center}
\newcommand{\tabincell}[2]{\begin{tabular}{@{}#1@{}}#2\end{tabular}}
\begin{tabular}{c|c|c|c|c}
\hline\hline
{\bf Methods}  & {\bf Precision}
& {\bf Recall} & {\bf $F_1$-Score}
& {\bf AUC} \\
\hline 
MatchNet \cite{han2015matchnet}               & 0.9039 \scriptsize{$\pm$ 0.0027}  & 0.9483 \scriptsize{$\pm$ 0.0105} & 0.9255 \scriptsize{$\pm$ 0.0050} & 0.8229 \scriptsize{$\pm$ 0.0055} \\ 
SiameseNet \cite{melekhov2016siamese}         & 0.7953 \scriptsize{$\pm$ 0.0124}  & 0.8960 \scriptsize{$\pm$ 0.0208} & 0.8426 \scriptsize{$\pm$ 0.0159} & 0.8328 \scriptsize{$\pm$ 0.0162} \\ 
HardNet \cite{mishchuk2017working}            & 0.9041 \scriptsize{$\pm$ 0.0016}  & 0.9562 \scriptsize{$\pm$ 0.0177} & 0.9294 \scriptsize{$\pm$ 0.0093} & 0.8261 \scriptsize{$\pm$ 0.0088} \\ 
SOSNet  \cite{tian2019sosnet}                 & 0.9042 \scriptsize{$\pm$ 0.0015}  & 0.9563 \scriptsize{$\pm$ 0.0160} & 0.9294 \scriptsize{$\pm$ 0.0083} & 0.8261 \scriptsize{$\pm$ 0.0080} 
\\
D2-Net \cite{dusmanu2019d2}                   & 0.9115 \scriptsize{$\pm$ 0.0031}  & 0.8789 \scriptsize{$\pm$ 0.0131} & 0.8949 \scriptsize{$\pm$ 0.0076} & 0.8115 \scriptsize{$\pm$ 0.0082} \\ 
ASLFeat \cite{luo2020aslfeat}                 & 0.9189 \scriptsize{$\pm$ 0.0022}  & 0.9008 \scriptsize{$\pm$ 0.0082} & 0.9098 \scriptsize{$\pm$ 0.0048} & 0.8312 \scriptsize{$\pm$ 0.0057} \\ 
SuperGlue \cite{sarlin2020superglue}          & 0.9067 \scriptsize{$\pm$ 0.0039}  & 0.9125 \scriptsize{$\pm$ 0.0123} & 0.9096 \scriptsize{$\pm$ 0.0072} & 0.8155 \scriptsize{$\pm$ 0.0093} \\ 
LoFTR \cite{sun2021loftr}                     & 0.9069 \scriptsize{$\pm$ 0.0025}  & 0.9243 \scriptsize{$\pm$ 0.0110} & 0.9154 \scriptsize{$\pm$ 0.0059} & 0.8197 \scriptsize{$\pm$ 0.0064} \\ 
\hline 
\tabincell{c}{VGIDM (GAT) [ours] \\}
                   & \blue{0.9425} \scriptsize{$\pm$ 0.0020}   & \blue{0.9733} \scriptsize{$\pm$ 0.0050}        & \red{0.9577} \scriptsize{$\pm$ 0.0026}       & \blue{0.8977} \scriptsize{$\pm$ 0.0038} \\ 
\tabincell{c}{VGIDM (GCN) [ours]\\}
                   & 0.9340 \scriptsize{$\pm$ 0.0027}  & \red{0.9753} \scriptsize{$\pm$ 0.0076} & 0.9543 \scriptsize{$\pm$ 0.0044} & 0.8847 \scriptsize{$\pm$ 0.0063} \\ 
\tabincell{c}{VGIDM (GraphSAGE) [ours]\\}
                   & \red{0.9464} \scriptsize{$\pm$ 0.0042}         & 0.9653 \scriptsize{$\pm$ 0.0129}  & \blue{0.9557} \scriptsize{$\pm$ 0.0073} & \red{0.9007} \scriptsize{$\pm$ 0.0098} \\
\hline\hline
\end{tabular}
\end{center}
\vspace{-0.3cm}
\end{table*}

\begin{figure*}[!htb]
\begin{center}
\includegraphics[width=\linewidth]{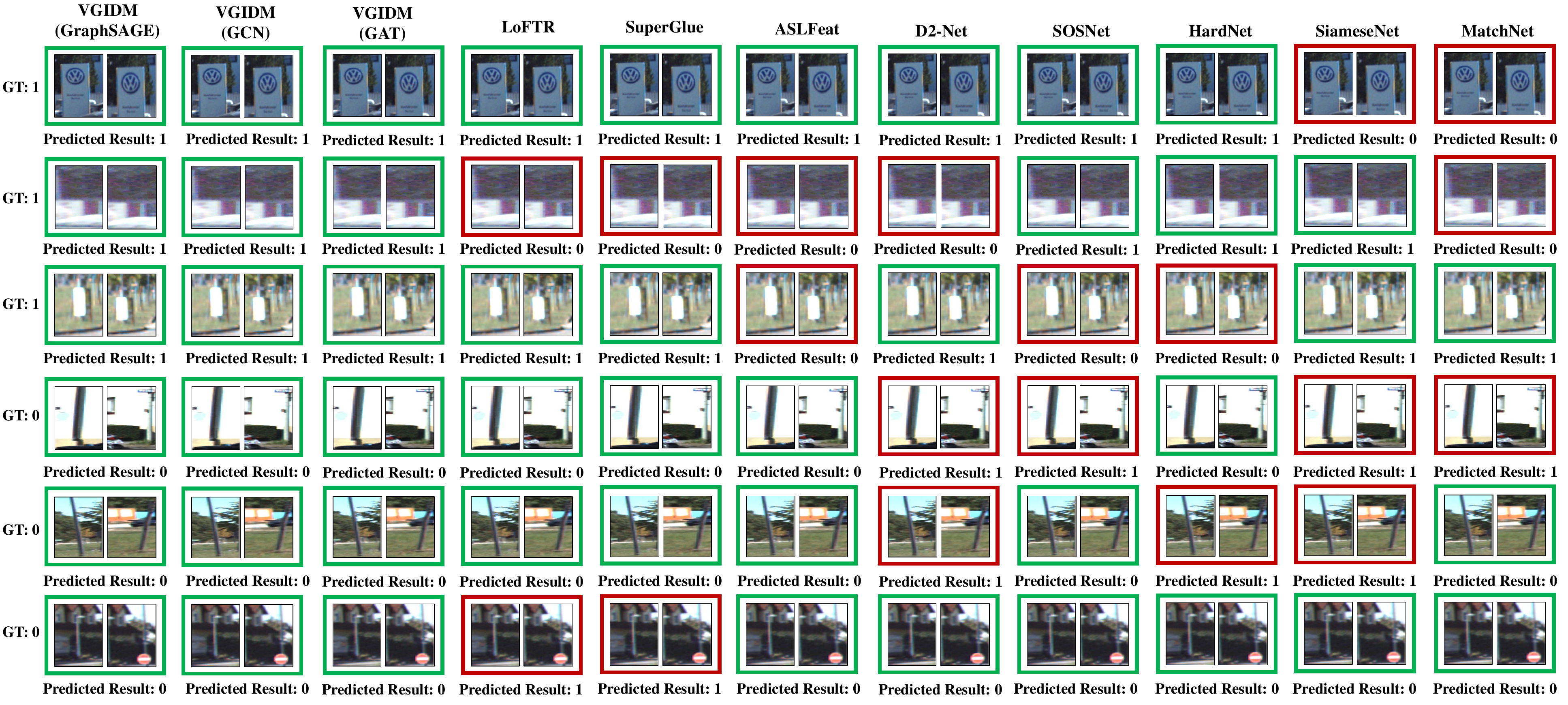}
\end{center}
\vspace{-0.3cm}
\caption{Examples of matched and mismatched pairs from the Landmark KITTI dataset. A green or red box indicates a correct or incorrect prediction result, respectively. ``GT'' stands for ground truth.}
\label{fig_kitti_results}
\vspace{-0.3cm}
\end{figure*}

\textbf{Model Setting.}
We use Resnet$18$ in \cite{Resnet} for the feature descriptor {$f$}, with output feature dimension $512$ after $17$ convolution layers.
In VGIDM, we choose several \gls{GNN}s for the neighborhood graph embedding, including GAT \cite{velivckovic2017graph}, GCN \cite{kipf2017semi} and GraphSAGE \cite{hamilton2017inductive}. 
When using GAT, the network contains $2$ GAT blocks with the exponential linear unit (ELU).
For each GAT block, we use $4$ attention heads, which compute $512$ hidden features in total. 
As for GCN and GraphSAGE, they both contain $2$ corresponding blocks with ReLU, where there are $512$ hidden features in each block.
Further details of our model architecture are provided in the supplementary material.
The Adam optimizer is selected with a learning rate of $0.0001$ to train the model by minimizing its corresponding loss in \cref{eq.L_loss}.
The number of training epochs is $150$ for all datasets.

\textbf{VGIDM with Image Depth Estimation.} 
To test VGIDM in the case where precise depth information like that provided by LiDAR is unavailable, we construct neighborhood graphs using estimated image depth and with different \gls{GNN}s in the backbone.
Specifically, we include an image depth estimation method called Monocular Depth Prediction Module proposed in \cite{yin2021learning}. 
Based on the image depth estimation, we can obtain the rough relative locations of the landmarks in the street scenes and use them to construct a neighborhood graph for each landmark in the test procedure.  
The depth estimation performance is provided in the supplementary material. 
The estimated image pixel depths are transformed to $3$D locations \gls{wrt} the camera using its intrinsic matrix. We then use the estimated locations to test VGIDM. In this depth estimation method, the pre-trained ResNeXt$101$ model from \cite{yin2021learning} is utilized in our experiments, and the images are from the two landmark datasets. 
We extract the predicted depth points from the static roadside landmarks, including traffic lights, traffic signs, and poles, to compute the locations of the objects. 
Therefore, we can construct the neighborhood graphs and test the VGIDM. 

\textbf{Implementation.} 
For a given sequence of street scene frames captured by a vehicular camera, we perform the following training steps: 
i) We use object detection methods like faster R-CNN \cite{ren2016faster} to extract landmark patches for each frame. The landmarks include traffic lights, traffic signs, poles, and windows. 
ii) We manually label matching landmark patches. To determine the global locations of these landmarks, we combine vehicular Global Positioning System (GPS) information with data from LiDAR or stereo cameras. With this information, we are able to establish the ground truth for the matching landmark patches between two frames captured at the same location. 
iii) We take the global locations of landmarks to construct the neighborhood graph for each landmark patch based on $K$-NN. 
iv) We train the VGIDM using landmark patch pairs with ground-truth labels. 
The details of VGIDM with training loss and test score are given in \cref{sect:model_details}. 

During testing, we perform steps i and iii as above but in step iii, we create neighborhood graphs by estimating the relative locations of landmarks using a stereo camera or a depth estimation method, which replaces the need for GPS and LiDAR information. The ground truth for computing the testing performance is found based on GPS and LiDAR information. The remaining steps are the same as those used during training.

\subsection{Performance Evaluation}\label{sec.Performance_Evaluation}\ \par

\textbf{Performance on Landmark KITTI Dataset.} 
\cref{Kitti-table} summarizes the test performance of models trained with $150$ training epochs on the Landmark KITTI dataset.
The evaluation uses statistics including mean value and standard deviation from $5$ experiments.
From \cref{Kitti-table}, we observe that VGIDM (with GAT, GCN or GraphSAGE) outperforms the other baseline methods under all four criteria, with a slight performance difference among these VGIDM models.
This implies that graph-based learning makes a positive difference in matching efficiency.
Moreover, we observe that VGIDM with GAT has a more stable performance than the other methods.
Several examples of the matching prediction are shown in \cref{fig_kitti_results}.

\textbf{Performance on Landmark Oxford Dataset.} 
From \cref{Oxford-table}, we observe that the VGIDM variants with different GNNs outperform the other benchmark methods on almost all measures.
Since the Oxford Radar RobotCar and KITTI datasets have different image qualities and are collected in different street scenes, the performances on both datasets are different.
From \cref{Kitti-table,Oxford-table}, we also observe that nearly all the methods have better performance on the Landmark KITTI dataset compared with the Landmark Oxford dataset.
This may be caused by more similarity among the window patches in the Landmark Oxford dataset, which makes distinguishing them more difficult.
A few matching prediction examples from the Landmark Oxford dataset are shown in \cref{fig_oxford_results}.

\begin{table*}[!htb]
\caption{Matching performance on the Landmark Oxford dataset. The best and the second-best result for each criterion are highlighted in \red{red} and \blue{blue} respectively. }
\label{Oxford-table}
\vspace{-0.3cm}
\begin{center}
\newcommand{\tabincell}[2]{\begin{tabular}{@{}#1@{}}#2\end{tabular}}
\begin{tabular}{c|c|c|c|c}
\hline\hline  
{\bf Methods}  & {\bf Precision}
& {\bf Recall} & {\bf $F_1$-Score}
& {\bf AUC} \\  
\hline 
MatchNet \cite{han2015matchnet}              & 0.8742 \scriptsize{$\pm$ 0.0068}  & 0.9589 \scriptsize{$\pm$ 0.0047} & 0.9146 \scriptsize{$\pm$ 0.0025} & 0.7723 \scriptsize{$\pm$ 0.0119} \\ 
SiameseNet \cite{melekhov2016siamese}        & 0.7210 \scriptsize{$\pm$ 0.0086}  & 0.8968 \scriptsize{$\pm$ 0.0109} & 0.7992 \scriptsize{$\pm$ 0.0076} & 0.7748 \scriptsize{$\pm$ 0.0088} \\ 
HardNet \cite{mishchuk2017working}           & 0.8762 \scriptsize{$\pm$ 0.0011}  & 0.9533 \scriptsize{$\pm$ 0.0094} & 0.9131 \scriptsize{$\pm$ 0.0048} & 0.7747 \scriptsize{$\pm$ 0.0047} \\ 
SOSNet \cite{tian2019sosnet}                 & 0.8763 \scriptsize{$\pm$ 0.0010}  & 0.9544 \scriptsize{$\pm$ 0.0086} & 0.9136 \scriptsize{$\pm$ 0.0044} & 0.7752 \scriptsize{$\pm$ 0.0043} \\ 
D2-Net \cite{dusmanu2019d2}                  & 0.8032 \scriptsize{$\pm$ 0.0033}  & 0.9005 \scriptsize{$\pm$ 0.0084} & 0.8491 \scriptsize{$\pm$ 0.0052} & 0.6194 \scriptsize{$\pm$ 0.0078} \\ 
ASLFeat \cite{luo2020aslfeat}                & 0.8729 \scriptsize{$\pm$ 0.0036}  & 0.9048 \scriptsize{$\pm$ 0.0073} & 0.8886 \scriptsize{$\pm$ 0.0035} & 0.7548 \scriptsize{$\pm$ 0.0062} \\ 
SuperGlue \cite{sarlin2020superglue}         & 0.8639 \scriptsize{$\pm$ 0.0054}  & 0.8747 \scriptsize{$\pm$ 0.0077} & 0.8692 \scriptsize{$\pm$ 0.0049} & 0.7305 \scriptsize{$\pm$ 0.0099} \\ 
LoFTR  \cite{sun2021loftr}                   & 0.8515 \scriptsize{$\pm$ 0.0020}  & \red{0.9837} \scriptsize{$\pm$ 0.0060} & 0.9129 \scriptsize{$\pm$ 0.0029} & 0.7346 \scriptsize{$\pm$ 0.0045} \\ 
\hline 
\tabincell{c}{VGIDM (GAT) [ours] \\}
                   & \red{0.9052} \scriptsize{$\pm$ 0.0047}  & 0.9517 \scriptsize{$\pm$ 0.0044} & \blue{0.9278} \scriptsize{$\pm$ 0.0040} & \red{0.8263} \scriptsize{$\pm$ 0.0092}\\ 
\tabincell{c}{VGIDM (GCN) [ours]\\}
                   & 0.8918 \scriptsize{$\pm$ 0.0051}  & 0.9515 \scriptsize{$\pm$ 0.0077} & 0.9206 \scriptsize{$\pm$ 0.0037} & 0.8025 \scriptsize{$\pm$ 0.0087} \\ 
\tabincell{c}{VGIDM (GraphSAGE) [ours]\\}
                   & \blue{0.8938} \scriptsize{$\pm$ 0.0046} & \blue{0.9648} \scriptsize{$\pm$ 0.0052} & \red{0.9279} \scriptsize{$\pm$ 0.0045} & \blue{0.8104} \scriptsize{$\pm$ 0.0096} \\
\hline\hline 
\end{tabular}
\end{center}
\vspace{-0.3cm}
\end{table*}

\begin{figure*}[!htb]
\begin{center}
\includegraphics[width=\linewidth]{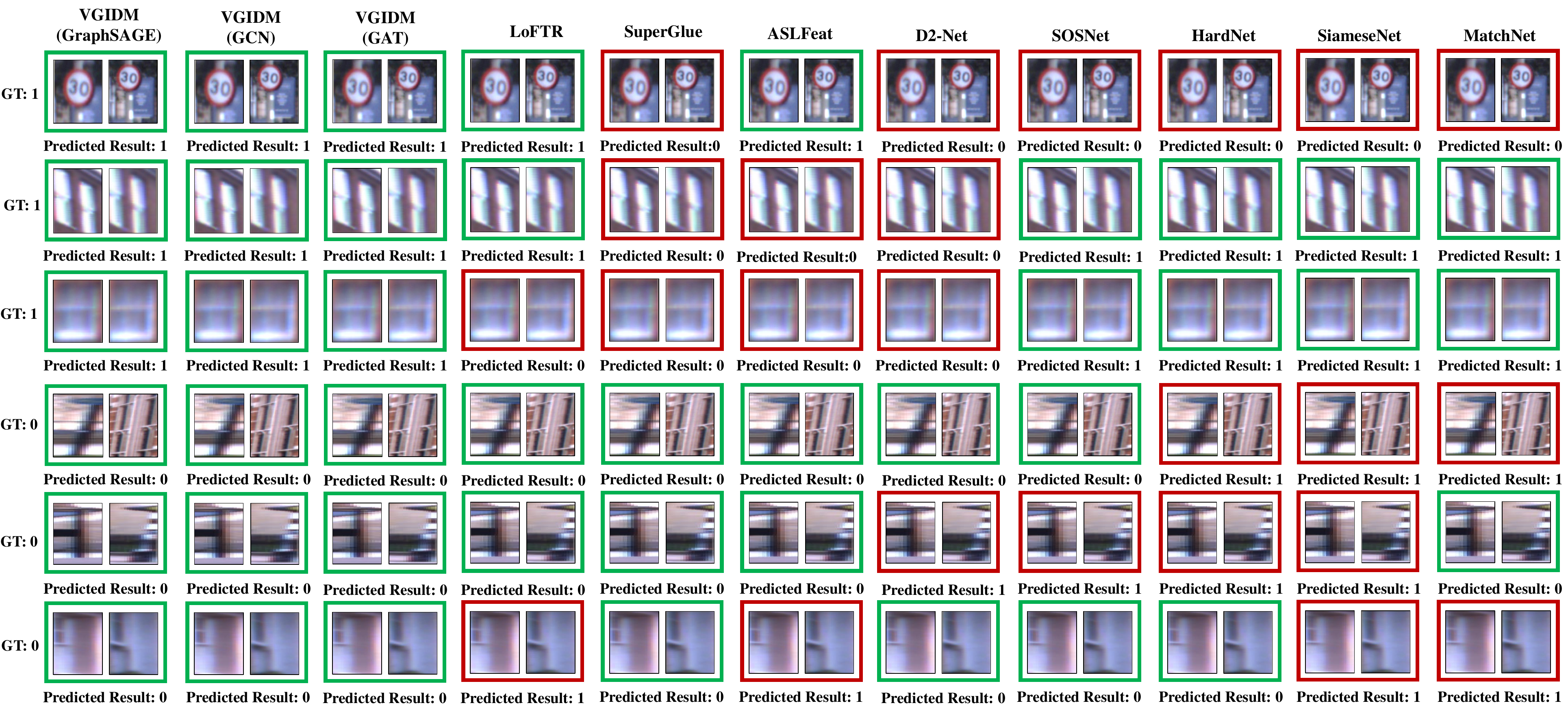}
\end{center}
\vspace{-0.3cm}
\caption{Examples of matched and mismatched pairs from the Landmark Oxford dataset. A green or red box indicates a correct or incorrect prediction result, respectively. ``GT'' stands for ground truth.}
\label{fig_oxford_results}
\vspace{-0.3cm}
\end{figure*}

\textbf{Performance Analysis.}
The VGIDM variants (with GAT, GCN or GraphSAGE) not only make use of landmark patch information but also the neighborhood information in the decision-making process for the matching task. 
Other feature descriptor learning as well as joint feature and metric learning methods such as MatchNet, SiameseNet, HardNet and SOSNet, depend only on the individual image patch rather than the neighborhood relationships. An erroneous match can happen between patches from two similar but distinct objects.
VGIDM mitigates this error by using the neighborhood information. However, VGIDM requires more computing resources for neighborhood graph processing. 

On the other hand, keypoint-based learning methods such as D2-Net, ASLFeat, SuperGlue and LoFTR, suffer from low pixel resolution of the image patches. As a landmark can be far away from the camera on the vehicle, its corresponding image patch can be small. As a result, it is more likely for these models to make mistakes in the matching decision. 

\textbf{Cross-Validation.}
To evaluate the generalization capability of VGIDM, we train VGIDM on the Landmark Oxford dataset and test it on the Landmark KITTI dataset. 
From \cref{cross-Kitti-Oxford-table} and \cref{Kitti-table}, we observe that the VGIDM variants still outperform the other baselines in all metrics.
From \cref{cross-Kitti-Oxford-table} and \cref{Oxford-table}, when we train on the Landmark KITTI dataset and test on the Landmark Oxford dataset, VGIDM outperforms the other baselines in precision and AUC.
Since the Landmark Oxford dataset contains windows as landmarks while not the Landmark KITTI dataset, the test performance on the Landmark Oxford dataset deteriorates more significantly. 
In \cref{Kitti-table,Oxford-table}, the baselines D2-Net, ASLFeat, SuperGlue and LoFTR do not perform training and test on the same dataset. Since there is no point-level ground-truth for our landmark patches, we adopt pre-trained models for these baselines from the literature \cite{dusmanu2019d2,luo2020aslfeat,sarlin2020superglue,sun2021loftr}. The other baselines are trained and tested on the same datasets. 
\begin{table*}[!htb]
\caption{
Cross-Validation on the Landmark KITTI dataset or Landmark Oxford dataset using the trained model based on the Landmark Oxford dataset or Landmark KITTI dataset, respectively. The ``Best in \cref{Kitti-table}'' or ``Best in \cref{Oxford-table}'' method refers to the best-performing baseline out of D2-Net, ASLFeat, SuperGlue and LoFTR from \cref{Kitti-table} or \cref{Oxford-table}.
}
\label{cross-Kitti-Oxford-table}
\vspace{-0.3cm}
\begin{center}
\newcommand{\tabincell}[2]{\begin{tabular}{@{}#1@{}}#2\end{tabular}}
\begin{tabular}{c|c|c|c|c|c}
\hline\hline  
{\bf Cross-Validation} & {\bf Methods}  & {\bf Precision}
& {\bf Recall} & {\bf $F_1$-Score}
& {\bf AUC} \\ 
\hline 
\multirow{3}{*}{\tabincell{c}{Oxford dataset (Train) \\ \& KITTI dataset (Test) \\}} 
& \tabincell{c}{VGIDM (GAT) \\}
                   & 0.9238 \scriptsize{$\pm$ 0.0020}  & 0.9047 \scriptsize{$\pm$ 0.0039} & 0.9141 \scriptsize{$\pm$ 0.0027} & 0.8403 \scriptsize{$\pm$ 0.0041}\\ 
& \tabincell{c}{VGIDM (GCN) \\}
                   & 0.9084 \scriptsize{$\pm$ 0.0019}  & \textbf{0.9579} \scriptsize{$\pm$ 0.0074} & 0.9325  \scriptsize{$\pm$ 0.0041} & 0.8341 \scriptsize{$\pm$ 0.0050} \\ 
& \tabincell{c}{VGIDM (GraphSAGE)  \\}
                   & \textbf{0.9255} \scriptsize{$\pm$ 0.0029} & 0.9483 \scriptsize{$\pm$ 0.0102} & \textbf{0.9368} \scriptsize{$\pm$ 0.0061} & \textbf{0.8597} \scriptsize{$\pm$ 0.0080} \\
\hline
\tabincell{c}{Baselines tested on KITTI dataset\\}
& \tabincell{c}{Best in \cref{Kitti-table} \\}
                   & 0.9189 \scriptsize{$\pm$ 0.0022} & 0.9243 \scriptsize{$\pm$ 0.0110}  & 0.9154 \scriptsize{$\pm$ 0.0059} & 0.8312 \scriptsize{$\pm$ 0.0057} \\
\hline\hline 
\multirow{3}{*}{\tabincell{c}{KITTI dataset (Train) \\ \& Oxford dataset (Test) \\}} 
& \tabincell{c}{VGIDM (GAT) \\}
                   & \textbf{0.9022} \scriptsize{$\pm$ 0.0031}  & 0.8930 \scriptsize{$\pm$ 0.0040} & 0.8976 \scriptsize{$\pm$ 0.0024} & \textbf{0.8013} \scriptsize{$\pm$  0.0051}\\ 
& \tabincell{c}{VGIDM (GCN) \\}
                   & 0.8909 \scriptsize{$\pm$ 0.0043}  & 0.8845 \scriptsize{$\pm$ 0.0124} & 0.8877  \scriptsize{$\pm$ 0.0074} & 0.7799 \scriptsize{$\pm$ 0.0096} \\ 
& \tabincell{c}{VGIDM (GraphSAGE)  \\}
                   & 0.8918 \scriptsize{$\pm$ 0.0055} & 0.9098 \scriptsize{$\pm$ 0.0064} & 0.9007 \scriptsize{$\pm$ 0.0043} & 0.7893 \scriptsize{$\pm$ 0.0099} \\
\hline
\tabincell{c}{Baselines tested on Oxford dataset\\}
& \tabincell{c}{Best in \cref{Oxford-table} \\}
                   & 0.8729 \scriptsize{$\pm$ 0.0036} & \textbf{0.9837} \scriptsize{$\pm$ 0.0060}  & \textbf{0.9129} \scriptsize{$\pm$ 0.0029} & 0.7548  \scriptsize{$\pm$ 0.0062} \\
\hline\hline 
\end{tabular}
\end{center}
\vspace{-0.3cm}
\end{table*}

\subsection{Ablation Study}\label{sect:ablation}

\textbf{Feature Pair Discrimination.}
We perform ablation studies on different feature pairs for the matching task shown in \cref{feature-compare-table}. The feature pair comparison include $d\parens{f(x), f(y)}$ for Resnet features, $d\parens{\rho(x), \rho(y)}$ for vertex embeddings, $d\parens{\varphi(x),\varphi(y)}$ for ensemble vertex embeddings, as well as $d\parens{\psi(\calG^x),\psi(\calG^y)}$ for neighborhood graph embeddings, where the learnable layer $d$ given by \cref{eq:ddd} as the metric. For each feature comparison, we train the corresponding models for $150$ epochs and select the optimal test result for the matching task based on the Landmark Oxford dataset. 
From \cref{feature-compare-table}, we observe that our proposed vertex-to-graph comparison outperforms the other feature pairs in most metrics. 
The feature pairs containing graph information, like $\psi(\calG^x)$ and $\psi(\calG^y)$, have an advantage over those based only on vertex information, like $f(x)$ and $f(y)$.
This demonstrates the benefit of utilizing graph information. 
\begin{table}[!htb]
\caption{
Ablation study for feature pair discrimination.
}
\label{feature-compare-table}
\vspace{-0.3cm}
\begin{center}
\newcommand{\tabincell}[2]{\begin{tabular}{@{}#1@{}}#2\end{tabular}}
\begin{tabular}{c|c|c|c|c}
\hline\hline  
\tabincell{c}{{\bf Feature} \\ {\bf Comparison}\\}  & {\bf Precision} & {\bf Recall} & {\bf $F_1$-Score} & {\bf AUC} \\ 
\hline 
\tabincell{c}{$d\parens{f(x),f(y)}$ \\}
                   & 0.8817          & 0.9146       & 0.8979        & 0.7733 \\ \hline 
\tabincell{c}{$d\parens{\rho(x), \rho(y)}$ \\}
                   & 0.7772          & \bf{0.9720}  & 0.8637        & 0.5680 \\ \hline 
\tabincell{c}{$d\parens{\varphi(x), \varphi(y)}$  \\}
                   & 0.8819          & 0.9560       & 0.9175        & 0.7860 \\ \hline 
\tabincell{c}{$d\parens{\psi(\calG^x),\psi(\calG^y)}$  \\}
                   & 0.9029          & 0.9427       & 0.9224        & 0.8193 \\
\hline
\tabincell{c}{$d\parens{\varphi(x),\psi(\calG^y)}$  \\}
                   & \bf{0.9097}     & 0.9533       & \bf{0.9310}   & \bf{0.8347} \\
\hline\hline 
\end{tabular}
\end{center}
\vspace{-0.3cm}
\end{table}

\textbf{Discriminator Function.} 
We evaluate the effectiveness of the learnable discriminator $d$ by comparing it with other discriminator functions. Specifically, we replace the learnable discriminator $d$ with either cosine similarity or $\mathcal{L}_2$ distance in \cref{eq.d_M}. We evaluate the patch-matching task on the Landmark Oxford dataset. From \cref{metrics-table}, we observe that the proposed learnable discriminator $d$ outperforms the other discriminators, which is likely due to the adaptability of neural networks to different feature dimensions. 
\begin{table}[!htb]
\caption{
Ablation study for different discriminator functions.
}
\label{metrics-table}
\vspace{-0.3cm}
\begin{center}
\newcommand{\tabincell}[2]{\begin{tabular}{@{}#1@{}}#2\end{tabular}}
\begin{tabular}{c|c|c|c|c}
\hline\hline  
\tabincell{c}{{\bf Feature} \\ {\bf Discriminator}\\} 
& {\bf Precision} & {\bf Recall} & {\bf $F_1$-Score} & {\bf AUC} \\ 
\hline 
{\tabincell{c}{ Cosine similarity \\}} 
& 0.9007            & 0.8707          & 0.8854           & 0.7913 \\
\hline 
{\tabincell{c}{$\mathcal{L}_2$ distance \\}} 
& 0.7277            & 0.8800          & 0.7966           & 0.5540 \\
\hline 
{\tabincell{c}{Learnable $d$ \\}} 
& {\bf 0.9097}      & {\bf 0.9533}    & {\bf 0.9310}     & {\bf 0.8347} \\
\hline\hline 
\end{tabular}
\end{center}
\vspace{-0.5cm}
\end{table}

\begin{table}[!htb]
\caption{
Matching performance of different methods based on spatial neighborhood information.
}
\label{neighborhood-comparison-table}
\vspace{-0.3cm}
\begin{center}
\resizebox{0.48\textwidth}{!}{\setlength{\tabcolsep}{1.5pt} 
\newcommand{\tabincell}[2]{\begin{tabular}{@{}#1@{}}#2\end{tabular}}
\begin{tabular}{c|c|c|c|c}
\hline\hline  
{\bf Methods}  & {\bf Precision} & {\bf Recall} & {\bf $F_1$-Score} & {\bf AUC} \\ 
\hline 
\tabincell{c}{SIFT \cite{lowe2004object}+RANSAC \cite{fischler1981random} (+neighbors)\\} 
                   & 0.8965              & 0.9467             & 0.9209             & 0.8093 \\ \hline 
\tabincell{c}{MatchNet \cite{han2015matchnet} (+neighbors)\\} 
                   & 0.9023              & 0.9600             & 0.9302             & 0.8240 \\ \hline 
\tabincell{c}{SuperGlue \cite{sarlin2020superglue} (+neighbors)\\} 
                   & 0.9225              & 0.9200             & 0.9212             & 0.8440 \\ \hline 
\tabincell{c}{LoFTR \cite{sun2021loftr} (+neighbors)\\} 
                   & 0.9192              & 0.9413             & 0.9302             & 0.8467 \\ \hline
\tabincell{c}{VGIDM [ours] \\} 
                   & \textbf{0.9280}     & \textbf{0.9627}    & \textbf{0.9450}    & \textbf{0.8693}  \\
\hline\hline 
\end{tabular}}
\end{center}
\vspace{-0.3cm}
\end{table}

\textbf{Spatial Neighborhood Information.}
We investigate whether the performance improvement of VGIDM is mainly due to the spatial neighborhood information. To do this, we introduce the neighborhood graphs used in VGIDM to other baseline methods. Specifically, for a given vertex, we sort its neighbors according to increasing distances from it. We then use each baseline method to compare not only the vertex pair but also the pairs of their corresponding neighbors with the same sort order. Then, we calculate the average of the predicted scores for the vertex pair and its neighbor pairs. Finally, we decide whether there is a match based on a threshold, which is a hyperparameter tuned separately to achieve the best performance for each baseline. 
\Cref{neighborhood-comparison-table} shows results on the Landmark KITTI dataset, where the best test performance for each baseline model is selected.
We compare with VGIDM (GraphSAGE), which is trained on the Landmark Oxford dataset.  
We observe that including neighborhood information generally improves the performance of every baseline, but VGIDM still outperforms them. This indicates that the neighborhood graph feature representation in VGDIM has advantages in the patch-matching task. 
As shown in \cref{time-table-neighborhood-comparison}, which presents the inference runtime for one pair of frames (with around twenty patch pairs for comparison), the computational complexity of VGIDM is lower than most baselines, except MatchNet.

\begin{table}[!htb]
\caption{
Inference runtime comparison for different methods on the Landmark KITTI dataset.
}
\label{time-table-neighborhood-comparison}
\vspace{-0.3cm}
\begin{center}
\resizebox{0.48\textwidth}{!}{\setlength{\tabcolsep}{1pt} 
\newcommand{\tabincell}[2]{\begin{tabular}{@{}#1@{}}#2\end{tabular}}
\begin{tabular}{c|c|c|c|c|c}
\hline\hline 
{\bf Methods}  & \tabincell{c}{SIFT+\\ RANSAC \\ (+neighbors)\\} 
& \tabincell{c}{MatchNet \\ (+neighbors)\\}  &  \tabincell{c}{SuperGlue \\ (+neighbors)\\}  
& \tabincell{c}{LoFTR \\ (+neighbors)\\} & \tabincell{c}{VGIDM \\ (GraphSAGE) \\}  \\
\hline
\tabincell{c}{{\bf Inference} \\ {\bf runtime} \\} &  0.7035s & 0.1591s & 3.8072s & 4.3740s & 0.2013s \\
\hline\hline 
\end{tabular}}
\end{center}
\vspace{-0.8cm}
\end{table}

\subsection{Computational Complexity}\label{sect:compuational_comlexity}
To evaluate the runtime performance, we test VGIDM on an NVIDIA RTX A5000 GPU. 
\cref{time-table} shows the inference runtime (mean time for one pair of frames in the testing phase) for the VGIDM variants with different GNNs. 
Specifically, given a pair of frames (i.e., full-size images), an average of around twenty patch pairs are compared, which takes less than $0.25$ seconds. 
The time taken is acceptable for practical applications, such as place recognition and autonomous driving. 
Moreover, the amount of the parameters for these VGIDM networks with the GAT, GCN and GraphSAGE is $12.16 \mathrm{M}$, $12.16 \mathrm{M}$ and $12.66 \mathrm{M}$, respectively. 
\begin{table}[!htb]
\caption{Inference runtime of VGIDM on Landmark Oxford dataset.}
\label{time-table}
\vspace{-0.3cm}
\begin{center}
\newcommand{\tabincell}[2]{\begin{tabular}{@{}#1@{}}#2\end{tabular}}
\begin{tabular}{c|c|c|c}
\hline\hline 
{\bf Methods}                                   &   \tabincell{c}{VGIDM \\ (GAT) \\}  &\tabincell{c}{VGIDM \\ (GCN) \\} &  \tabincell{c}{VGIDM \\ (GraphSAGE) \\}\\
\hline
\tabincell{c}{{\bf Inference} {\bf Runtime} \\} & 0.2330s & 0.1953s & 0.2092s \\
\hline\hline 
\end{tabular}
\end{center}
\vspace{-0.5cm}
\end{table}

\subsection{Further Possible Applications}\label{app.appli}

\subsubsection{Application of VGIDM in Visual Place Recognition}\label{section.place_recognition}

\begin{figure}[!htb]
\begin{center}
\includegraphics[width=\linewidth, height=0.16\textheight]{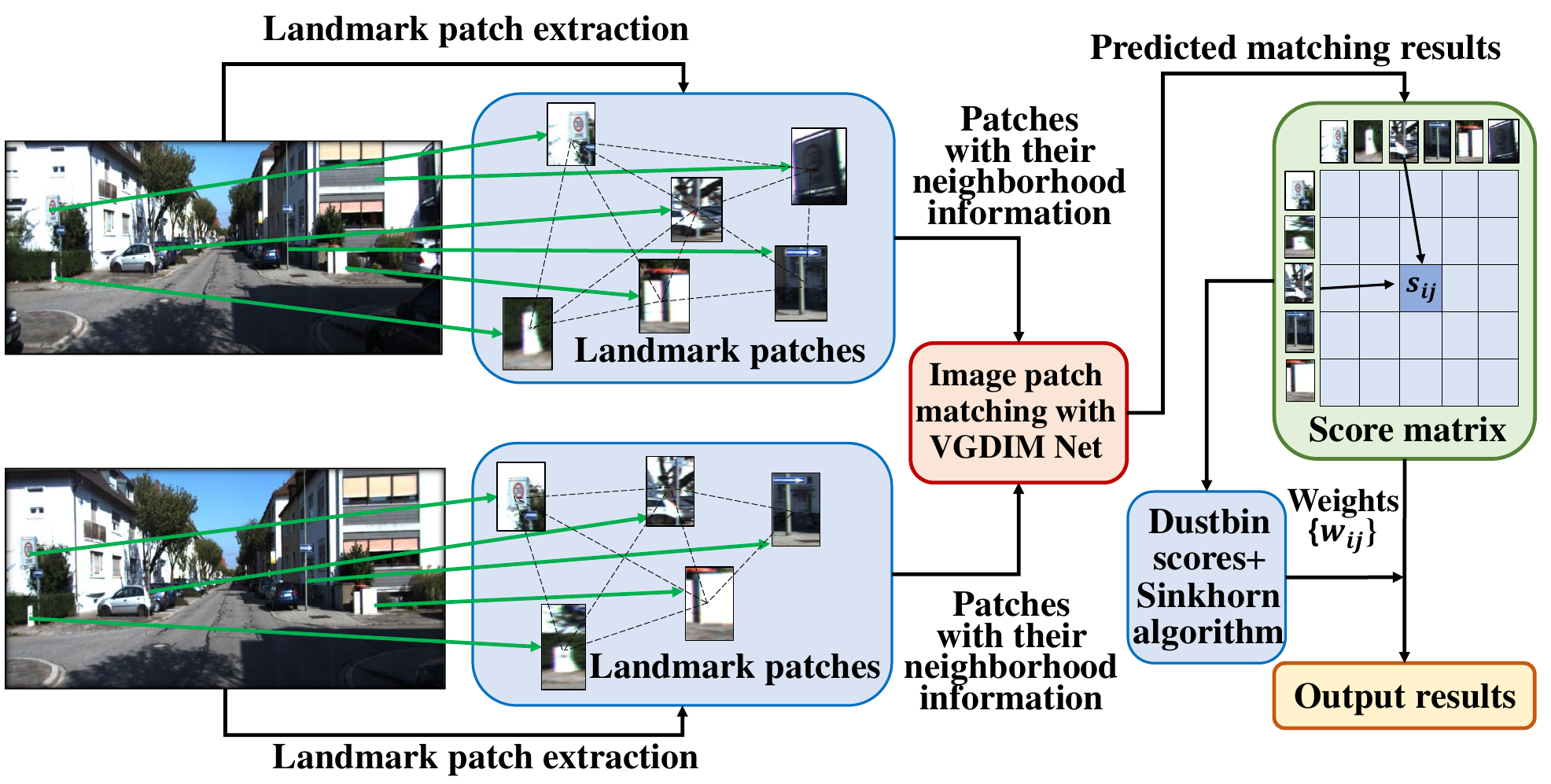}
\end{center}
\vspace{-0.3cm}
\caption{The diagram of visual place recognition with VGIDM.}
\label{fig_place_recognition}
\vspace{-0.3cm}
\end{figure}

A possible application of VGIDM is visual place recognition. We apply our local patch-matching to obtain global frame matching to determine if two frames show the same place. In visual place recognition, we construct a bipartite graph with edges being the scores output by our network for all landmark patch pairs, which is used to construct a matching score matrix. Then, similar to the Optimal Matching Layer in \cite{sarlin2020superglue}, by appending learnable dustbin scores for the score matrix, the Sinkhorn algorithm is used to output the partial assignment. Finally, we obtain frame-matching results by summing up the elements in the matching score matrix with the weights from the partial assignment.
The details are shown in \cref{fig_place_recognition}, where GAT is chosen for the GNN part in VGIDM.

\begin{table}[!htb]
\caption{
Visual place recognition performance in KITTI data.
}
\label{recognition-table}
\vspace{-0.6cm}
\begin{center}
\begin{tabular}{c|c|c|c|c|c}
\hline\hline 
{\bf Methods}     & MatchNet  & NetVLAD & SuperGlue & LoFTR   & VGIDM \\ 
\hline 
{\bf $F_1$-Score} & 0.9668    & 0.9702  & 0.9694    & 0.9711  & \textbf{0.9719} \\ 
\hline 
{\bf Accuracy}    & 0.9360    & 0.9424  & 0.9408    & 0.9440  & \textbf{0.9452} \\ 
\hline\hline 
\end{tabular}
\end{center}
\vspace{-0.6cm}
\end{table}
\begin{figure}[!htb]
\begin{center}
\includegraphics[width=\linewidth]{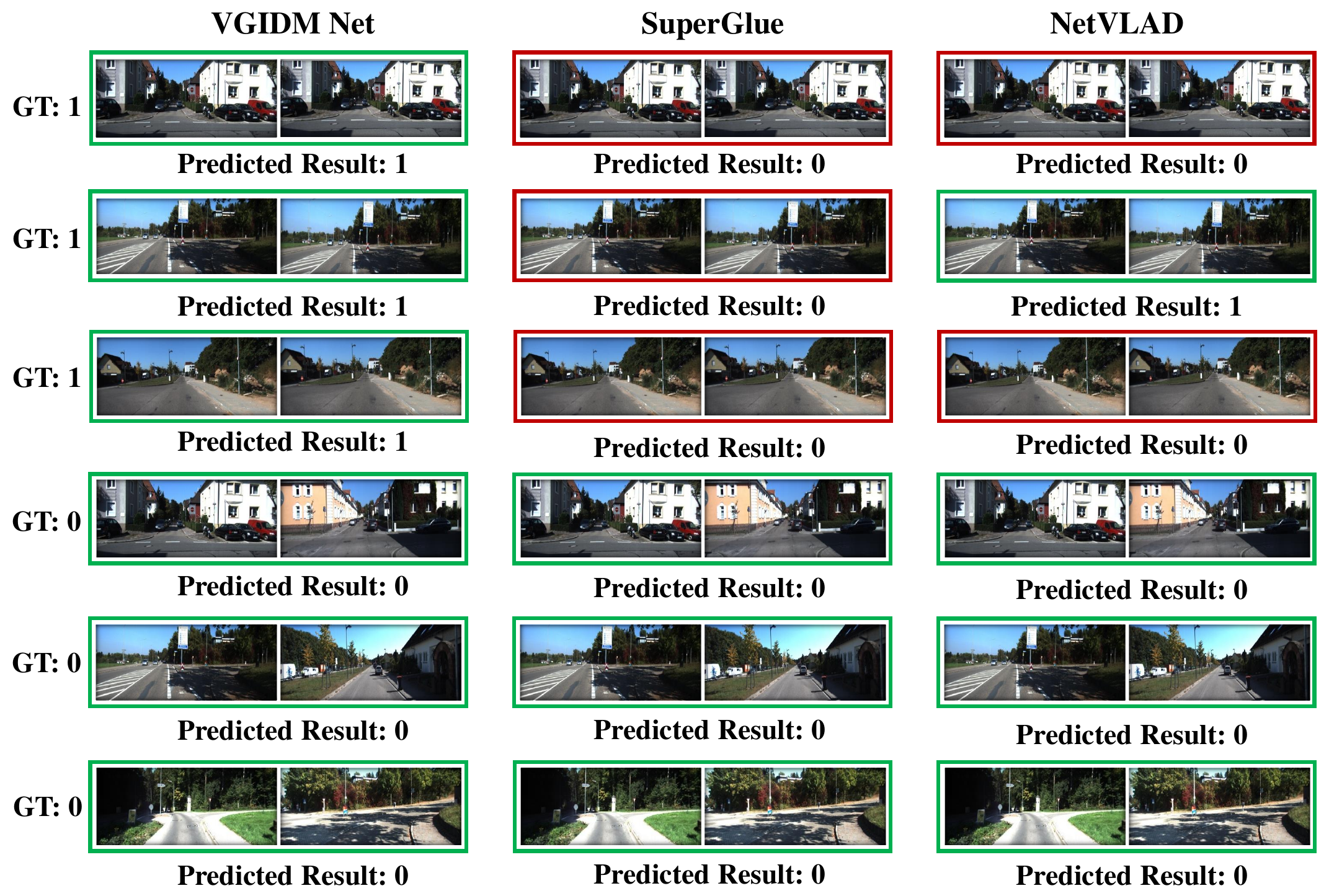}
\end{center}
\vspace{-0.3cm}
\caption{Several examples of place recognition on the KITTI dataset. The prediction ``1'' indicates the frames are from the same place, while ``0'' indicates they are from different places. A green box indicates a correct prediction result while a red box an incorrect one.}
\label{fig_place_recognition_ex}
\vspace{-0.3cm}
\end{figure}

\begin{table}[!htb]
\caption{
Cross-validation performance for visual place recognition on Oxford dataset.
}
\label{cross-recognition-table}
\vspace{-0.6cm}
\begin{center}
\begin{tabular}{c|c|c|c|c|c}
\hline\hline 
{\bf Methods}      & MatchNet & NetVLAD & SuperGlue & LoFTR   & VGIDM \\ 
\hline 
{\bf $F_1$-Score}  & 0.9010   & 0.9069  & 0.9190    & 0.9207  & \textbf{0.9266} \\ 
\hline 
{\bf Accuracy}     & 0.8273   & 0.8303  & 0.8563    & 0.8607  & \textbf{0.8680} \\ 
\hline\hline 
\end{tabular}
\end{center}
\vspace{-0.3cm}
\end{table}
\begin{figure}[!htb]
\begin{center}
\includegraphics[width=\linewidth]{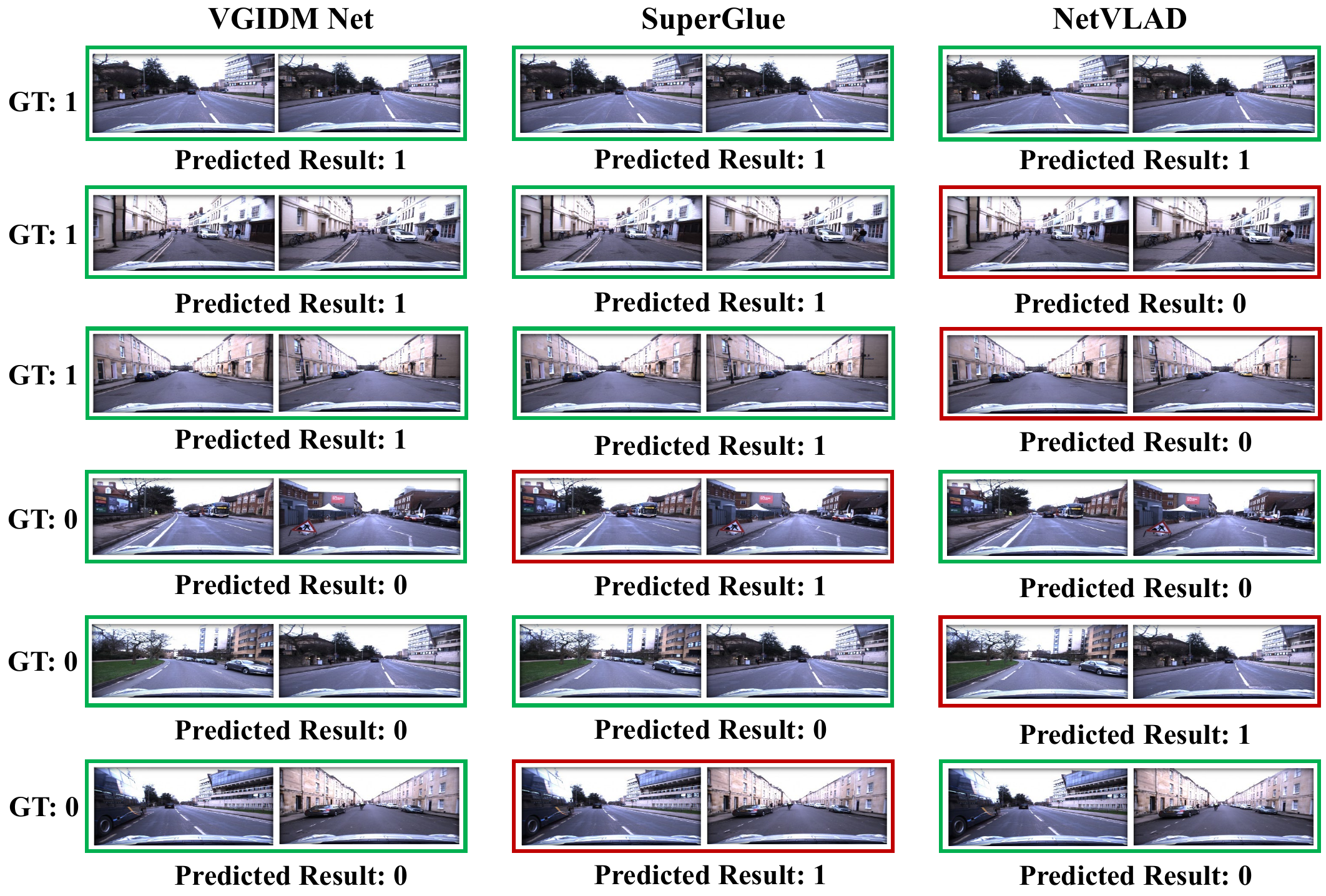}
\end{center}
\vspace{-0.2cm}
\caption{
Several examples of place recognition on the Oxford dataset where the model is trained on the KITTI dataset. The prediction ``1'' or ``0'' indicates the frames are from the same or different places. A green or red box indicates a correct or incorrect prediction result.
}
\label{fig_place_recognition_Oxford}
\vspace{-0.4cm}
\end{figure}

We compare VGIDM with MatchNet \cite{han2015matchnet}, NetVLAD \cite{arandjelovic2016netvlad}, SuperGlue \cite{sarlin2020superglue}, and LoFTR \cite{sun2021loftr} under the place recognition task with around $600$ pairs of place images from the KITTI dataset. 
The two places contained in each image frame pair are regarded as the same if the distance between the camera locations is less than $10$ meters.  
To recognize the same places more accurately, the thresholds of the matching results for place image pairs are set to higher recall levels.
The results and example outputs are shown in \cref{recognition-table} and \cref{fig_place_recognition_ex}, respectively. 
From \cref{recognition-table}, it is observed that all the methods perform well, with VGIDM having a slight advantage. The reason may be that there exist obvious differences among the image pairs that are not from the same place. However, unlike MatchNet, NetVLAD, SuperGlue, and LoFTR, which require the full image, VGIDM can perform place recognition by using only landmark patches and their spatial relationships. Moreover, the landmark patches based on static objects are more stable than the keypoints based on edges or corners and are not affected by noisy image pixels from non-persistent objects or surroundings.

We conduct cross-validation experiments for the visual place recognition task.
Specifically, we use VGIDM (GAT) trained on the KITTI dataset, and compare it with baselines for inference on the Oxford dataset. We use $3000$ frames in the experiments. 
From \cref{cross-recognition-table}, we observe that VGIDM is superior to the baselines. This suggests that VGIDM has good generalization ability. 
We include a few examples in \cref{fig_place_recognition_Oxford}.

\subsubsection{Application of VGIDM in Stereo Depth Estimation of Landmarks}\label{section.stereo_depth_estimation}

Another application is \emph{stereo} depth estimation of landmarks. Similar to the steps described in \cref{sect:Datasets}, we obtain the landmarks from the full-sized frames captured from both the \emph{left} and \emph{right} stereo cameras. 

Different from the experiment settings in \cref{sect:Datasets} where the matching is performed for landmark patches in image frames captured at different locations, here we use VGIDM to perform the matching between landmark patches captured at the \emph{same} location but from \emph{different} cameras. We split $3000$ pairs of stereo frames into training and testing sets with a ratio of around $2:1$. During testing, we set a high similarity threshold of $0.9$ to prevent false positive matching. Since the landmark objects we have chosen have regular shapes, the original narrow landmark object detection boxes (without the intentionally added background to form the landmark patches) are sufficiently accurate to locate the landmarks in the frames. 

For each of the matched landmark objects in the left and right frames, we compute the pixel disparity on the center line (average of left and right sides) of the original bounding box. The depth of the pixels on the landmark bounding box middle lines can be calculated using the camera focal length and distance. Following the diagram in \cref{fig:mono-stereodepth}, the coarse monocular depth is improved to a more accurate stereo depth as shown in \cref{tab.depth}. The vanilla Monocular Depth \cite{yin2021learning} only achieves $14.22 \mathrm{m}$ RMSE. After applying VGIDM, we can improve the depth estimation accuracy to $0.86 \mathrm{m}$ RMSE. In contrast, the current state-of-the-art DIFFNET \cite{zhou2021diffnet} only achieves $4.45 \mathrm{m}$ RMSE performance in stereo depth estimation. However, a direct application of VGIDM can only output the estimated depth for sparse pixels (only for the chosen landmarks). To improve general stereo depth estimation, it is possible to incorporate VGIDM into existing methods, e.g., using VGIDM's accurately estimated stereo landmark depths as calibration for other general stereo depth estimation algorithms like DIFFNet \cite{zhou2021diffnet}, HRDepth \cite{lyu2020hrdepth}, and CADepth \cite{yan2021cadepth}.

\begin{figure}[!htb]
\centering
\includegraphics[width=\linewidth]{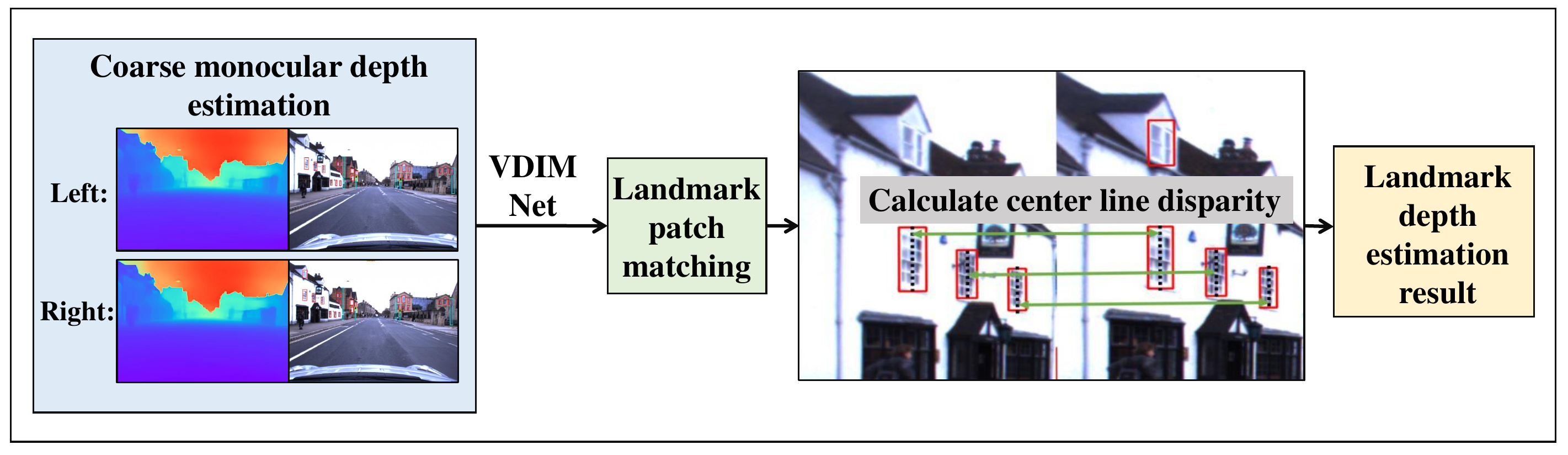}
\vspace{-0.6cm}
\caption{Depth estimation: from coarse monocular depth to fine stereo depth.}
\label{fig:mono-stereodepth}
\vspace{-0.1cm}
\end{figure}

\begin{table}
\caption{Accuracy of depth estimation.}
\label{tab.depth}
\begin{center}
\vspace{-0.3cm}
\newcommand{\tabincell}[2]{\begin{tabular}{@{}#1@{}}#2\end{tabular}}
\begin{tabular}{c|c|c|c}
\hline\hline 
\textbf{Method} & \tabincell{c}{Monocular \\ Depth \cite{yin2021learning} \\} & DIFFNet \cite{zhou2021diffnet} & VGIDM [ours] \\
\hline 
\textbf{RMSE} ($\mathrm{m}$)  & 14.22  & 4.45   & {\bf 0.86} \\
\hline\hline 
\end{tabular}
\end{center}
\vspace{-0.6cm}
\end{table}

VGIDM can serve as a module in various learning-based localization techniques, such as Detect-SLAM \cite{zhong2018detect}, EAO-SLAM \cite{wu2020eao}, and other semantic SLAM with object-level data association \cite{qian2021semantic,lin2023robust}. These techniques underscore the importance of landmark patch matching, which is central to VGIDM in real-world applications. Consequently, our approach has crucial implications for diverse applications, demonstrating its versatility and effectiveness. 

\section{Conclusion}\label{sect:conc}

We have developed an image patch-matching model VGIDM that incorporates spatial information of the landmark patches through graph-based learning. We provided a theoretical basis for our approach. Extensive experiments demonstrate that our approach outperforms the current state-of-the-art baselines, which do not take into account the spatial relationships between patches. Our framework indicates that such spatial information can be beneficial to landmark patch-matching in street scenes.

In future work, it is of interest to incorporate a greater variety of objects as landmarks to adapt VGIDM to more diverse street scenes and generalize to more datasets. To achieve this, we can train VGIDM on landmarks from a wider range of classes.
As our method is better suited for matching tasks in scenarios with sufficient landmarks or static objects, pixel-level matching methods can serve a complementary role in scenarios with fewer landmarks. Additionally, combining our patch-level matching method with point-level methods shows promise in achieving more accurate pixel-level or sub-pixel-level matching. To this end, we can use our method as a post-processing step to emphasize the keypoints with more attention or to filter out weak correspondences. 
We can further generalize VGIDM to multi-view camera relocalization \cite{xue2020learning} to estimate camera poses by determining the matched landmarks in multiple image frames. Furthermore, the matched landmarks can serve as anchor points or interest regions to aid other applications. In LiDAR relocalization \cite{zhang2015visual,engel2019deeplocalization} or LiDAR odometry estimation, by restricting the LiDAR points from matched landmarks using VGIDM, many outlier points can be removed, leading to better estimation accuracy. 



\appendices

\section{Proof of \cref{prop.KL_divergence_Max}}\label{app.proposition_1}

From \cref{eq.L_ID_expect}, we have 
\begin{align}
    & L_{\mathrm{ID}} \nn
    & = \P(x\leftrightarrow y) \E[\log d(\varphi(x),\psi(\mathcal{G}^y)) \given x \leftrightarrow y] \nn
    & \quad + \P(x\nleftrightarrow y)\E[\log(1-d(\varphi(x),\psi(\mathcal{G}^y))) \given x \nleftrightarrow y] \label{eq.L_ID_original0}\\
    & = \int_{\mathbb{A}} \Big\{ \P(x\leftrightarrow y)
    p(\varphi(x),\psi(\mathcal{G}^{y})\mid{x \leftrightarrow y}) \log d(\varphi(x),\psi(\mathcal{G}^y))\nn
    & \quad + \P(x\nleftrightarrow y) p(\varphi(x),\psi(\mathcal{G}^{y})\mid{x \nleftrightarrow y}) \nn
    & \qquad \cdot \log(1-d(\varphi(x), \psi(\mathcal{G}^{y}))) \Big\}{\rm d}(\varphi(x),\psi(\mathcal{G}^y)),\label{eq.L_ID_original}
\end{align}
where $\P(x\leftrightarrow y)$ and $\P(x\nleftrightarrow y)$ are the matched and unmatched probabilities. From \cref{eq.L_ID_original}, it is clear that $L_{\mathrm{ID}}$ is concave in $d(\varphi(x), \psi(\mathcal{G}^{y}))$ for every $(x,y)$. Taking the first derivatives and setting them to zero, we obtain
\begin{align}
    & d^{*}(\varphi(x), \psi(\mathcal{G}^{y})) 
    = \frac{\P(x\leftrightarrow y) p(\varphi(x),\psi(\mathcal{G}^{y})\mid{x \leftrightarrow y})}
    {p(\varphi(x),\psi(\mathcal{G}^{y}))}. \label{eq.D*} 
\end{align}

Substituting \cref{eq.D*} into \cref{eq.L_ID_original}, we have 
\begin{align}
    & L_{\mathrm{ID}}^{d^*} \nn
    & = \P(x\leftrightarrow y)
    \E[\log \frac{p(\varphi(x),\psi(\mathcal{G}^y)\mid{x \leftrightarrow y})}{p(\varphi(x),\psi(\mathcal{G}^y)\mid{x \nleftrightarrow y})} \given x \leftrightarrow y]\nn
    & \quad + \E[\log \frac{p(\varphi(x),\psi(\mathcal{G}^{y}) \mid x \nleftrightarrow y)}{p(\varphi(x),\psi(\mathcal{G}^{y}))}] - H_b(\P(x\leftrightarrow y)). \label{Ld*}
\end{align}
Applying Jensen's inequality to the second term in the right-hand side of \cref{Ld*}, we have
\begin{align}
&\E[\log \frac{p(\varphi(x),\psi(\mathcal{G}^{y}) \mid x \nleftrightarrow y)}{p(\varphi(x),\psi(\mathcal{G}^{y}))}] \nn
&\leq \log \E[\frac{p(\varphi(x),\psi(\mathcal{G}^{y}) \mid x \nleftrightarrow y)}{p(\varphi(x),\psi(\mathcal{G}^{y}))}] \nn
&= \log \int_{\bbA} p(\varphi(x),\psi(\mathcal{G}^{y}) \mid x \nleftrightarrow y) {\rm d}(\varphi(x),\psi(\calG^y)) \nn
&= 0.
\end{align}
Therefore, from \cref{Ld*}, we have
\begin{align}
    & L_{\mathrm{ID}}^{d^*} 
    \leq \P(x\leftrightarrow y) \nn
    & \qquad \cdot \KLD{p(\varphi(x),\psi(\mathcal{G}^{y})\mid{x \leftrightarrow y})}{p(\varphi(x),\psi(\mathcal{G}^{y})\mid{x \nleftrightarrow y})} \nn
    & \qquad - H_b(\P(x\leftrightarrow y)).
\end{align}
Rearranging the inequality completes the proof.

\section{Proof of \cref{prop.disturbance_discriminator}}\label{app.proposition_disturbance}

{Let $\tilde{d}=d+\epsilon$.}
Similar to \cref{eq.L_ID_original0} in the proof of {\cref{prop.KL_divergence_Max}}, it is easy to see
\begin{align}
    & L_{\mathrm{ID}}(\varphi,\psi,\tilde{d}) \nn
    & = \P(x\leftrightarrow y)
    \E[ \log(d(\varphi(x),\psi(\mathcal{G}^y))+\varepsilon) \given x \leftrightarrow y ] \nn
    & \quad + \P(x\nleftrightarrow y) \E[ \log(1-d(\varphi(x),\psi(\mathcal{G}^y))-\varepsilon) \given x \nleftrightarrow y] \label{eq.L_ID_tilde_D},
\end{align}
where the notations are the same as those in \cref{eq.L_ID_original0}.

According to Taylor's series expansion theorem \cite{crowell1968calculus}, 

we have
\begin{align}
    & L_{\mathrm{ID}}(\varphi,\psi,\tilde{d}) - L_{\mathrm{ID}}(\varphi,\psi,d) \nn
    & = \varepsilon \Bigg\{
    \P(x\leftrightarrow y) \E[ \frac{1}{d(\varphi(x),\psi(\mathcal{G}^y))} \given x\leftrightarrow y] \nn
    & \qquad - \P(x\nleftrightarrow y) \E[\frac{1}{1-d(\varphi(x),\psi(\mathcal{G}^y))} \given x\nleftrightarrow y] \Bigg\} \nn
    & \quad - \frac{\varepsilon^2}{2} \Bigg\{
    \P(x\leftrightarrow y) \E[\frac{1}{(d(\varphi(x),\psi(\mathcal{G}^y)))^2} \given x\leftrightarrow y]\nn
    & \qquad \qquad 
    + \P(x\nleftrightarrow y) \E[\frac{1}{(1-d(\varphi(x),\psi(\mathcal{G}^y)))^2} \given x\nleftrightarrow y]\Bigg\} \nn
    & \quad + o(\varepsilon^2).  \label{eq.L_tilde_D_L_D}
\end{align}

Furthermore, by substituting $d(\varphi(x),\psi(\mathcal{G}^y))=d^{*}(\varphi(x),\psi(\mathcal{G}^y))$ given in \cref{eq.D*} into \cref{eq.L_tilde_D_L_D}, we have 
\cref{eq.difference_L}
\begin{figure*}
\hrulefill
\begin{align}
    & L_{\mathrm{ID}}(\varphi,\psi,\tilde{d}^*) - L_{\mathrm{ID}}(\varphi,\psi,d^*) \nn
    & = - \frac{\varepsilon^2}{2} \int_{\mathbb{A}} 
    \frac{
    (\P(x\leftrightarrow y) p(\varphi(x),\psi(\mathcal{G}^{y})\mid{x \leftrightarrow y})
    + \P(x\nleftrightarrow y) p(\varphi(x),\psi(\mathcal{G}^{y})\mid{x \nleftrightarrow y})  )^3 }
    {\P(x\leftrightarrow y) \P(x\nleftrightarrow y) p(\varphi(x),\psi(\mathcal{G}^{y})\mid{x \leftrightarrow y}) p(\varphi(x),\psi(\mathcal{G}^{y})\mid{x \nleftrightarrow y}) }
    {\rm d}(\varphi(x),\psi(\mathcal{G}^y))
    + o(\varepsilon^2) \label{eq.difference_L}
\end{align}
\hrulefill
\end{figure*}
and the proof is complete.

\section{Proof of \cref{prop.proposition_3}}\label{app.proposition_3}

We have
\begin{align}
&\norm{\pfgm - \pfgu}_{\mathrm{TV}} \nn
& = \sum_{(\varphi(x), \psi(\calG^{y})) \in \bbB} \Big\{ m(x\leftrightarrow y) \pfgGm{\Gm} \nn
& \qquad + (1-m(x\leftrightarrow y))\pfgGm{\Gu} \nn
&\qquad - m(x\nleftrightarrow y)\pfgGu{\Gm} \nn
& \qquad - (1-m(x\nleftrightarrow y))\pfgGu{\Gu} \Big\} \nn
&\geq \sum_{(\varphi(x), \psi(\calG^{y})) \in \bbB} \Big\{ \min_{x\leftrightarrow y} m(x\leftrightarrow y) \nn
& \qquad \qquad \qquad \qquad \cdot \pfgGm{\Gm} \nn
& \qquad - \max_{x \nleftrightarrow y} m(x\nleftrightarrow y)\pfgGu{\Gm} \nn 
&\qquad - \pfgGu{\Gu} \Big\} \label{TV1}\\
&\geq \min_{x\leftrightarrow y}m(x\leftrightarrow y) \nn
& \qquad \cdot \sum_{(\varphi(x), \psi(\calG^{y})) \in \bbB} \Big\{\pfgGm{\Gm} \nn
&\qquad - \pfgGu{\Gu} \Big\} + \min_{x\leftrightarrow y}m(x\leftrightarrow y)\nn
&\qquad - \max_{x \nleftrightarrow y} m(x\nleftrightarrow y) - 1,
\end{align}
where the inequality \cref{TV1} follows from $0 \le m(x \leftrightarrow y) \le 1$ and $0 \le m(x \nleftrightarrow y) \le 1$. The proof is now complete. 

\bibliographystyle{IEEEtran}
\bibliography{IEEEabrv,StringDefinitions,Journal_reference}

 
%


\newpage
\text{}
\newpage

\section*{Supplementary Material}

\subsection{Landmark Datasets for Image Patch Matching}\label{app_sect:Landmarkpre}

In this section, we present two landmark patch matching datasets,\footnote{https://github.com/AI-IT-AVs/Landmark\_patch\_datasets} named the \textit{Landmark KITTI Dataset} and the \textit{Landmark Oxford Dataset}, derived from the street scene KITTI dataset and the Oxford Radar RobotCar Dataset respectively.

We first briefly introduce the two original public datasets both of which contain image frames and LiDAR scans captured from onboard cameras and Velodyne LiDAR sensors. The KITTI dataset is a public dataset\footnote{http://www.cvlibs.net/datasets/kitti/} with multi-sensor data for autonomous driving. It contains street scene image frames and their corresponding LiDAR point clouds, which are captured in Karlsruhe, Germany, using the Point Grey Flea 2 (FL2-14S3C-C) Camera and Velodyne HDL-64E Laserscanner, respectively. The frame resolution is $1241 \times 376$  pixels.
The Oxford Radar RobotCar dataset\footnote{http://ori.ox.ac.uk/datasets/radar-robotcar-dataset} contains image frames and LiDAR scans captured on the streets in Oxford, UK, by the Point Grey Grasshopper2 (GS2-FW-14S5C-C) Camera and Velodyne HDL-32E Laserscanner, respectively. The resolution of each frame in this dataset is $1280 \times 960$ pixels. 

We extract the landmark object patches from the full-sized image frames of the two original street scene datasets using an object detection neural network. 
In the literature on landmark-based applications, Edge Boxes are used to detect a bounding box around a patch that contains a large number of internal contours compared to the number of contours exiting from the box, which indicates the presence of an object in the enclosed patch. DeepLabV$3$+ is used to extract significant landmark regions. However, all of the aforementioned patch extraction or landmark detection approaches are not stable when removing dynamic objects and many noisy regions are presented. By contrast, in our datasets, we use Faster R-CNN as the stable landmark object detector to locate the region of interest for static roadside objects including traffic lights, traffic signs, poles, and facade windows. To facilitate the detection efficacy, we manually labeled those objects using the frames from the street scene KITTI dataset and the Oxford Radar RobotCar dataset.
In Faster R-CNN, we choose Resnet$50$  with Feature Pyramid Network (FPN) as the backbone, which is already pretrained on the Imagenet dataset. During training, we use Adam optimizer with learning rate $0.0002$ and weight decay $0.0001$ to train the detector for $50$ epochs. The training batch size is set as $2$ and random horizontal flipping is used for data augmentation. 

We next introduce our landmark patch matching datasets. For both the Landmark KITTI dataset and the Landmark Oxford dataset, the full-sized image frames are captured by stereo cameras, and we \emph{only use the left frames} to extract landmark patches. The details like the landmark object bounding box labels and the patch matching ground truth are described separately for each dataset as follows. 

\begin{figure}[!htb]
\centering
\includegraphics[width=0.5\linewidth, height=0.06\textheight]{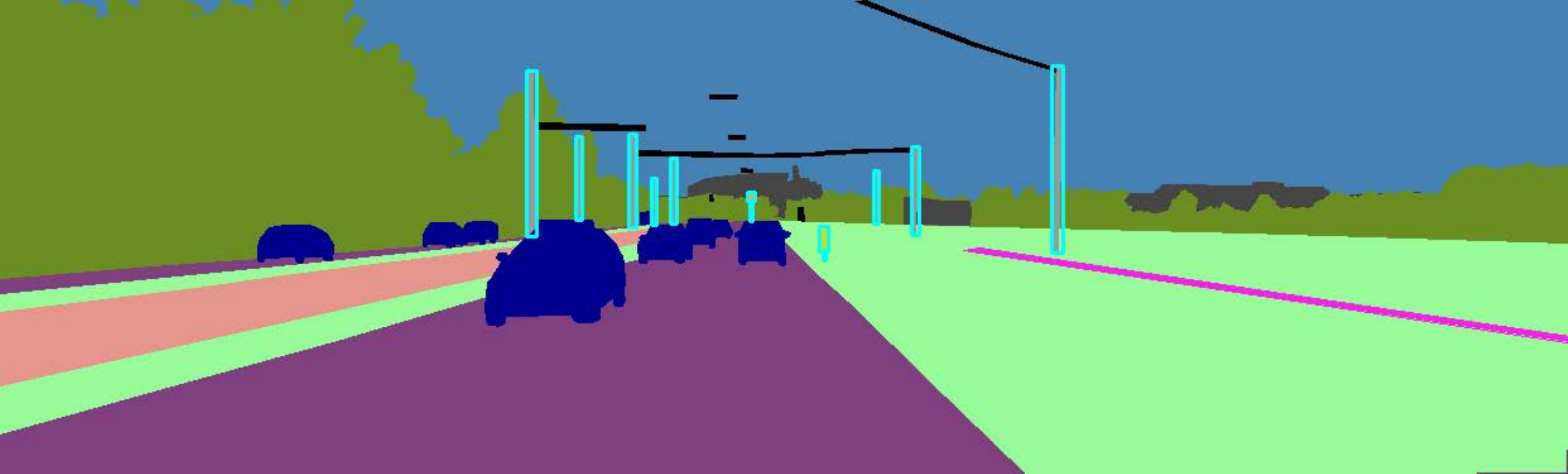}\hfill
\includegraphics[width=0.5\linewidth, height=0.06\textheight]{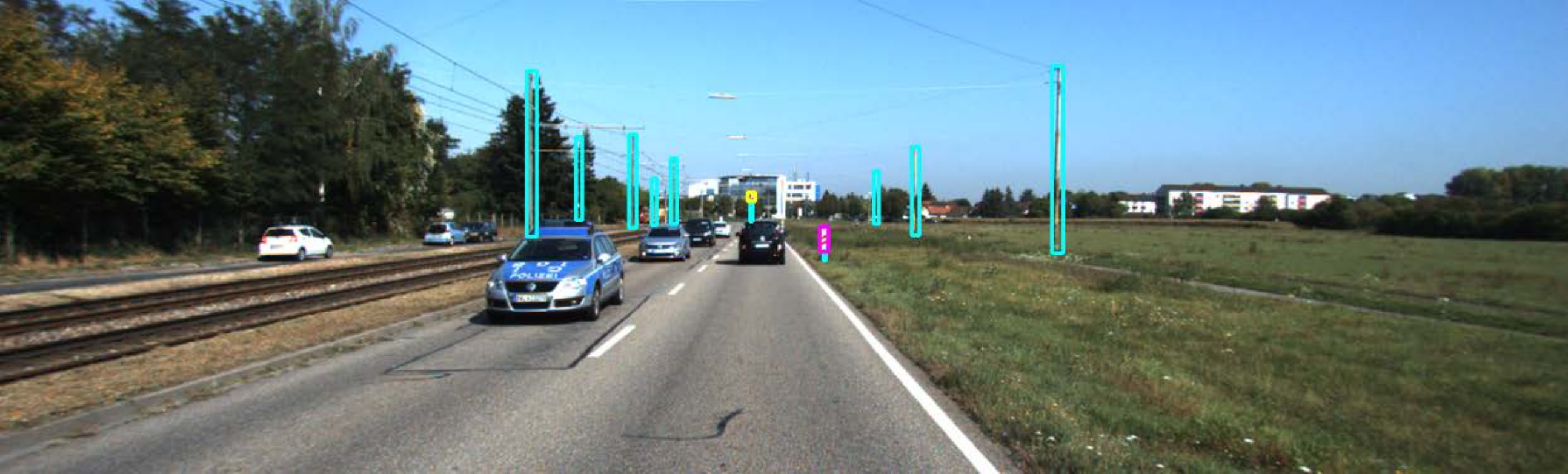}\hfill
\includegraphics[width=0.5\linewidth, height=0.06\textheight]{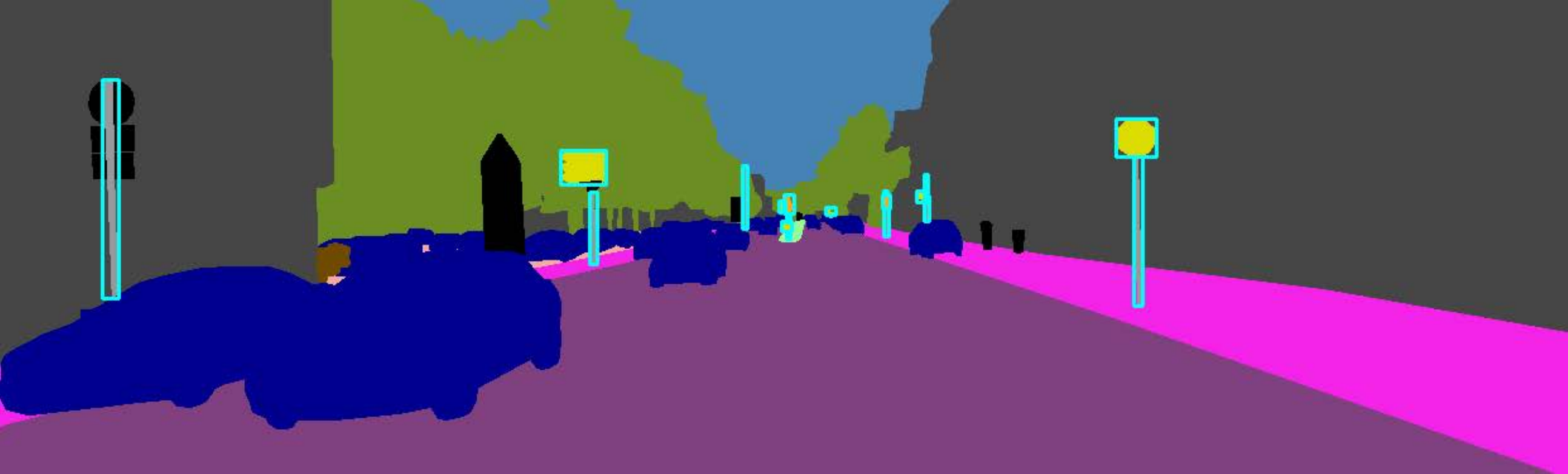}\hfill
\includegraphics[width=0.5\linewidth, height=0.06\textheight]{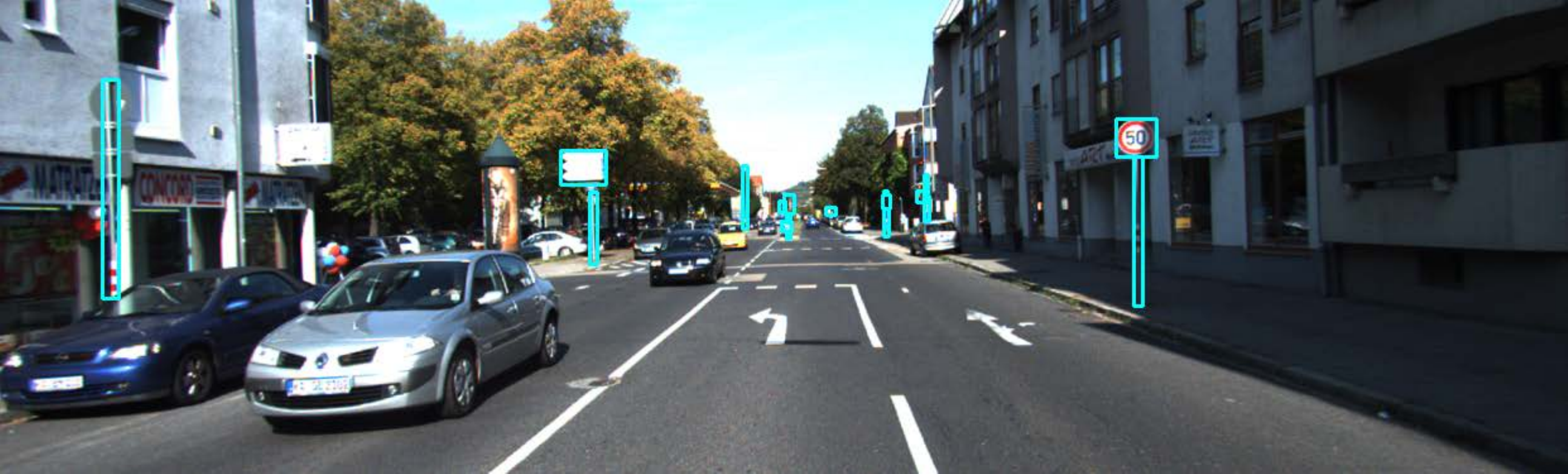}\hfill
\includegraphics[width=0.5\linewidth, height=0.06\textheight]{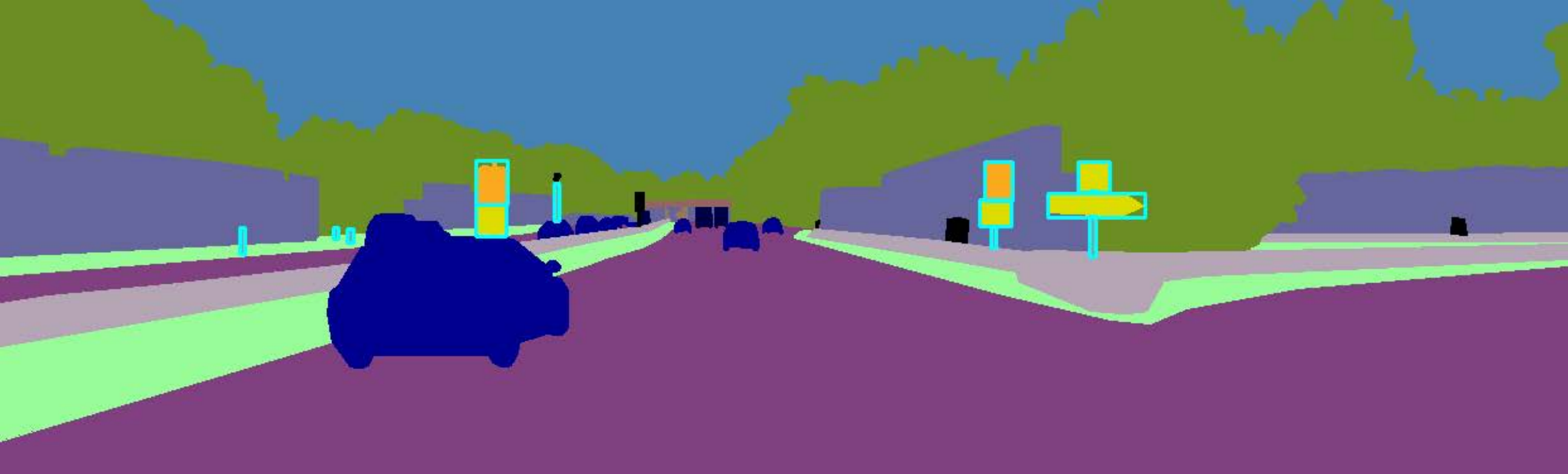}\hfill
\includegraphics[width=0.5\linewidth, height=0.06\textheight]{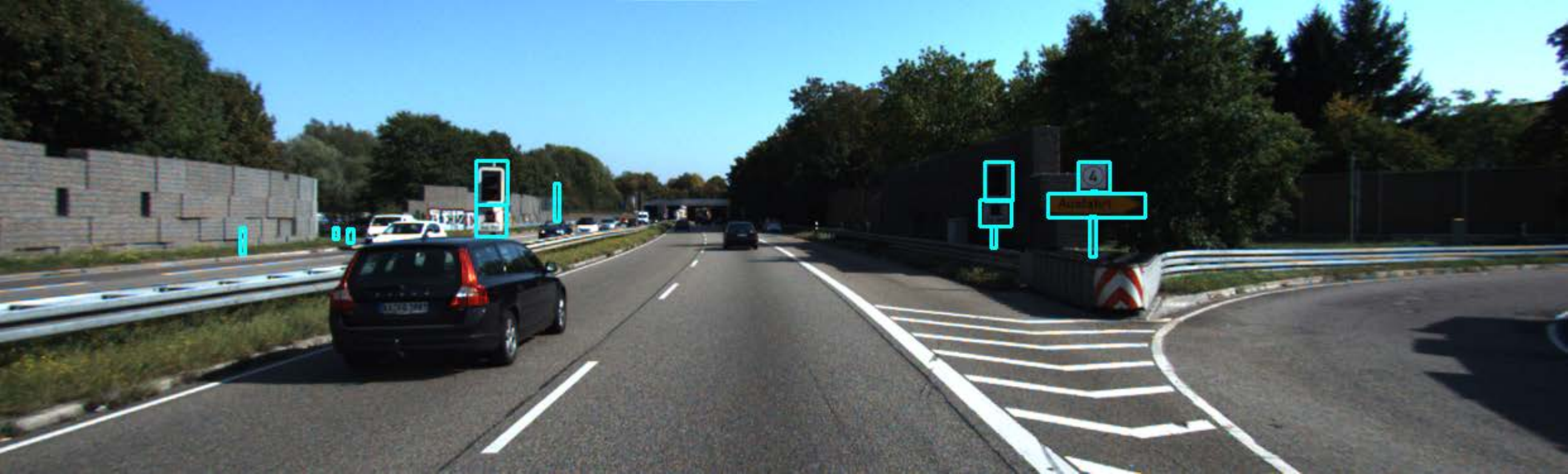}\hfill
\includegraphics[width=0.5\linewidth, height=0.06\textheight]{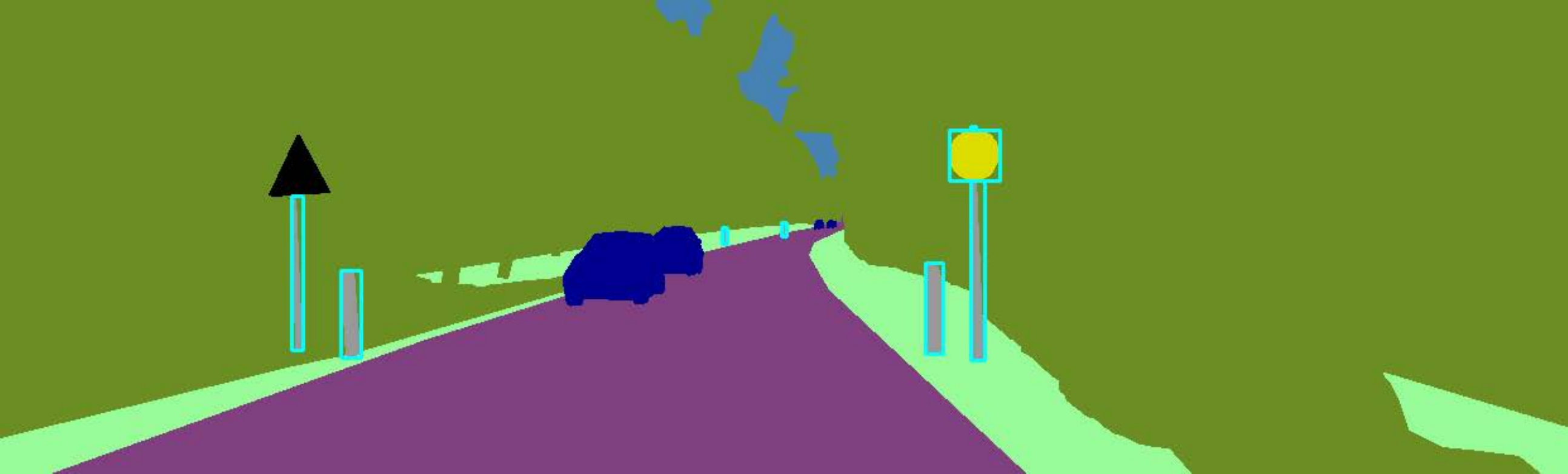}\hfill
\includegraphics[width=0.5\linewidth, height=0.06\textheight]{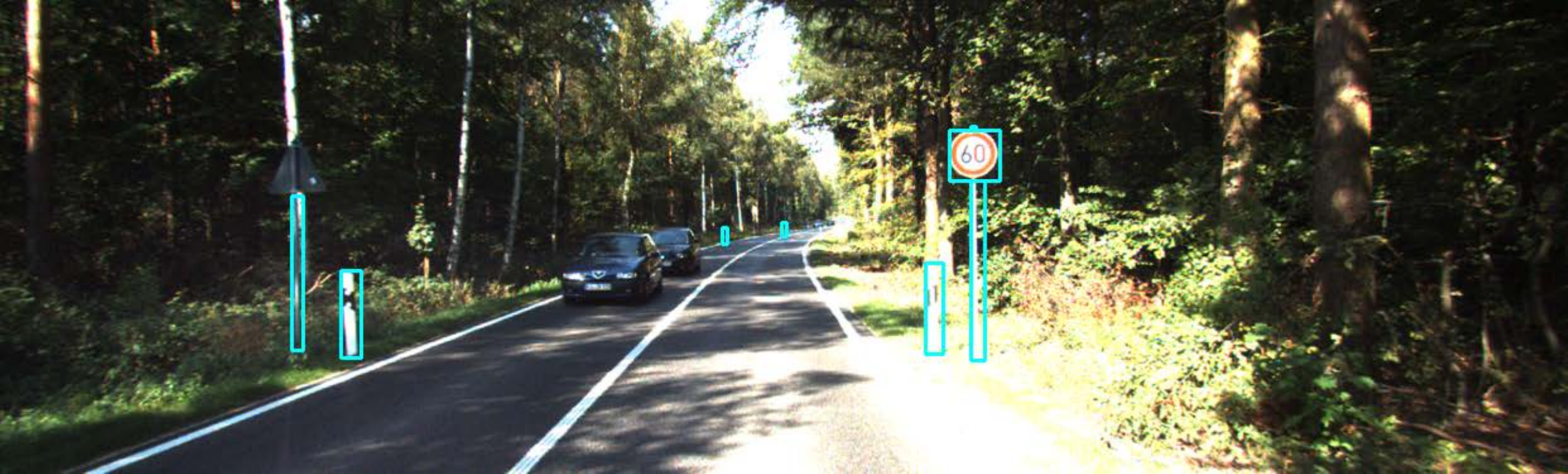}\hfill
\includegraphics[width=0.5\linewidth, height=0.06\textheight]{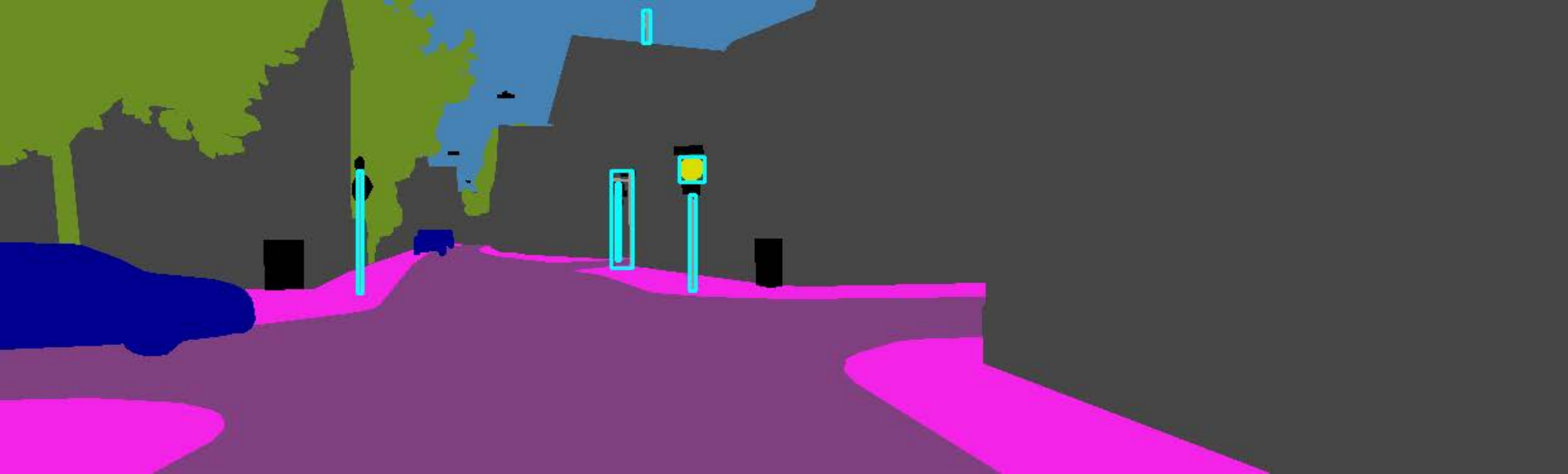}\hfill
\includegraphics[width=0.5\linewidth, height=0.06\textheight]{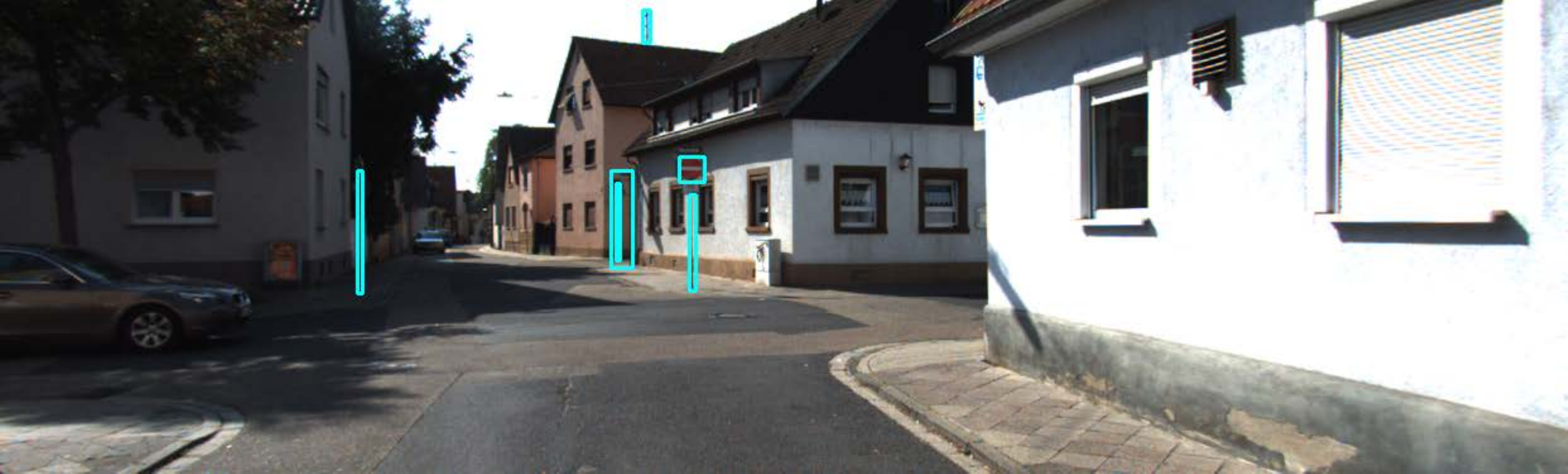}\hfill
\vspace{-0.3cm}
\caption{Several examples of ground truth landmark bounding box labels based on semantic segmentation masks in the KITTI dataset. 
Left:  semantic segmentation images with bounding box labels. Right: real images with bounding box labels.}
\label{fig:label_kitti}
\end{figure}

\begin{figure}[!htb]
\centering
\includegraphics[width=0.5\linewidth, height=0.12\textheight]{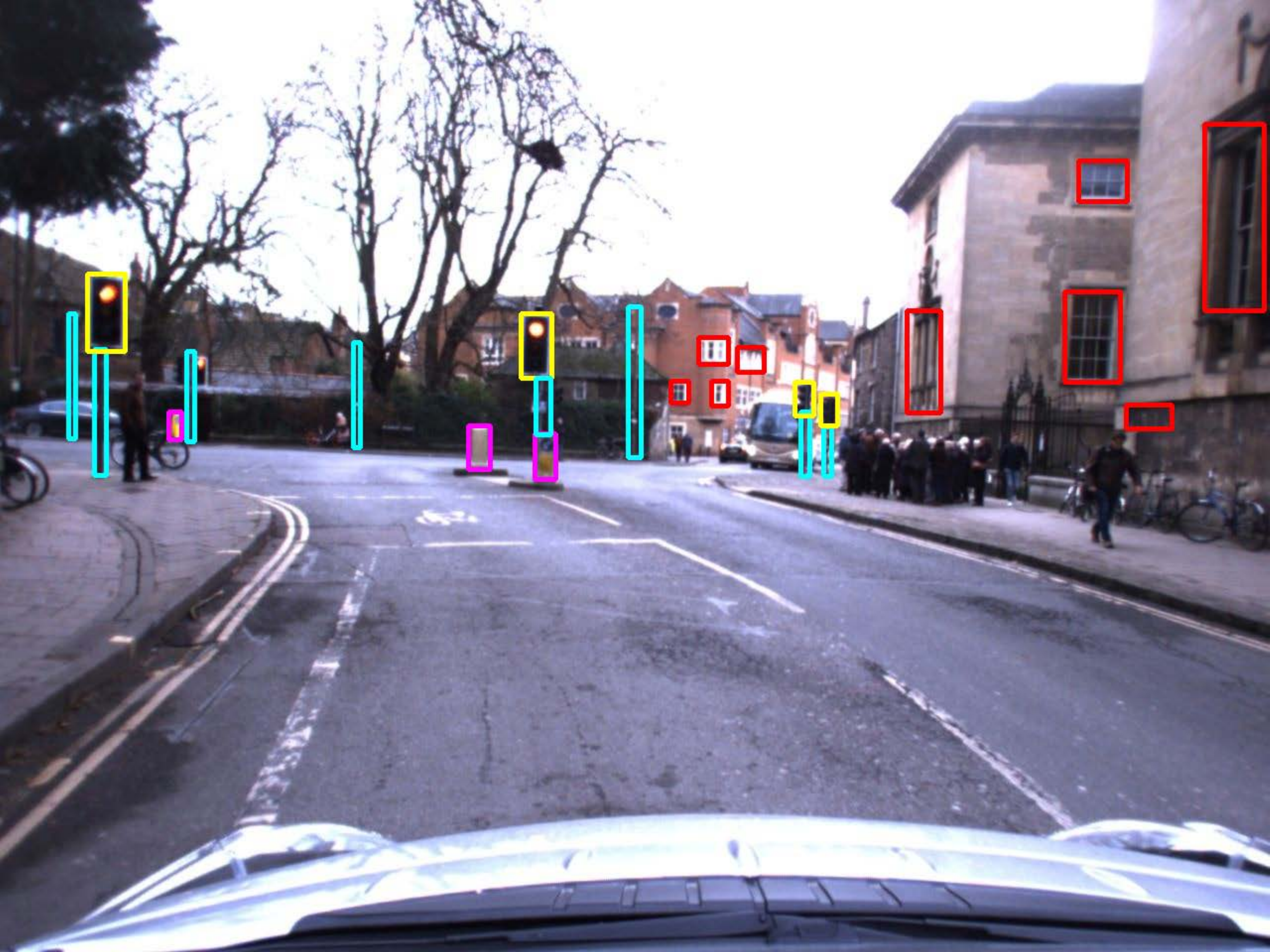}\hfill
\includegraphics[width=0.5\linewidth, height=0.12\textheight]{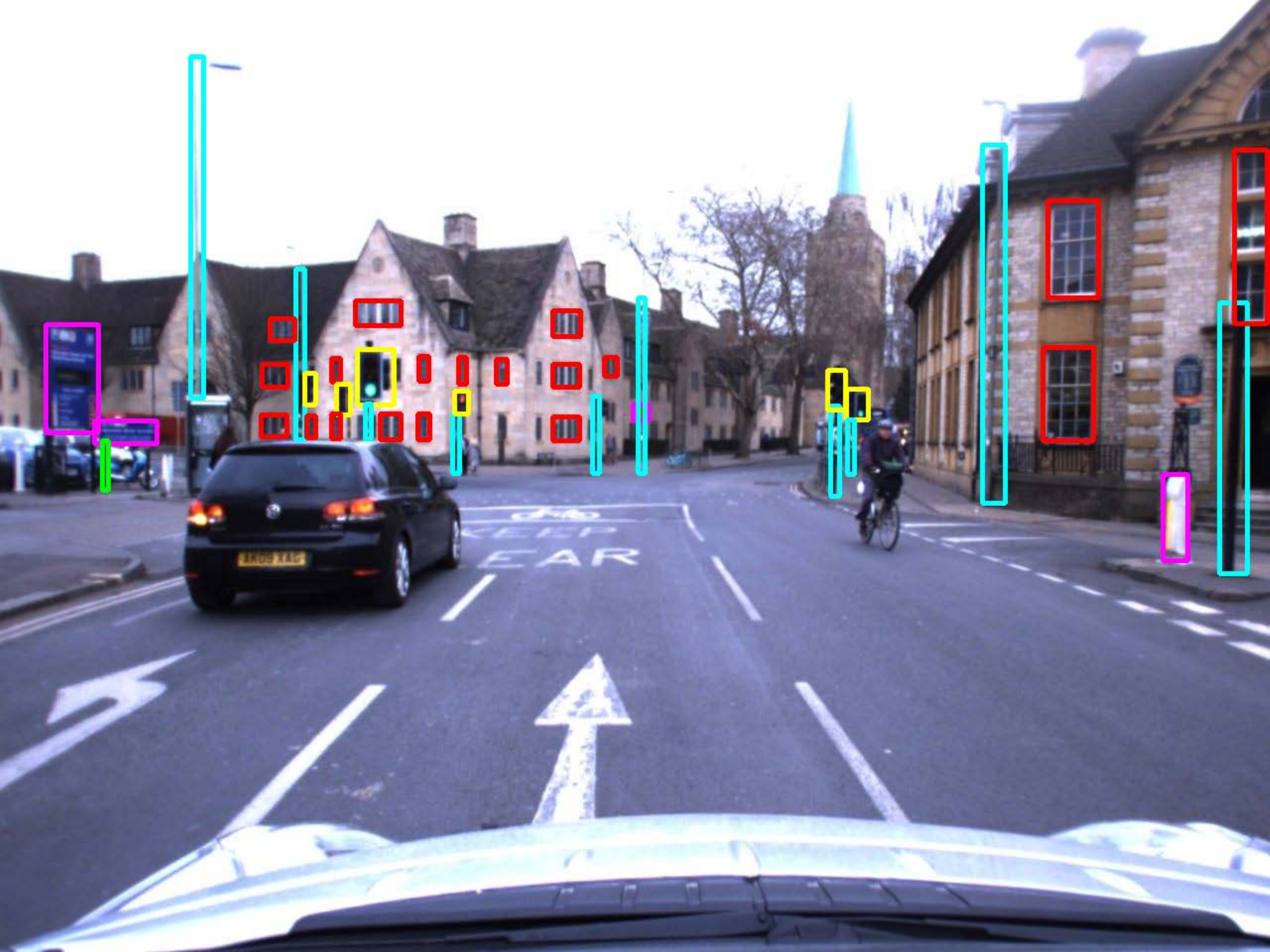}\hfill
\includegraphics[width=0.5\linewidth, height=0.12\textheight]{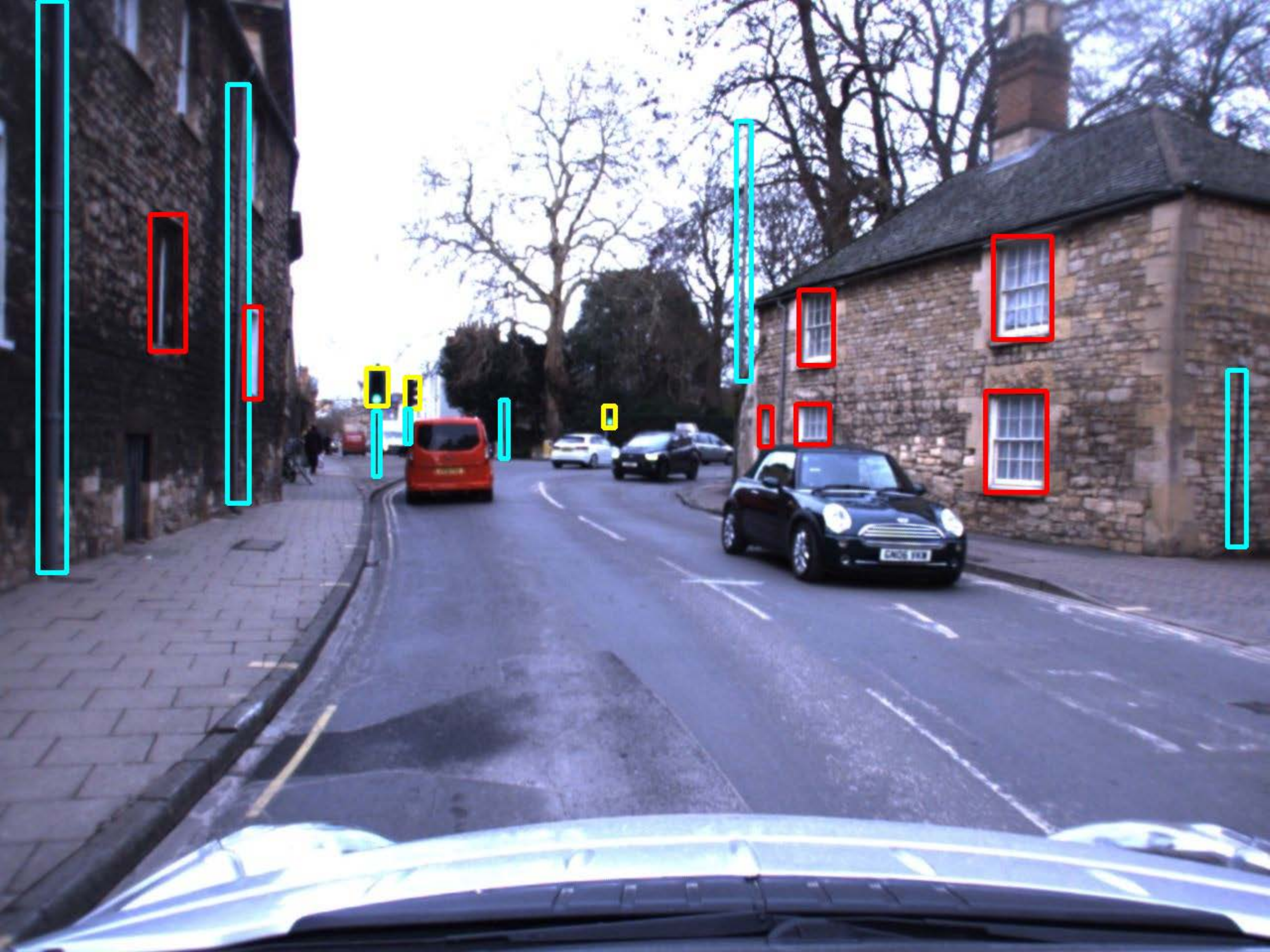}\hfill
\includegraphics[width=0.5\linewidth, height=0.12\textheight]{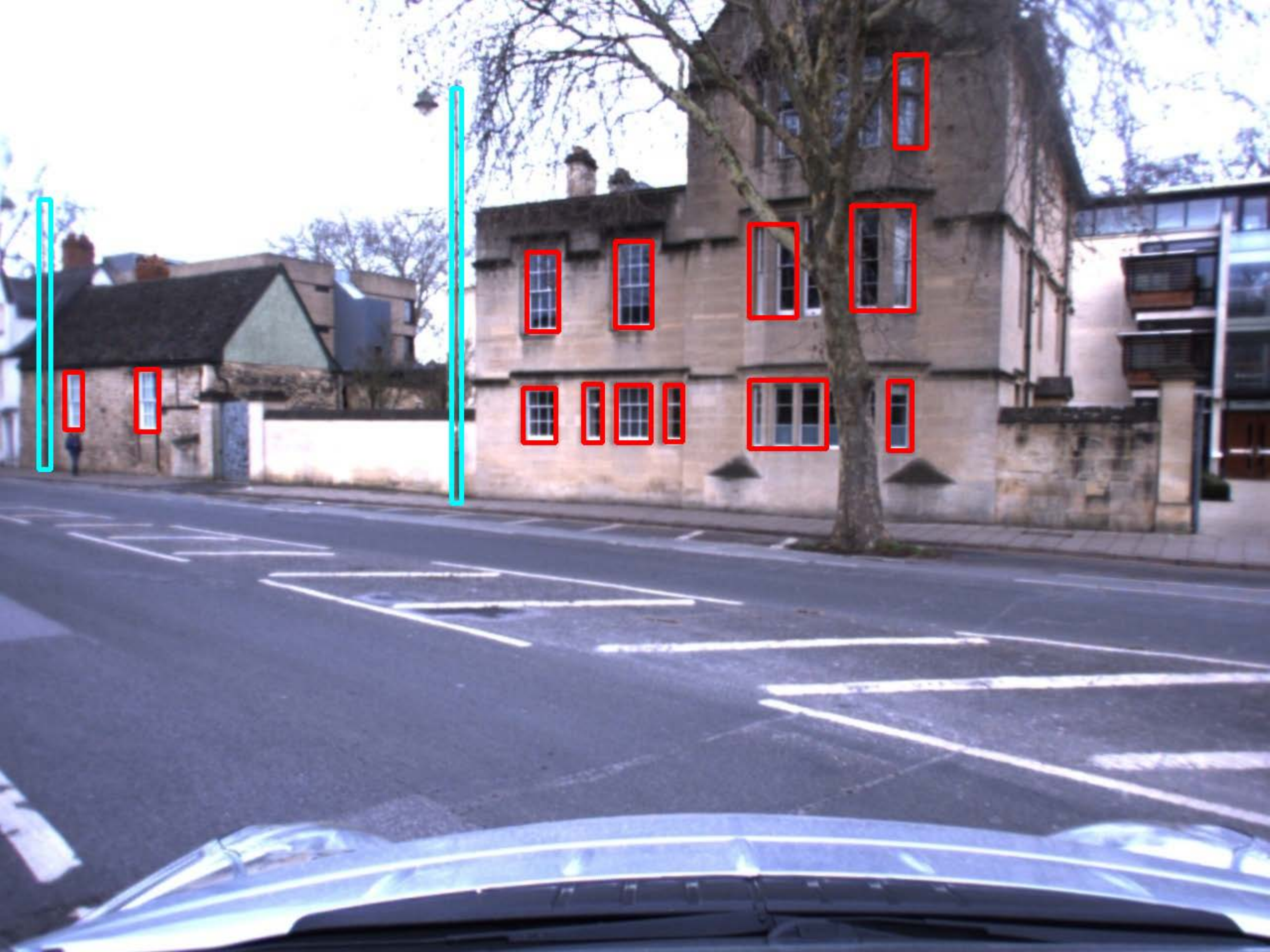}\hfill
\includegraphics[width=0.5\linewidth, height=0.12\textheight]{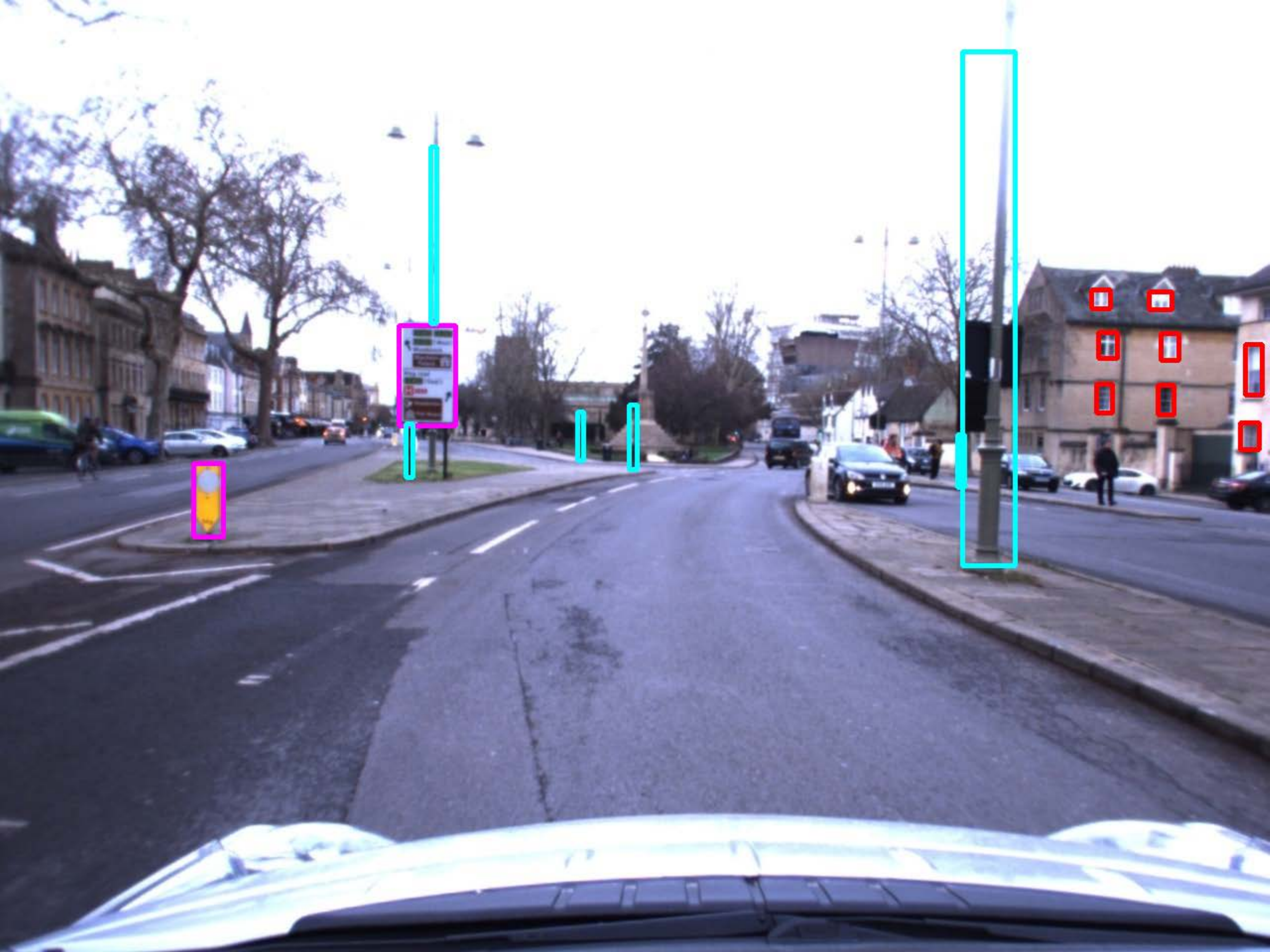}\hfill
\includegraphics[width=0.5\linewidth, height=0.12\textheight]{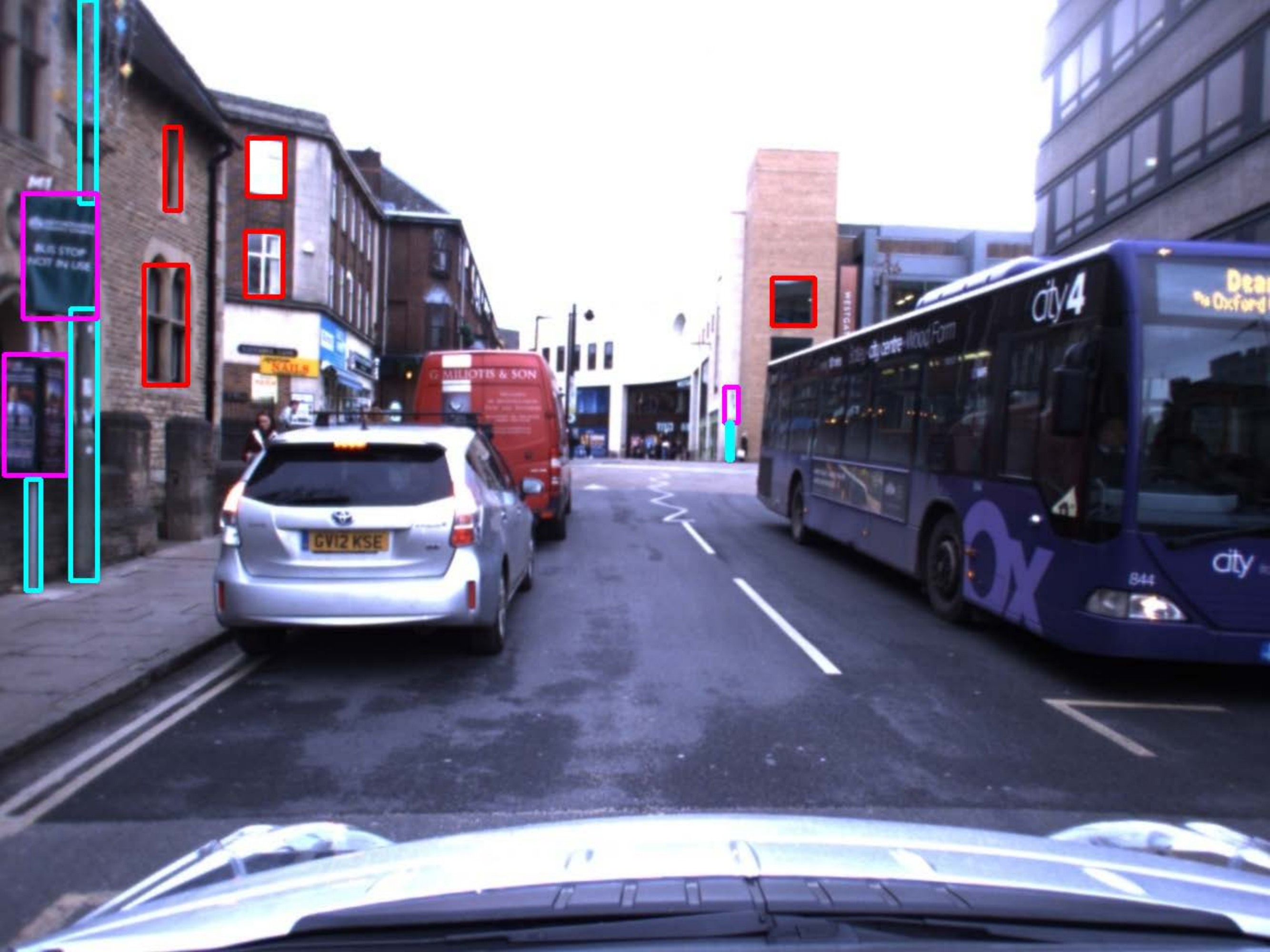}\hfill
\vspace{-0.3cm}
\caption{Several examples of the ground truth landmark bounding box labels for the Oxford Radar RobotCar dataset.}
\label{fig:label_oxford}
\end{figure}

\begin{figure}[!htb]
\centering
\begin{center}
\includegraphics[width=0.9\linewidth, height=0.4\textheight]{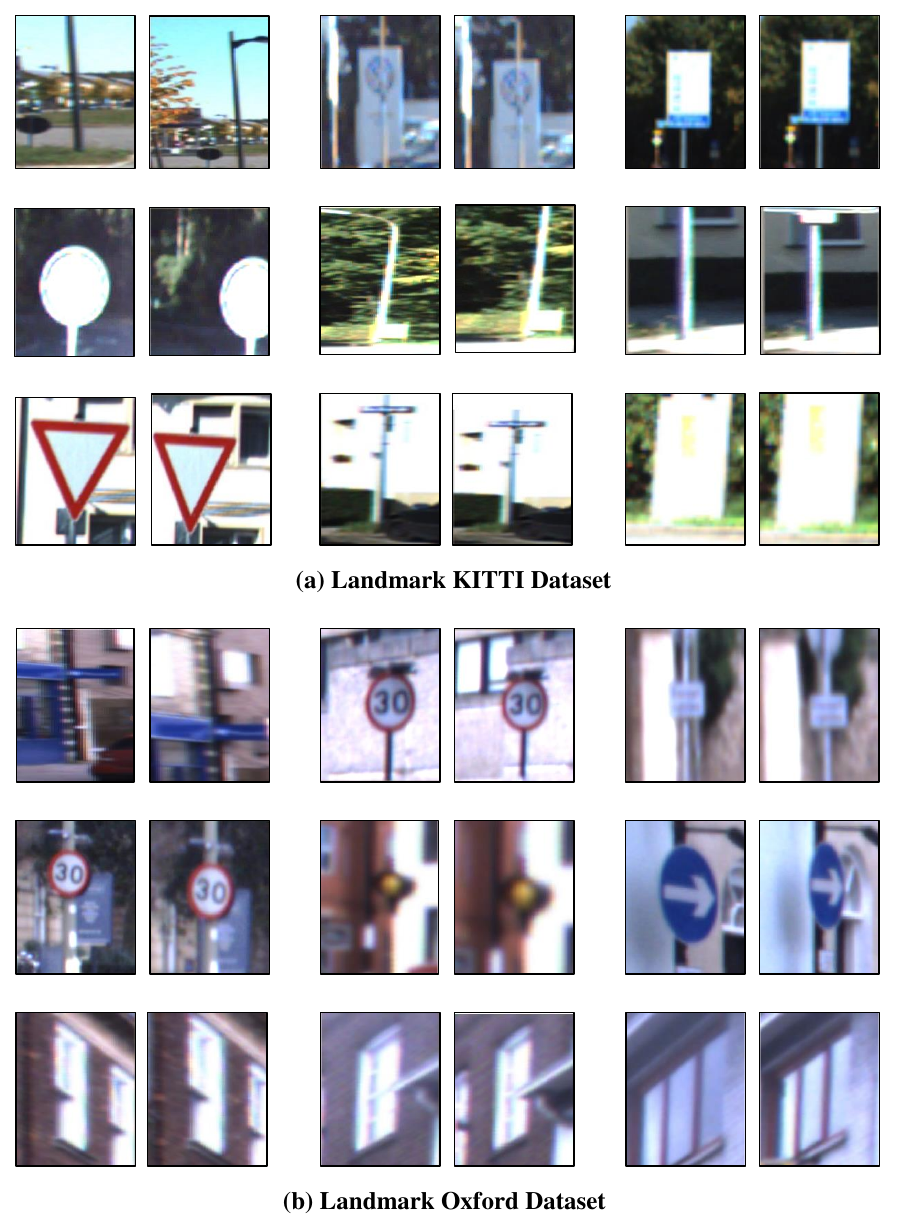}
\end{center}
\vspace{-0.5cm}
\caption{(a) and (b) are examples of landmark patch pairs from the Landmark KITTI dataset and the Landmark Oxford Dataset respectively.}
\label{fig_image_patches_pair}
\end{figure}

\begin{table*}[!htb]
\caption{Neural network models and the parameters in the image patch matching framework.}
\label{model-setting}
\begin{center}
\newcommand{\tabincell}[2]{\begin{tabular}{@{}#1@{}}#2\end{tabular}}
\begin{tabular}{cclc}
\hline\hline
{\bf Mapping}  & {\bf Models} & {\bf Layers (model parameters) } 
& \tabincell{c}{{\bf Dimension} {\bf of Outputs} }
\\ \hline
$f$ & Resnet            & \tabincell{l}{Resnet18 (without the last FC layer, 
                                        with 17 convolution layers)\\}              
                        & 512 \\ 
\hline
\multirow{1}{*}{$g$} 
    & \tabincell{c}{GAT/ \\
    GCN/ \\
    GraphSAGE \\}       & \tabincell{l}{GAT block 1 (4 attention heads, 
                                        4 $\times$ 128 hidden features \& ELU) \\
                                        GAT block 2 (4 attention heads, 
                                        4 $\times$ 128 hidden features \& ELU)/ \\
                                        GCN block 1 (512 hidden features \& ReLU) \\
                                        GCN block 2 (512 hidden features \& ReLU)/ \\
                                        GraphSAGE block 1 ([512, 512] hidden features,
                                        ReLU \& BatchNorm) \\
                                        GraphSAGE block 2 (512 hidden features) \\
                                        } 
                        & \tabincell{l}{512\\ 512\\ \ \\ \ \\ \ \\ \ \\}\\ 
\hline
$d$ & Discrimiator      & \tabincell{l}{Bilinear layer (four 512 $\times$ 512 hidden partitioned matrices
                                        \& Sigmoid function)\\} 
                        & \tabincell{l}{1\\} \\
\hline\hline
\end{tabular}
\end{center}
\vspace{-0.5cm}
\end{table*}

\textbf{Landmark KITTI Dataset.}
The segmentation labels are semantic segmentation masks. To perform landmark object detection, we need to first convert the semantic segmentation labels to object bounding box labels. We use the skimage.measure.label to label connected regions for pixel classes including traffic lights, traffic signs and poles. See \cref{fig:label_kitti} for an example. In some rare cases, multiple poles may overlap and the connected region algorithm outputs an inaccurate bounding box. 
We manually exclude these overlapped objects in the generated bounding box labels. As mentioned above, Faster R-CNN trained using the labels is used to produce the object detection results for all the other unlabeled frames contained in the dataset.

We project the surrounding LiDAR points onto the image frame plane using the intrinsic camera matrix and extrinsic camera matrix. Here, we have used the sensors' information (i.e., vehicle global ground truth locations) to accumulate collected LiDAR scans to build the $3$D LiDAR reference map. Due to the limited LiDAR field of view, a single LiDAR scan may not have any LiDAR point corresponding to some landmarks. To avoid this, we build a unified $3$D LiDAR reference map similar to that in PointNetVLAD. 
Based on the $3$D reference map, the LiDAR points reflected from the landmark patch are read out to obtain the global locations of the corresponding landmark objects. We apply DBSCAN to filter out some outlier points and obtain compact landmark objects. We then compute the $\mathcal{L}_2$ distance of each landmark patch pair from two frames to determine the patch matching ground truth. We have also gone through all the frames manually to remove or correct a few noisy landmark objects.
Finally, for each detected landmark object, we intentionally expand its bounding box by $15$ pixels on each side to include some background information. See \cref{fig_image_patches_pair} for an example. 

\textbf{Landmark Oxford Dataset.} 
To build the Landmark Oxford dataset, we manually labeled landmarks including traffic lights, traffic signs, poles, and facade windows for $500$ frames. See \cref{fig:label_oxford} for examples. Compared with the Landmark KITTI Dataset, we additionally include the window class in this dataset. (Window labels are not available for the Landmark KITTI Dataset yet. We will enrich the Landmark KITTI Dataset with window labels in future work.)
We then train Faster R-CNN to obtain the landmarks for all $29,687$ frames.
Similar operations are performed to obtain the final landmark patches with matching ground truths. See \cref{fig_image_patches_pair} for some landmark patch examples.

\subsection{Detailed Model Parameters}\label{app.Model_setting}

The details of the model setting mentioned in \cref{sect:experiment} of the paper are provided in the following \cref{model-setting}.

\subsection{Monocular Depth Estimation for VGIDM}\label{app.depth}

In our work, we assume the spatial information of the segmentation is available to construct the neighborhood graphs in VGIDM.
In \cref{sect:exper} of the paper, we perform evaluations on the two landmark datasets where Monocular Depth Prediction Module is used to obtain the spacial relationships among landmark patches contained in full-size images. The reported AbsRel of this depth estimation method is around $14$ meters. \cref{fig_depth_estimation_kitti} shows a few examples of the predicted depth. We observe that many objects in the predicted depth visualization are well distinguished from their surroundings. 

\begin{figure}[!htb]
\begin{center}
\includegraphics[width=1.02\linewidth]{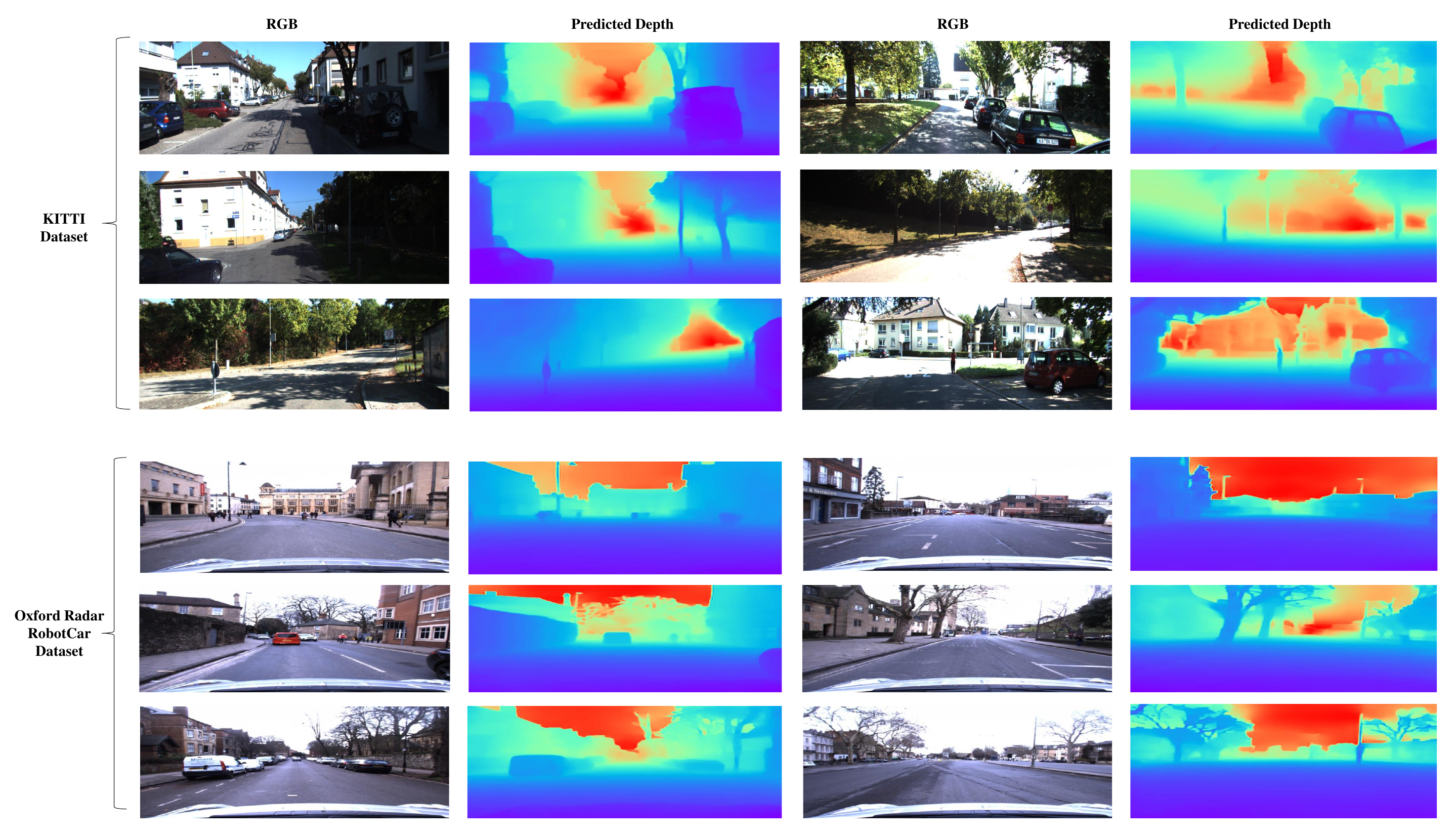}
\end{center}
\vspace{-0.5cm}
\caption{Image depth estimation results from the Monocular Depth Prediction Module for the KITTI dataset and the Oxford Radar RobotCar Dataset.}
\label{fig_depth_estimation_kitti}
\vspace{-0.3cm}
\end{figure}

\vfill

\end{document}